\documentclass{article}
\PassOptionsToPackage{numbers,compress}{natbib}
\usepackage[main,final]{neurips_2026}

\usepackage[utf8]{inputenc} % allow utf-8 input
\usepackage[T1]{fontenc}    % use 8-bit T1 fonts
\usepackage{hyperref}       % hyperlinks
\usepackage{url}            % simple URL typesetting
\usepackage{booktabs}       % professional-quality tables
\usepackage{amsfonts}       % blackboard math symbols
\usepackage{nicefrac}       % compact symbols for 1/2, etc.
\usepackage{microtype}      % microtypography
\usepackage{xcolor}         % colors

\usepackage{wrapfig,booktabs,amsmath}

\usepackage{threeparttable}
\usepackage{multirow}
\usepackage{enumitem}
\newcommand{\bst}[1]{{\textbf{\textcolor{red}{#1}}}}
\newcommand{\subbst}[1]{\textcolor{blue}{\underline{{#1}}}}
\usepackage{graphicx}
\usepackage{tcolorbox}
\usepackage{colortbl}
\usepackage{color}
\usepackage[table]{xcolor}
\definecolor{tabhighlight}{HTML}{e5e5e5}
\usepackage{amsmath}
\newcommand{\scalea}[1]{\scalebox{0.78}{#1}}
\newcommand{\scaleb}[1]{\scalebox{0.8}{#1}}
\definecolor{fbApp}{HTML}{DDEBFF}
\newcommand{\rowc}{\rowcolor{fbApp}}

\usepackage{amssymb}

\usepackage{subfigure}

\newtheorem{theorem}{Theorem}[section]
\newtheorem{proposition}[theorem]{Proposition}

\usepackage{pifont}
\newcommand{\XSolidBrush}{\ding{55}}
\newcommand{\Checkmark}{\ding{51}}

\newcommand{\best}[1]{{\textbf{\textcolor{red}{#1}}}}
\newcommand{\second}[1]{\textcolor{blue}{\underline{{#1}}}}

\PassOptionsToPackage{numbers, compress}{natbib}

\def \time{CvLoss}

\title{Multivariate Time Series Forecasting needs \\ Cross Variable Loss}

\author{%
  \textbf{Kuiye Ding}$^{1}$, \textbf{Yifan Hu}$^{2}$,
  \textbf{Hao Xue}$^{3,*}$, \textbf{Hanchen Wang}$^{1}$ \\[2pt]
  \normalfont
  $^{1}$ University of Technology Sydney, \ $^{2}$ Tsinghua University, \\
  $^{3}$ The Hong Kong University of Science and Technology (Guangzhou) \\[2pt]
  $^{*}$ Corresponding author. \\
  \texttt{kuiye.ding@student.uts.edu.au, \ huyf25@mails.tsinghua.edu.cn} \\
  \texttt{haoxue@hkust-gz.edu.cn, \ hanchen.wang@uts.edu.au}
}

\begin{document}

\maketitle

\begin{abstract}
  Multivariate time series forecasting presents unique challenges because future variables often co-evolve under shared system dynamics. While existing studies mainly focus on cross-variable dependencies in historical observations, dependencies among future values are much less explored. Specifically, modern forecasting models largely follow the Direct Forecasting (DF) paradigm, generating multi-step forecasts with point-wise objectives that do not explicitly constrain cross-variable structure. In this work, we show that the DF objective is mismatched in the presence of cross-variable and lagged dependencies, revealing an objective gap. To address this issue, we propose \textbf{C}ross-\textbf{V}ariable \textbf{Loss} (CvLoss), a plug-in structural regularizer that constrains forecast residuals on a cross-variable graph. CvLoss penalizes inconsistent edge-wise residual differences over forecast patches, encouraging consistency across both synchronous and asynchronous interactions. Our experiments show that CvLoss consistently improves competitive forecasting models, outperforms representative learning objectives, and is compatible with a variety of forecasting backbones.
  Code is available at \url{https://github.com/Day333/CvLoss}.
\end{abstract}

\section{Introduction}

Multivariate time series forecasting (MTSF) is a fundamental task across diverse domains, where the core challenge lies in accurately capturing complex synergistic relationships among variables. To model these temporal dynamics, recent pioneering works~\citep{itransformer, LIFT, qiu2025duet, hu2025timefilter, ding2025dualsg, hou2025graph} have predominantly focused their efforts on extracting representations from {historical} input features, a paradigm known as input-side channel dependence modeling. In stark contrast, the intrinsic correlations among future predictions at the output-side have been largely overlooked. Current mainstream Direct Forecasting (DF) paradigms~\citep{wang2025fredf} typically decouple the multivariate outputs forcibly, projecting the unified historical embeddings independently into future time steps, thereby degrading MTSF into a collection of isolated scalar prediction problems. This asymmetry heavily prioritizes the input while neglecting the output. Consequently, it not only severs the physical co-evolution among future variables but also induces severe structural errors.

On one hand, blindly decoupling prediction targets directly violates the intrinsic physical and operational laws of multivariate systems. In the real world, time series are rarely generated independently; rather, they are governed by shared system dynamics and exhibit strong co-evolutionary characteristics. For instance, in quantitative financial markets~\citep{hu2025fintsb}, the returns of correlated assets often display highly synchronous surges, plunges, or hedging effects when subjected to macroeconomic shocks~\citep{hu2025finmamba}. Similarly, in traffic flow forecasting, the traffic volumes recorded by hundreds of sensors within an urban network are driven by shared meteorological conditions (e.g., sudden rainstorms causing widespread congestion), and synchronized daily commuting patterns. This results in a profound co-evolution of future traffic dynamics, as illustrated in Figure~\ref{fig:idea_pems03}. If a model isolates these variables at the output stage, the predicted curves for individual variables might appear numerically well-fitted. However, when combined into a joint multivariate forecast, they frequently produce contradictory results that violate true physical constraints and market rules.

Furthermore, overlooking these synergistic relationships induces critical structural errors at the optimization level. Modern MTSF models rely heavily on point-wise loss functions, such as Mean Squared Error (MSE)~\citep{wang2024tssurvey, qiu2026survey}, for optimization. From the perspective of probabilistic inference, this point-wise optimization is fundamentally predicated on a highly restrictive assumption: the prediction residuals across all variables and time steps follow completely independent and isotropic Gaussian distributions (i.e., possessing a diagonal covariance matrix). However, given the aforementioned co-evolutionary phenomena, the true spatio-temporal covariance matrix of the residuals inevitably contains numerous non-zero off-diagonal elements. This creates an unbridgeable theoretical divide, termed an Objective Gap, between existing standard loss functions and the true joint spatio-temporal distribution of the data. Consequently, even if two models achieve remarkably similar point-wise MSE, their cross-variable covariance structures might differ drastically, completely failing to capture the joint geometric topology of the original data.

To overcome these dual physical and mathematical deficiencies, the optimization objective must transcend the singular pursuit of absolute point-wise accuracy and explicitly constrain the {relative structural consistency} among variables. This implies that we must capture not only the concurrent interactions (synchronous effects) between variables at the same time step but also the lagged propagation of system perturbations as time evolves (asynchronous effects). To elegantly unify these two phenomena, we introduce a graph-structured perspective, as graphs are inherently suited for depicting complex, non-uniform pair-wise interactions among multivariate data. Specifically, we propose a novel graph-based structural regularization scheme termed Cross-Variable Loss (CvLoss). By constructing a cross-variable graph, CvLoss connects different variables at identical time steps to capture synchronous concurrent interactions, while also linking variables across different steps to model asynchronous lagged effects. Essentially, CvLoss imposes an $\ell_1$-norm Graph Total Variation (Graph TV) regularization on the prediction residual field, compelling the model to strictly adhere to the underlying relative structural consistency while minimizing absolute errors. Furthermore, as a plug-and-play regularizer, CvLoss seamlessly integrates into various mainstream forecasting architectures. Our introduced edge-sampling mechanism significantly reduces computational complexity during training, and crucially, the regularizer imposes zero overhead during inference, ensuring testing speed remains completely unaffected while maintaining extremely high predictive accuracy.

Our main contributions are summarized as follows:
\begin{itemize}[leftmargin=*,noitemsep,nosep]
    \item We reveal a critical limitation in current MTSF paradigms, demonstrating that standard point-wise learning objectives neglect inherent inter-channel error dependencies among future predictions.
    \item We propose a novel Cross-Variable Loss acting as a structural regularizer, explicitly constraining the inter-channel relationships of predicted values via a residual consistency formulation.
    \item Extensive experiments demonstrate our proposed objective consistently improves various state-of-the-art models, achieving superior predictive performance with affordable computational overhead.
\end{itemize}

\section{Preliminaries and Related Work}

\subsection{Problem definition}
In this paper, we focus on the MTSF problem. Throughout, italic capitals denote matrices (e.g., $X$, $Y$, $Z$, $E$), bold lowercase letters denote vectors (e.g., $\mathbf{e}$), calligraphic capitals denote sets (e.g., $\mathcal{E}$), and lowercase letters denote scalars and indices (e.g., $t$, $d$, $\alpha$); the dimension constants $H$, $T$, $D$, $P$, $L$ and $N$ are written as capitals by convention. Given a multivariate time-series dataset with $D$ variables, the input historical sequence is defined as $X \in \mathbb{R}^{H \times D}$, and the ground-truth label sequence is defined as $Y \in \mathbb{R}^{T \times D}$, where $H$ is the lookback window and $T$ is the forecast horizon. Modern time series forecasting models are primarily trained under a multitask learning manner, known as the direct forecasting (DF) paradigm. Therefore, the target is to learn a multi-output predictor $f_{\theta}:\mathbb{R}^{H\times D}\rightarrow\mathbb{R}^{T\times D}$, parameterised by the trainable weights $\theta$ of the forecasting backbone, that directly produces $T$-step forecasts for all $D$ variables simultaneously, yielding the prediction $\hat{Y}=f_{\theta}(X)$. Unlike current paradigms that decouple multivariate outputs into isolated scalar predictions, our formulation explicitly maintains the joint structure to capture cross-channel dependencies and shared temporal dynamics.

\subsection{Learning objectives in time-series forecasting}
Recent learning objectives extend point-wise MSE to capture temporal dynamics through shape alignment~\citep{GDTW, Dilate, soft-dtw}, likelihood maximization~\citep{wang2026iclrqdf, wang2025nipstimeo1, koopman, wang2025fredf}, distribution balancing~\citep{kmbdf, wang2026iclrdistdf}, and intra-sequence modeling via decomposition~\citep{qiu2025DBLoss} or local step alignment~\citep{TDAlign, patchloss}. Despite these advancements, existing approaches fundamentally treat multivariate forecasting as a collection of isolated scalar predictions. By optimizing exclusively for intra-channel temporal structures, they completely neglect explicit constraints on cross-variable dependencies. Consequently, these methods fail to capture the joint geometric structure and correlated disturbances among multivariate outputs.

\subsection{Cross-variable Dependencies in time-series forecasting}

Effectively modeling cross-variable dependencies is fundamental for MTSF~\citep{wang2024tssurvey, liangsurvey}. Existing paradigms typically employ Channel Dependence (CD) strategies to explicitly capture inter-channel correlations, as seen in representative architectures like iTransformer~\citep{itransformer}, TSMixer~\citep{chen2023tsmixer}, Crossformer~\citep{crossformer}, and TimesNet~\citep{Timesnet}. However, CD models often overfit to spurious correlations and struggle with noise interference from irrelevant variables. To address this, recent advancements introduce Channel Partiality (CP) to selectively filter dependencies. For instance, DUET~\citep{qiu2025duet} proposes a Channel Clustering Module to filter noisy channels via frequency-domain soft clustering. Similarly, TimeFilter~\citep{hu2025timefilter} constructs a patch-specific spatial-temporal graph with dynamic routing to retain essential correlations while discarding irrelevant noise. Other approaches like MTGNN~\citep{wu2020connecting} and MCformer~\citep{mcformer} also explore sparse interactions to balance accuracy and efficiency. Despite these significant improvements, current models predominantly focus on extracting representations from historical observations~\citep{zhao2024rethinking, wei2026tarfvae, li2026fadti}. The complex structural geometry and co-evolving cross-variable dependencies among future predictions remain fundamentally underexplored.

\begin{figure}
\subfigure[Ground-truth structure.]{\includegraphics[width=0.23\linewidth]{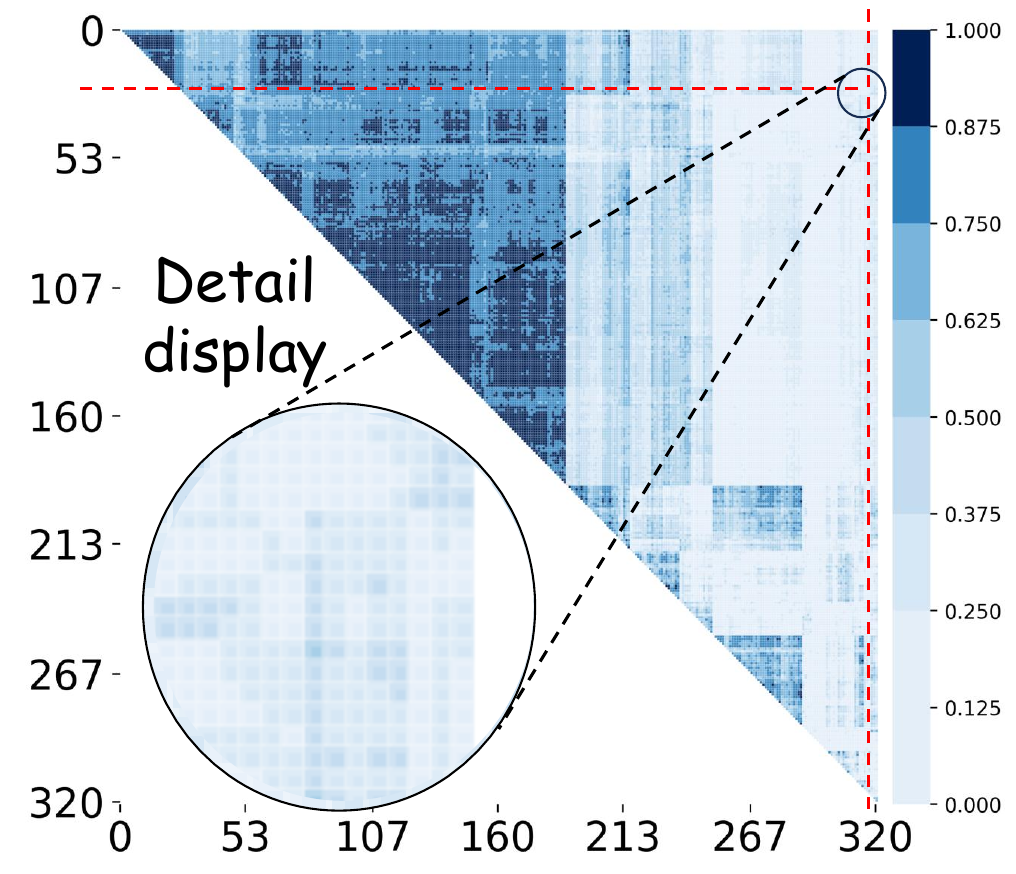}}
\hfill
    \raisebox{-0.0\height}{\rule{0.8pt}{2.6cm}} % Vertical line of 4cm height and 0.5pt width
\hfill
\subfigure[Baseline structure.]{\includegraphics[width=0.23\linewidth]{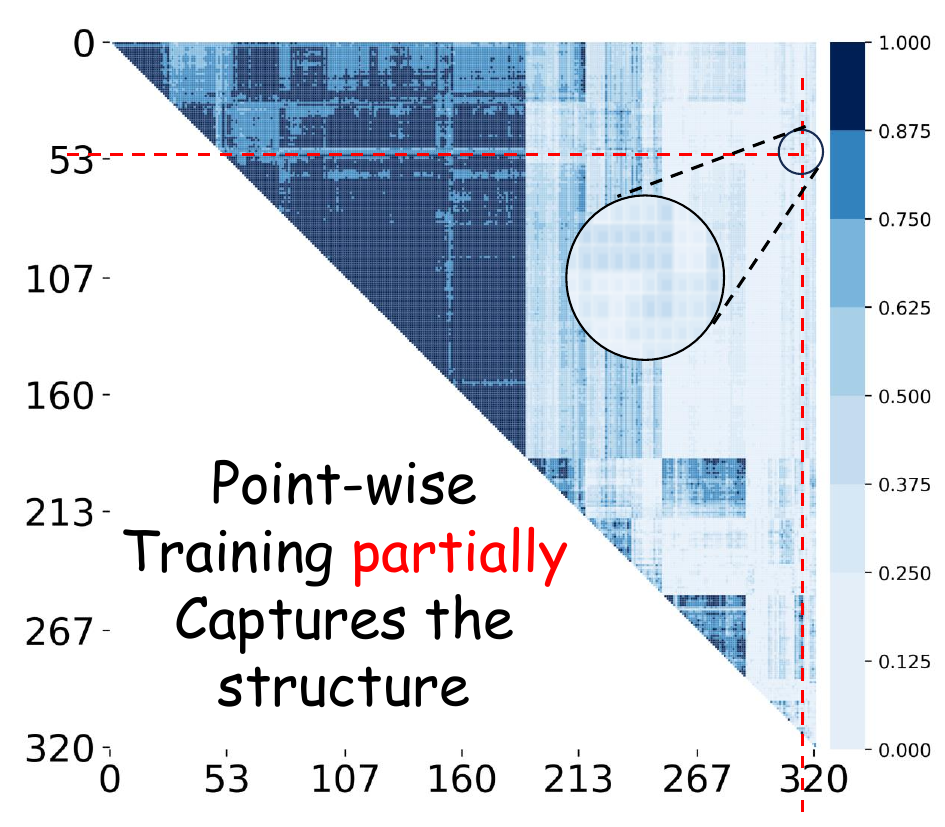}}
\hfill
    \raisebox{-0.0\height}{\rule{0.8pt}{2.6cm}} % Vertical line of 4cm height and 0.5pt width
\hfill
\subfigure[Improvement by CvLoss.]{\includegraphics[width=0.23\linewidth]{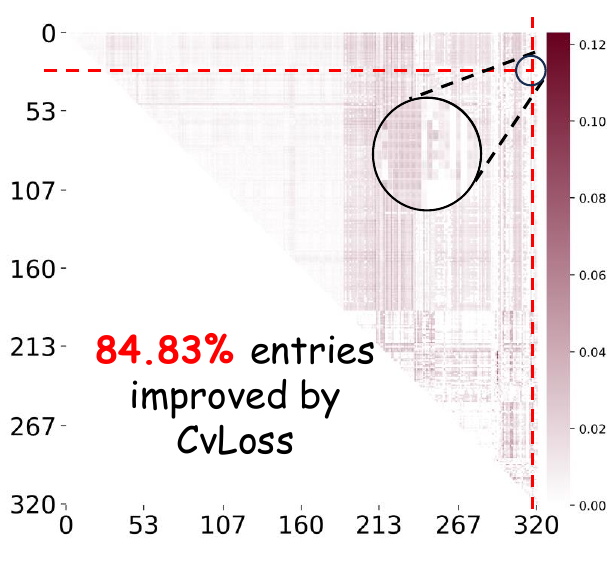}}
\hfill
    \raisebox{-0.0\height}{\rule{0.8pt}{2.6cm}} % Vertical line of 4cm height and 0.5pt width
\hfill
\subfigure[ECL Snapshot.]{\includegraphics[width=0.2\linewidth]{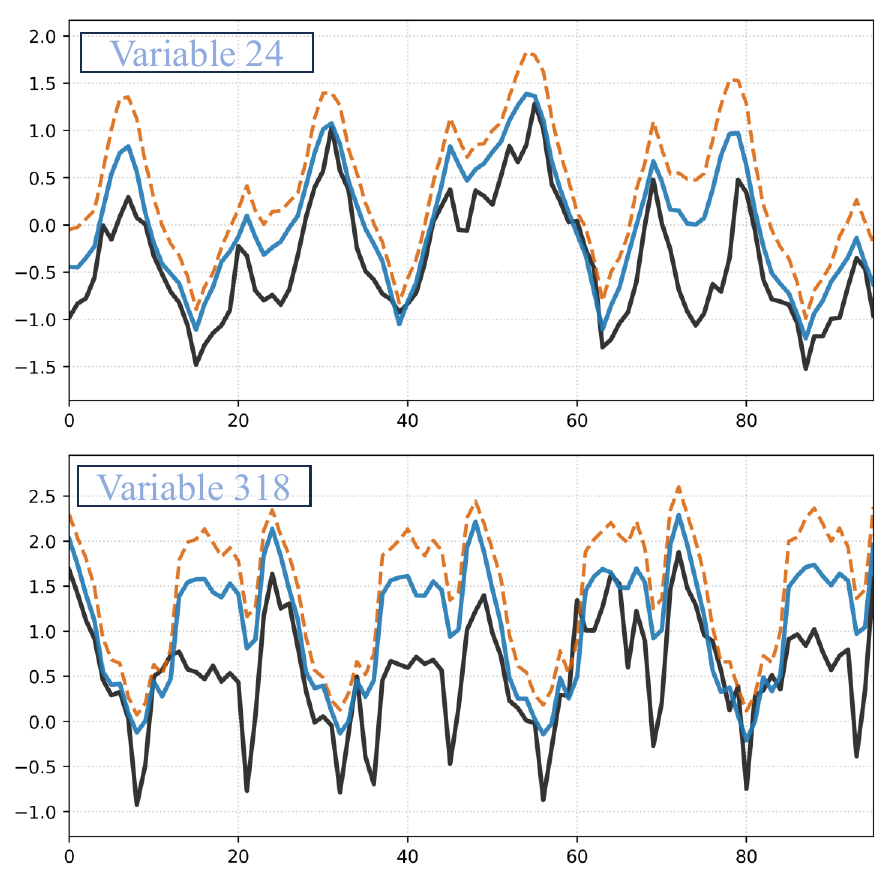}}
\caption{Motivating example on ECL with input length $H=96$ and forecast horizon $T=96$. (a,b) Cross-variable correlation structures of the ground-truth future series and of the series predicted by a plain iTransformer~\citep{itransformer}. (c) Entries on which the same backbone trained with \time\ attains a smaller structural error than the plain baseline, covering $84.83\%$ of all cross-variable entries. (d) Forecast snapshot on a representative variable. The protocol used to compute (a)--(c) is defined in Appendix~\ref{app:heatmap_vis}.}
\label{fig:idea}
\end{figure}

\section{Proposed Method}

\subsection{Motivation}

Modern time series forecasting models are primarily trained under the multitask learning manner, known as the direct forecasting (DF) paradigm~\citep{wang2025fredf}. Given an input history \(X\in\mathbb{R}^{H\times D}\), a multi-output predictor \(f_\theta:\mathbb{R}^{H\times D}\rightarrow\mathbb{R}^{T\times D}\) directly produces \(T\)-step forecasts for \(D\) variables:
$
\hat Y = f_\theta(X).
$
The standard training objective is the point-wise MSE:
\begin{equation}\label{eq:df_loss}
\mathcal{L}_{\mathrm{df}} = \frac{1}{T} \sum_{t=1}^{T} \frac{1}{D} \| Y_{t,:} - \hat{Y}_{t,:} \|_2^2
\end{equation}

This objective is simple and effective, but it only penalizes prediction errors at individual time-variable entries. Although modern forecasting architectures may implicitly capture cross-variable interactions through shared parameters, such dependencies are not explicitly reflected in the objective in Eq.~\eqref{eq:df_loss}, which does not constrain the \emph{relative error structure} across variables. As a result, two models with similar point-wise MSE may still exhibit substantially different cross-variable correlated patterns.

Figure~\ref{fig:idea} illustrates the objective gap on ECL. With the same iTransformer backbone, CvLoss recovers the ground-truth correlation matrix more accurately on $84.83\%$ of cross-variable entries and improves the forecast snapshot. Explicit coupling across related targets is also consistent with multi-task learning~\citep{evgeniou2005learning}. We therefore complement point-wise supervision with residual consistency across variables and examine its probabilistic motivation below.

Let $\mathbf{e}=\mathrm{vec}(\hat{Y}-Y)\in\mathbb{R}^{TD}$ denote the flattened prediction error vector over the entire forecast horizon and all variables. Suppose the residuals follow a zero-mean multivariate Gaussian distribution
$\mathbf{e}\sim\mathcal{N}(\mathbf{0},\mathbf{\Sigma}_{ST})$,
where $\mathbf{\Sigma}_{ST}\in\mathbb{R}^{TD\times TD}$ is the spatiotemporal covariance matrix that captures both cross-variable and cross-temporal dependencies.

\begin{theorem}[Objective Gap with Spatiotemporal Dependencies]\label{thm:spatiotemporal_nll}
Ignoring additive constants independent of the predictions, the scaled negative log-likelihood (NLL) under the true covariance $\mathbf{\Sigma}_{ST}$ is
$
\mathcal{L}_{\mathrm{nll}}
=
\frac{1}{2TD}\mathbf{e}^{\top}\mathbf{\Sigma}_{ST}^{-1}\mathbf{e}.
$
On the other hand, the standard DF objective in Eq.~\eqref{eq:df_loss} can be written as
$
\mathcal{L}_{\mathrm{df}}
=
\frac{1}{TD}\mathbf{e}^{\top}\mathbf{e}.
$
Therefore, compared with the isotropic Gaussian NLL under variance $\sigma^2$, the discrepancy between the two objectives is 
\begin{equation}
\Delta
=
\mathcal{L}_{\mathrm{nll}}
-
\frac{1}{2\sigma^2}\mathcal{L}_{\mathrm{df}}
=
\frac{1}{2TD}\mathbf{e}^{\top}
\left(
\mathbf{\Sigma}_{ST}^{-1}
-
\frac{1}{\sigma^2}\mathbf{I}
\right)
\mathbf{e},
\end{equation}
where $\mathbf{I}\in\mathbb{R}^{TD\times TD}$ is the identity matrix. If the residual precision has non-zero off-diagonal entries, $\mathbf{\Sigma}_{ST}^{-1}\neq\sigma^{-2}\mathbf{I}$ and the two objectives are not identical as functions of the residual vector. The gap may vanish for particular residual vectors, including $\mathbf{e}=\mathbf{0}$.
\end{theorem}

Theorem~\ref{thm:spatiotemporal_nll} identifies the geometry omitted by a spherical Gaussian objective. Cross-variable and lagged residual dependencies motivate a relation-aware penalty without estimating the full precision matrix, which has more than $5\times10^{10}$ entries on ECL at $T=720$ and $D=321$. The theorem motivates coupling residuals; Proposition~\ref{prop:precision} describes a structured Gaussian family, while Appendix~\ref{app:norm_choice} evaluates the choice of norm.

\begin{proposition}[MSE with a squared graph penalty is a structured Gaussian NLL]\label{prop:precision}
Let $\mathcal{E}$ be a cross-variable edge set over the $N$ patch nodes, let $\tilde A=A\otimes I_L$ be the lifted incidence matrix of Section~\ref{cvl}, and let $\mathrm{Lap}(\mathcal{E})=\tilde A^{\top}\tilde A$ be the corresponding graph Laplacian. Define the structured precision family
\begin{equation}\label{eq:precision_family}
\mathbf{\Sigma}^{-1}(\lambda)
=
\frac{1}{\sigma^{2}}\bigl(\mathbf{I}+\lambda\,\mathrm{Lap}(\mathcal{E})\bigr),
\qquad \lambda\ge 0 .
\end{equation}
Then, writing $\mathbf{r}=\mathrm{vec}(E)$ for the flattened residual field, the following identity holds for every $\mathbf{r}$:
\begin{equation}\label{eq:precision_identity}
\mathbf{r}^{\top}\mathbf{\Sigma}^{-1}(\lambda)\,\mathbf{r}
=
\frac{1}{\sigma^{2}}
\Bigl(
\|\mathbf{r}\|_2^{2}
+
\lambda\!\!\sum_{(i,j)\in\mathcal{E}}\!\!\|\mathbf{e}_i-\mathbf{e}_j\|_2^{2}
\Bigr).
\end{equation}
\end{proposition}

Proposition~\ref{prop:precision}, proved in Appendix~\ref{app:proposition}, expresses the Gaussian NLL as point-wise MSE plus a squared graph penalty. The edge set specifies supported cross-variable interactions, and $\lambda=0$ recovers the spherical case. Our default complete cross-variable graph includes every cross-variable pair; the implemented $\ell_1$ penalty uses this graph with a robust, non-Gaussian residual model, as discussed below.

\subsection{Predicting Multivariate Time Series with Cross Variable Loss Constraints} \label{cvl}

\subsubsection{Definition}

\paragraph{Patch nodes.}
Let \(Y\in\mathbb{R}^{T\times D}\) be the target and \(\hat Y=f_\theta(X)\in\mathbb{R}^{T\times D}\) be the prediction. We split the time axis into \(P\) non-overlapping patches of length \(L\)~\citep{PatchTST, 2026timemosaic}, so that \(T=P\times L\). Each node is defined as a pair \(v=(p,d)\), where \(p\in\{1,\dots,P\}\) indexes the patch and \(d\in\{1,\dots,D\}\) indexes the variable. Let \(\mathbf{z}_v\in\mathbb{R}^{L}\) and \(\hat{\mathbf{z}}_v\in\mathbb{R}^{L}\) denote the ground-truth and predicted patch vectors at node \(v\), respectively. Stacking all nodes gives
\begin{equation}
Z\in\mathbb{R}^{N\times L},
\qquad
\hat Z\in\mathbb{R}^{N\times L},
\qquad
N=P\times D.
\end{equation}
We define the node-wise residuals as
\begin{equation}
\mathbf{e}_v=\hat{\mathbf{z}}_v-\mathbf{z}_v\in\mathbb{R}^{L},
\end{equation}
and stack them into the residual matrix \(E\in\mathbb{R}^{N\times L}\), whose vectorisation \(\mathrm{vec}(E)\in\mathbb{R}^{NL}\) is a permutation of the flattened error vector \(\mathbf{e}\) of Theorem~\ref{thm:spatiotemporal_nll}, since \(NL=TD\). For $L=1$, patch nodes recover point nodes.

\paragraph{Cross-variable edges.}
To capture the spatiotemporal dependencies across different series, we introduce an edge set $\mathcal{E}$ to define the cross-variable graph topology. Guided by Theorem~\ref{thm:spatiotemporal_nll}, rather than explicitly estimating the intractable precision matrix, we utilize $\mathcal{E}$ as a practical surrogate structure to regularize residual discrepancies. This formulation is highly flexible and can easily accommodate different structural priors. Specifically, we design three variants for constructing $\mathcal{E}$: (1) \emph{Synchronous cross-variable edges}, which strictly capture concurrent interactions at the same time step (i.e., connecting identical patch indices); (2) \emph{Asynchronous cross-variable edges}, which exclusively model lagged dependencies across different time steps; and (3) \emph{Fully connected cross-variable edges}, which include both synchronous and lagged interactions. We adopt the fully connected graph as our default setting to maximize the coverage of potential interactions without relying on domain-specific prior knowledge. We evaluate these topological designs in the ablation studies in Section~\ref{sec:ablation}.

\paragraph{Difference operator.}
Fix an arbitrary orientation for each undirected edge and let \(A\in\mathbb{R}^{|\mathcal{E}|\times N}\) be the incidence matrix: for an edge \(k=(i,j)\), the \(k\)-th row contains \(+1\) at position \(i\), \(-1\) at position \(j\), and \(0\) elsewhere. Since each node is an \(L\)-dimensional patch, we lift the incidence operator as
\begin{equation}
\tilde A := A\otimes I_L.
\end{equation}
Let \(\mathrm{vec}(E)\in\mathbb{R}^{NL}\) denote the vectorization of \(E\). Then \(\tilde A\,\mathrm{vec}(E)\in\mathbb{R}^{|\mathcal{E}|L}\) stacks the residual differences \(\mathbf{e}_i-\mathbf{e}_j\in\mathbb{R}^{L}\) over all edges.

\subsubsection{Cross-Variable Loss}

Building upon the defined graph topology, we now formalize our objective. First, by reshaping the predictions into patch matrices, the standard DF loss can be equivalently rewritten as
\begin{equation}
\mathcal{L}_{\mathrm{df}}
=
\frac{1}{N\times L}\|\hat Z-Z\|_F^2
=
\frac{1}{T\times D}\|\hat Y-Y\|_F^2.
\end{equation}
To capture the relative structural geometry between variables, for each edge $(i,j)\in\mathcal{E}$, we define the predicted and true patch differences as
\begin{equation}
\Delta^{\mathrm{pred}}_{ij}=\hat{\mathbf{z}}_i-\hat{\mathbf{z}}_j,
\qquad
\Delta^{\mathrm{true}}_{ij}=\mathbf{z}_i-\mathbf{z}_j.
\end{equation}
We then define CvLoss as the mean absolute discrepancy of these structural differences:
\begin{equation}
\mathcal{L}_{\mathrm{cv}}
:=
\frac{1}{|\mathcal{E}|\times L}
\sum_{(i,j)\in\mathcal{E}}
\bigl\|
\Delta^{\mathrm{pred}}_{ij}
-
\Delta^{\mathrm{true}}_{ij}
\bigr\|_1.
\end{equation}

Using the node-wise residual $\mathbf{e}_i=\hat{\mathbf{z}}_i-\mathbf{z}_i$, the structural discrepancy on each edge simplifies neatly:
\begin{equation}
\Delta^{\mathrm{pred}}_{ij}-\Delta^{\mathrm{true}}_{ij}
=
(\hat{\mathbf{z}}_i-\hat{\mathbf{z}}_j)-(\mathbf{z}_i-\mathbf{z}_j)
=
\mathbf{e}_i-\mathbf{e}_j.
\end{equation}
Therefore, CvLoss admits a compact graph-theoretic form via the incidence matrix $\tilde{A}$:
\begin{equation}\label{eq:cvl}
\mathcal{L}_{\mathrm{cv}}
=
\frac{1}{|\mathcal{E}|\times L}
\sum_{(i,j)\in\mathcal{E}}
\|\mathbf{e}_i-\mathbf{e}_j\|_1
=
\frac{1}{|\mathcal{E}|\times L}
\|\tilde A\,\mathrm{vec}(E)\|_1.
\end{equation}
This computes the total variation of prediction errors across the graph.

\begin{figure*}[t]
    \centering
    \includegraphics[width=1\linewidth]{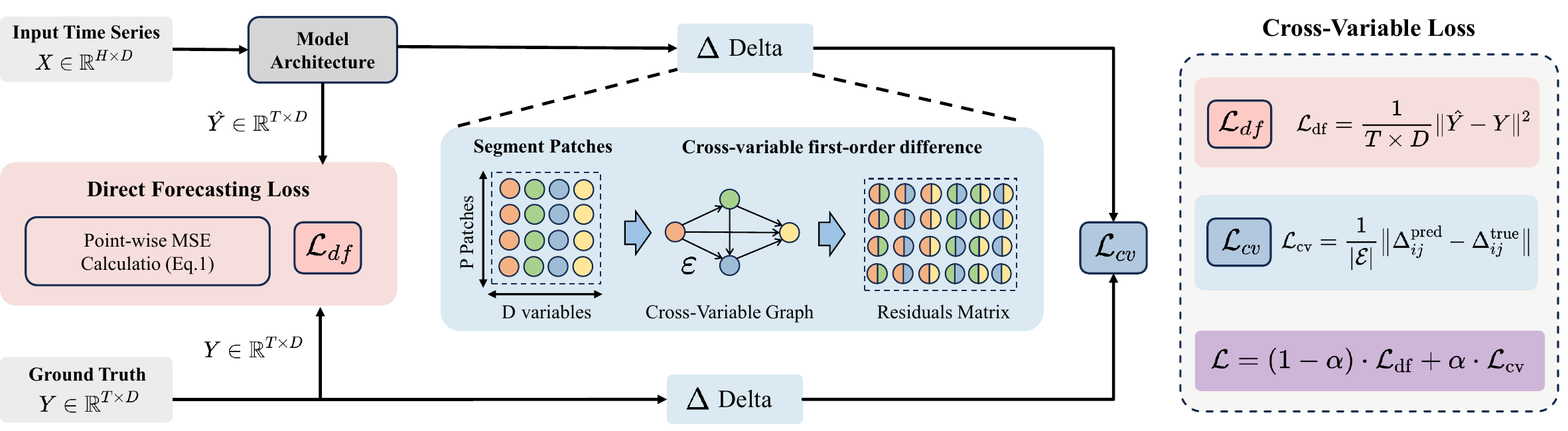}
    \caption{Overview of the workflow CvLoss.}
    \label{fig:framework}
\end{figure*}

\subsubsection{Interpretation}

The role of CvLoss is to explicitly regularize the cross-variable residual structure that is not directly controlled by the point-wise DF loss. In particular, \(\mathcal{L}_{\mathrm{df}}\) encourages each residual \(\mathbf{e}_v\) to be small in absolute magnitude, while \(\mathcal{L}_{\mathrm{cv}}\) encourages connected nodes to have compatible residual patterns. This induces an explicit coupling across variables at the objective level.

From the graph perspective, \(\mathcal{L}_{\mathrm{cv}}\) is an \(\ell_1\)-type graph total variation regularizer on the residual field. Unlike point-wise MSE, which evaluates each entry independently, CvLoss constrains the edge-wise discrepancies of prediction errors across variables. This encourages the model to exploit shared temporal structure and improves coordination among related variables.

We use the \(\ell_1\) norm because it is less sensitive to a few large edge-wise mismatches caused by localized anomalies or variable-specific shocks, while still promoting graph-level residual consistency. Replacing the squared penalty of Proposition~\ref{prop:precision} by \(\ell_1\) yields a pairwise Gibbs field over the same graph rather than a Gaussian one, so the \(\ell_1\) form is a task-specific empirical choice rather than a consequence of the proposition; Appendix~\ref{app:norm_choice} compares the two norms directly under an identical protocol.

\subsection{Model Implementation}

Our framework augments vanilla DF with the CvLoss constraint, complementing point-wise supervision with cross-variable regularization.

As shown in \autoref{fig:framework}, the normalized history $X$ is fed into the model to produce $\hat{Y}=f_\theta(X)$. We compute the point-wise error $\mathcal{L}_{\mathrm{df}}$, map forecasts and labels into patch nodes, and compute the structural error $\mathcal{L}_{\mathrm{cv}}$ over $\mathcal{E}$. Since the incidence operations are linear, $\mathcal{L}_{\mathrm{cv}}$ is optimized end-to-end.

Finally, the independent DF error and the structural cross-variable error are fused. A weighting parameter $\alpha \in (0,1)$ controls the relative contribution of each objective:
\begin{equation}\label{eq:cvl_alpha}
\mathcal{L}_{\alpha} := (1-\alpha)\cdot \mathcal{L}_{\mathrm{df}} + \alpha \cdot \mathcal{L}_{\mathrm{cv}}.
\end{equation}
Throughout the paper $\alpha$ is a \emph{fixed} scalar selected on the validation split, together with the patch length $L$ the only quantity we tune; it is never learned during training and never varies across batches, variables or edges. Appendix~\ref{app:tuning} records the selection protocol, and Appendix~\ref{app:diagnostic_weights} reports a separate diagnostic study whose adaptive coefficients are not values of $\alpha$.

By sharing statistical strength across the multivariate series, our framework complements the DF paradigm with explicit cross-variable regularization while preserving its practical advantages, including efficient inference and multi-task learning capabilities. It is worth noting that the standard $\mathcal{L}_{\mathrm{df}}$ term serves as a mathematically necessary absolute anchor. It prevents uniform drift and trivial translation-invariant solutions (where $\mathbf{e}_i = \mathbf{e}_j = \mathbf{c} \neq \mathbf{0}$) even when the structural penalty dominates (i.e., $\alpha \to 1$). Furthermore, our proposed regularization scheme is entirely model-agnostic, seamlessly compatible with various forecasting architectures $f_\theta$ (e.g., Transformers, CNNs, and MLPs).

\section{Experiments}

We evaluate forecast accuracy, learning-objective comparisons, generalization across backbones, ablations and hyperparameter sensitivity. Appendix~\ref{sec:comp_app} examines training cost and inference latency.

\subsection{Setup}
\paragraph{Datasets.} 
Following~\citep{Timesnet}, we evaluate on standard benchmarks: four ETT subsets, ECL, Traffic, Weather, and PEMS~\citep{itransformer}. All datasets are chronologically split into training, validation, and test sets. Statistics are in Appendix~\ref{sec:dataset}.

\paragraph{Baselines.} We compare \time\ to a range of competitive baselines, categorized as: (1) Transformer-based models: PatchTST~\citep{PatchTST}, iTransformer~\citep{itransformer},  TQNet~\citep{lin2025TQNet}, Crossformer~\citep{crossformer}, FEDformer~\citep{fedformer} and Leddam~\citep{leddam}; (2) Non-transformer models: DLinear~\citep{DLinear}, SOFTS~\citep{softs}, TimesNet~\citep{Timesnet} and TimeFilter~\citep{hu2025timefilter}.

\paragraph{Implementation.} Baseline implementations closely follow the official codebase from TimeFilter~\citep{hu2025timefilter} and TQNet~\citep{lin2025TQNet}. To ensure fair comparison, the drop-last trick is disabled for all models, as recommended in ~\citep{qiutfb}. All models are trained with the Adam optimizer~\citep{Adam}. When integrating \time\ into a baseline forecast model, we retain all hyperparameters from the public benchmarks~\citep{hu2025timefilter,lin2025TQNet}, only tuning $\alpha$ and the patch size. Experiments are run on Intel(R) Xeon(R) Platinum 8470Q with 8 NVIDIA RTX 5090 GPUs. Further implementation details are provided in Appendix~\ref{sec:reproduce}.

\subsection{Overall performance}\label{sec:overall}

\begin{table*}[t]
  \caption{Overall forecasting performance.}\label{tab:overall}
  \renewcommand{\arraystretch}{1} 
  \setlength{\tabcolsep}{2.2pt}
  \centering
  \scriptsize
  \renewcommand{\multirowsetup}{\centering}
  \begin{threeparttable}
  \begin{tabular}{c|c|cc|cc|cc|cc|cc|cc|cc|cc|cc|cc|cc}
    \toprule
    \multicolumn{2}{l}{\multirow{2}{*}{\rotatebox{0}{\scaleb{Models}}}} & 
    \multicolumn{2}{c}{\rotatebox{0}{\scaleb{\textbf{\time}}}} &
    \multicolumn{2}{c}{\rotatebox{0}{\scaleb{TimeFilter}}} &
    \multicolumn{2}{c}{\rotatebox{0}{\scaleb{TQNet}}} &
    \multicolumn{2}{c}{\rotatebox{0}{\scaleb{iTransformer}}} &
    \multicolumn{2}{c}{\rotatebox{0}{\scaleb{Leddam}}} &
    \multicolumn{2}{c}{\rotatebox{0}{\scaleb{SOFTS}}} &
    \multicolumn{2}{c}{\rotatebox{0}{\scaleb{PatchTST}}} &
    \multicolumn{2}{c}{\rotatebox{0}{\scaleb{Crossformer}}} &
    \multicolumn{2}{c}{\rotatebox{0}{\scaleb{TimesNet}}} &
    \multicolumn{2}{c}{\rotatebox{0}{\scaleb{DLinear}}} &
    \multicolumn{2}{c}{\rotatebox{0}{\scaleb{FEDformer}}} \\
    \multicolumn{2}{c}{} &
    \multicolumn{2}{c}{\scaleb{\textbf{(Ours)}}} & 
    \multicolumn{2}{c}{\scaleb{(2025)}} & 
    \multicolumn{2}{c}{\scaleb{(2025)}} & 
    \multicolumn{2}{c}{\scaleb{(2024)}} & 
    \multicolumn{2}{c}{\scaleb{(2024)}} & 
    \multicolumn{2}{c}{\scaleb{(2023)}} & 
    \multicolumn{2}{c}{\scaleb{(2023)}} & 
    \multicolumn{2}{c}{\scaleb{(2023)}} & 
    \multicolumn{2}{c}{\scaleb{(2023)}} & 
    \multicolumn{2}{c}{\scaleb{(2023)}} & 
    \multicolumn{2}{c}{\scaleb{(2022)}} \\
    \cmidrule(lr){3-4} \cmidrule(lr){5-6}\cmidrule(lr){7-8} \cmidrule(lr){9-10}\cmidrule(lr){11-12} \cmidrule(lr){13-14} \cmidrule(lr){15-16} \cmidrule(lr){17-18} \cmidrule(lr){19-20} \cmidrule(lr){21-22} \cmidrule(lr){23-24}
    \multicolumn{2}{l}{\rotatebox{0}{\scaleb{Metrics}}}  & \scalea{MSE} & \scalea{MAE}  & \scalea{MSE} & \scalea{MAE}  & \scalea{MSE} & \scalea{MAE}  & \scalea{MSE} & \scalea{MAE}  & \scalea{MSE} & \scalea{MAE}  & \scalea{MSE} & \scalea{MAE} & \scalea{MSE} & \scalea{MAE} & \scalea{MSE} & \scalea{MAE} & \scalea{MSE} & \scalea{MAE} & \scalea{MSE} & \scalea{MAE} & \scalea{MSE} & \scalea{MAE} \\
    \midrule
    \multicolumn{2}{l}{\scalea{ETTm1}} & \bst{\scalea{0.372}} & \bst{\scalea{0.381}} & \subbst{\scalea{0.377}} & \subbst{\scalea{0.393}} & \subbst{\scalea{0.377}} & \subbst{\scalea{0.393}} & \scalea{0.407} & \scalea{0.410} & \scalea{0.386} & \scalea{0.397} & \scalea{0.393} & \scalea{0.403} & \scalea{0.387} & \scalea{0.400} & \scalea{0.513} & \scalea{0.496} & \scalea{0.400} & \scalea{0.406} & \scalea{0.404} & \scalea{0.408} & \scalea{0.448} & \scalea{0.452} \\
    \midrule
    \multicolumn{2}{l}{\scalea{ETTm2}} & \bst{\scalea{0.270}} & \bst{\scalea{0.317}} & \subbst{\scalea{0.272}} & \subbst{\scalea{0.321}} & \scalea{0.278} & \scalea{0.322} & \scalea{0.288} & \scalea{0.332} & \scalea{0.281} & \scalea{0.325} & \scalea{0.287} & \scalea{0.330} & \scalea{0.281} & \scalea{0.326} & \scalea{0.757} & \scalea{0.610} & \scalea{0.291} & \scalea{0.333} & \scalea{0.354} & \scalea{0.402} & \scalea{0.305} & \scalea{0.349} \\
    \midrule
    \multicolumn{2}{l}{\scalea{ETTh1}} & \bst{\scalea{0.419}} & \bst{\scalea{0.427}} & \subbst{\scalea{0.420}} & \subbst{\scalea{0.428}} & \scalea{0.441} & \scalea{0.434} & \scalea{0.454} & \scalea{0.447} & \scalea{0.431} & \scalea{0.429} & \scalea{0.449} & \scalea{0.442} & \scalea{0.469} & \scalea{0.454} & \scalea{0.529} & \scalea{0.522} & \scalea{0.458} & \scalea{0.450} & \scalea{0.461} & \scalea{0.457} & \scalea{0.440} & \scalea{0.460} \\
    \midrule
    \multicolumn{2}{l}{\scalea{ETTh2}} & \bst{\scalea{0.359}} & \bst{\scalea{0.391}} & \subbst{\scalea{0.365}} & \subbst{\scalea{0.398}} & \scalea{0.378} & \scalea{0.403} & \scalea{0.383} & \scalea{0.407} & \scalea{0.373} & \scalea{0.399} & \scalea{0.385} & \scalea{0.408} & \scalea{0.387} & \scalea{0.407} & \scalea{0.942} & \scalea{0.684} & \scalea{0.414} & \scalea{0.427} & \scalea{0.563} & \scalea{0.519} & \scalea{0.437} & \scalea{0.449} \\
    \midrule
    \multicolumn{2}{l}{\scalea{Weather}} & \bst{\scalea{0.236}} & \bst{\scalea{0.260}} & \subbst{\scalea{0.240}} & \scalea{0.270} & \scalea{0.242} & \subbst{\scalea{0.268}} & \scalea{0.258} & \scalea{0.279} & \scalea{0.242} & \scalea{0.272} & \scalea{0.255} & \scalea{0.278} & \scalea{0.259} & \scalea{0.281} & \scalea{0.259} & \scalea{0.315} & \scalea{0.259} & \scalea{0.287} & \scalea{0.265} & \scalea{0.315} & \scalea{0.309} & \scalea{0.360} \\
    \midrule
    \multicolumn{2}{l}{\scalea{ECL}} & \bst{\scalea{0.157}} & \bst{\scalea{0.252}} & \subbst{\scalea{0.159}} & \subbst{\scalea{0.256}} & \scalea{0.164} & \scalea{0.258} & \scalea{0.178} & \scalea{0.270} & \scalea{0.169} & \scalea{0.263} & \scalea{0.174} & \scalea{0.264} & \scalea{0.216} & \scalea{0.304} & \scalea{0.244} & \scalea{0.334} & \scalea{0.193} & \scalea{0.295} & \scalea{0.225} & \scalea{0.319} & \scalea{0.214} & \scalea{0.327} \\
    \midrule
    \multicolumn{2}{l}{\scalea{Traffic}} & \bst{\scalea{0.407}} & \bst{\scalea{0.254}} & \subbst{\scalea{0.408}} & \scalea{0.269} & \scalea{0.445} & \scalea{0.276} & \scalea{0.428} & \scalea{0.282} & \scalea{0.467} & \scalea{0.294} & \subbst{\scalea{0.409}} & \subbst{\scalea{0.267}} & \scalea{0.555} & \scalea{0.362} & \scalea{0.550} & \scalea{0.304} & \scalea{0.620} & \scalea{0.336} & \scalea{0.625} & \scalea{0.383} & \scalea{0.610} & \scalea{0.376} \\
    \midrule
    \multicolumn{2}{l}{\scalea{Solar}} & \subbst{\scalea{0.218}} & \bst{\scalea{0.254}} & \scalea{0.228} & \subbst{\scalea{0.262}} & \bst{\scalea{0.197}} & \scalea{0.255} & \scalea{0.233} & \scalea{0.262} & \scalea{0.230} & \scalea{0.264} & \scalea{0.229} & \scalea{0.256} & \scalea{0.270} & \scalea{0.307} & \scalea{0.641} & \scalea{0.639} & \scalea{0.301} & \scalea{0.319} & \scalea{0.330} & \scalea{0.401} & \scalea{0.292} & \scalea{0.381} \\
    \midrule
    \multicolumn{2}{l}{\scalea{PEMS03}} & \bst{\scalea{0.079}} & \bst{\scalea{0.183}} & \scalea{0.084} & \scalea{0.191} & \subbst{\scalea{0.080}} & \subbst{\scalea{0.186}} & \scalea{0.096} & \scalea{0.204} & \scalea{0.101} & \scalea{0.210} & \scalea{0.087} & \scalea{0.192} & \scalea{0.151} & \scalea{0.265} & \scalea{0.138} & \scalea{0.253} & \scalea{0.119} & \scalea{0.271} & \scalea{0.219} & \scalea{0.295} & \scalea{0.167} & \scalea{0.291} \\
    \midrule
    \multicolumn{2}{l}{\scalea{PEMS04}} & \bst{\scalea{0.078}} & \bst{\scalea{0.181}} & \scalea{0.083} & \scalea{0.186} & \subbst{\scalea{0.081}} & \subbst{\scalea{0.185}} & \scalea{0.098} & \scalea{0.207} & \scalea{0.102} & \scalea{0.213} & \scalea{0.091} & \scalea{0.196} & \scalea{0.162} & \scalea{0.273} & \scalea{0.145} & \scalea{0.267} & \scalea{0.109} & \scalea{0.220} & \scalea{0.236} & \scalea{0.350} & \scalea{0.195} & \scalea{0.308} \\
    \midrule
    \multicolumn{2}{l}{\scalea{PEMS07}} & \bst{\scalea{0.063}} & \bst{\scalea{0.156}} & \scalea{0.071} & \scalea{0.170} & \subbst{\scalea{0.065}} & \subbst{\scalea{0.161}} & \scalea{0.088} & \scalea{0.190} & \scalea{0.087} & \scalea{0.192} & \scalea{0.075} & \scalea{0.173} & \scalea{0.166} & \scalea{0.270} & \scalea{0.181} & \scalea{0.272} & \scalea{0.106} & \scalea{0.208} & \scalea{0.241} & \scalea{0.343} & \scalea{0.133} & \scalea{0.282} \\
    \midrule
    \multicolumn{2}{l}{\scalea{PEMS08}} & \bst{\scalea{0.080}} & \bst{\scalea{0.182}} & \subbst{\scalea{0.083}} & \subbst{\scalea{0.186}} & \scalea{0.105} & \scalea{0.203} & \scalea{0.127} & \scalea{0.212} & \scalea{0.102} & \scalea{0.211} & \scalea{0.114} & \scalea{0.208} & \scalea{0.238} & \scalea{0.289} & \scalea{0.232} & \scalea{0.270} & \scalea{0.150} & \scalea{0.244} & \scalea{0.347} & \scalea{0.421} & \scalea{0.234} & \scalea{0.326} \\
    \bottomrule
  \end{tabular}
  \begin{tablenotes}
    \item  \scriptsize \textit{Note}: We fix the input length as 96 for all baselines following~\citep{itransformer}. \bst{Bold} and \subbst{underlined} denote best and second-best results, respectively. \emph{Avg} indicates average results over multiple forecast horizons. CvLoss using TimeFilter~\citep{hu2025timefilter} as backbone.
    
\end{tablenotes}
\end{threeparttable}
\end{table*}

\begin{table}[t]
  \caption{\textbf{Backbone-controlled summary of \time.} Every comparison contrasts a forecasting backbone trained with its standard objective against the \emph{identical} backbone, data split, and training protocol trained with \time, so the loss term is the only variable. A \emph{cell} is one (backbone, dataset, metric) triple evaluated at the dataset-level average over forecast horizons; W/T/L counts cells on which \time\ is better, equal, or worse, and the last two columns give the mean relative error reduction over the cells of that block.}
  \label{tab:controlled}
  \centering
  \setlength{\tabcolsep}{4.5pt}
  \renewcommand{\arraystretch}{1.05}
  \small
  \begin{threeparttable}
  \begin{tabular}{llccccc}
    \toprule
    \multirow{2}{*}{Evidence} & \multirow{2}{*}{Backbone(s) held fixed} & \multirow{2}{*}{Datasets} & \multirow{2}{*}{Cells} & \multirow{2}{*}{W\,/\,T\,/\,L} & \multicolumn{2}{c}{Mean reduction (\%)} \\
    \cmidrule(lr){6-7}
    & & & & & MSE & MAE \\
    \midrule
    Table~\ref{tab:overall}                    & TimeFilter                            & 12 & 24 & \best{24}\,/\,0\,/\,0 & 3.20 & 3.12 \\
    Table~\ref{tab:loss_avg}                   & TQNet, PDF                            & 4  & 16 & \best{15}\,/\,0\,/\,1 & 2.50 & 2.06 \\
    Table~\ref{tab:system_ablation_app}        & TQNet                                 & 4  & 8  & \best{8}\,/\,0\,/\,0  & 1.16 & 1.66 \\
    Table~\ref{tab2}                           & TQNet, TimeFilter, CFPT, iTransformer & 3  & 24 & \best{24}\,/\,0\,/\,0 & 1.61 & 2.17 \\
    Table~\ref{tab:comprehensive_all_variants} & iTransformer, DLinear, PatchTST       & 7  & 42 & \best{40}\,/\,2\,/\,0 & 5.29 & 5.66 \\
    \midrule
    \rowc \textbf{Total} & \textbf{7 distinct backbones} & \textbf{12} & \textbf{114} & \textbf{111\,/\,2\,/\,1} & \textbf{3.39} & \textbf{3.61} \\
    \bottomrule
  \end{tabular}
  \begin{tablenotes}
    \item \scriptsize \textit{Note}: Counts summarize overlapping evaluation blocks under their respective protocols and are descriptive rather than a pooled statistical test. The only loss is TQNet/ETTh1 MAE in Table~\ref{tab:loss_avg} ($0.434\!\rightarrow\!0.438$); the two ties are DLinear/Weather MSE and PatchTST/ECL MSE in Table~\ref{tab:comprehensive_all_variants}, equal at three-decimal precision.
  \end{tablenotes}
  \end{threeparttable}
\end{table}

Table~\ref{tab:overall} compares long- and short-term forecasting results. Adding \time\ consistently improves accuracy across datasets; on Traffic, for example, MAE decreases by 0.014. The gains come from explicit structural regularization over synchronous and asynchronous cross-variable dependencies.

\paragraph{Backbone-controlled attribution.} In Table~\ref{tab:overall}, CvLoss augments TimeFilter with the same data and training protocol and improves all $24$ dataset-level entries ($12$ datasets $\times\,2$ metrics). The other columns establish the competitiveness of the backbone. Table~\ref{tab:controlled} consolidates $114$ fixed-backbone comparisons across seven backbones, three architecture families, twelve datasets and both metrics: \time\ improves $111$, ties $2$ and loses $1$, with mean relative reductions of $3.39\%$ in MSE and $3.61\%$ in MAE. At the horizon level, it improves all $96$ comparisons in Table~\ref{tab2} and $31$ of $32$ in Table~\ref{tab:system_ablation_app}, with one tie.

These blocks retain their respective evaluation protocols and overlap in backbones and datasets; the aggregate summarizes evidence rather than a pooled significance test. Appendix~\ref{app:paired} reports paired effect sizes and confidence intervals. Gains depend on the cross-variable structure left unmodeled by each backbone: Traffic with iTransformer improves MSE by $17.7\%$, PEMS07 with TimeFilter by $11.3\%$, and ETTh1 with TimeFilter by $0.24\%$. Section~\ref{sec:hyper} and Appendix~\ref{app:heatmap_vis} examine this dependence.

\paragraph{Examples.} A qualitative comparison between forecasts generated by DF versus \time\ is presented in Figure~\ref{fig:case}. While the model trained with DF captures general periodic trends, it often fails to model accurate amplitudes and intricate temporal dynamics. For example, on the ETTm2 dataset, it suffers from severe amplitude attenuation, completely missing the prominent ground-truth peaks. Similarly, on ECL, it overestimates the high periodic peaks, such as those observed around the 1750th and 250th steps. In contrast, \time\ accurately captures these large peak magnitudes and complex fluctuations.

\begin{figure}
\begin{center}
\subfigure[ETTm2 snapshot.]{\includegraphics[width=0.24\linewidth]{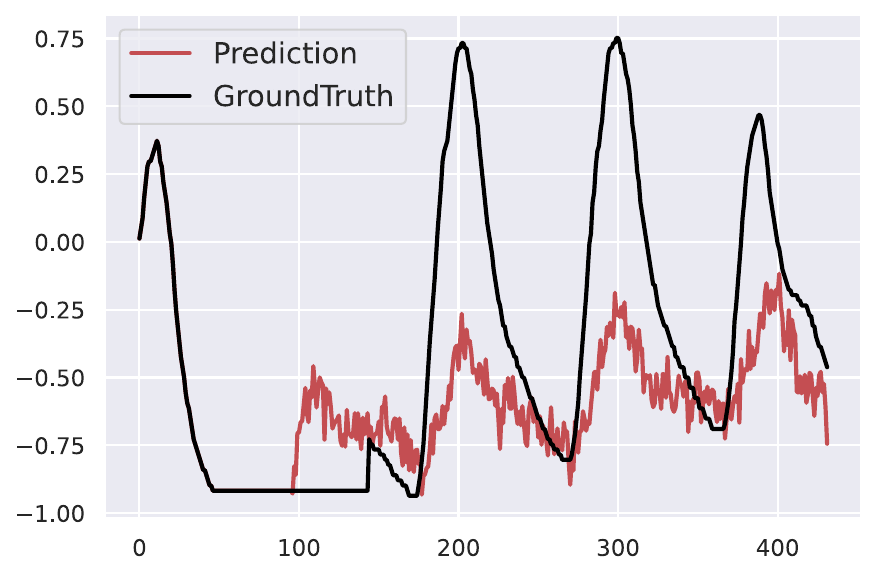}
\includegraphics[width=0.24\linewidth]{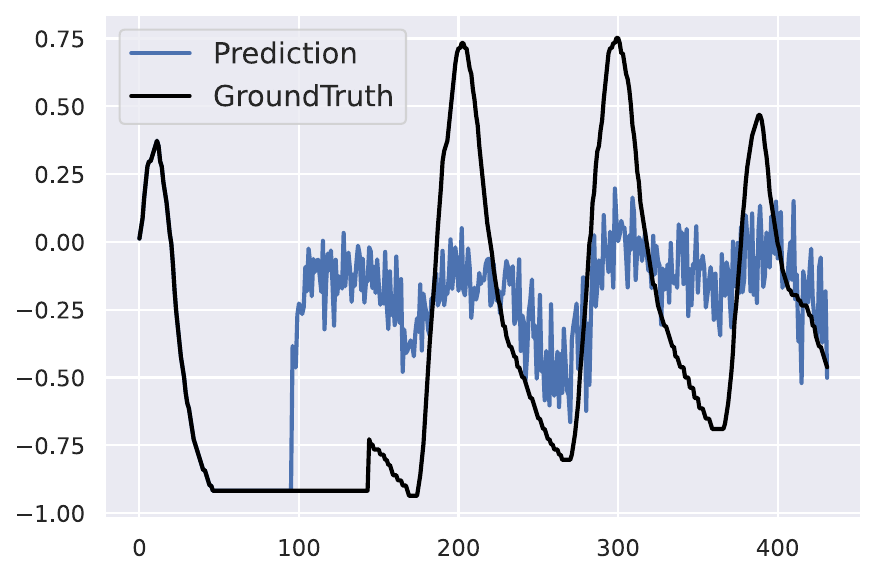}
}
\subfigure[ECL snapshot.]{\includegraphics[width=0.24\linewidth]{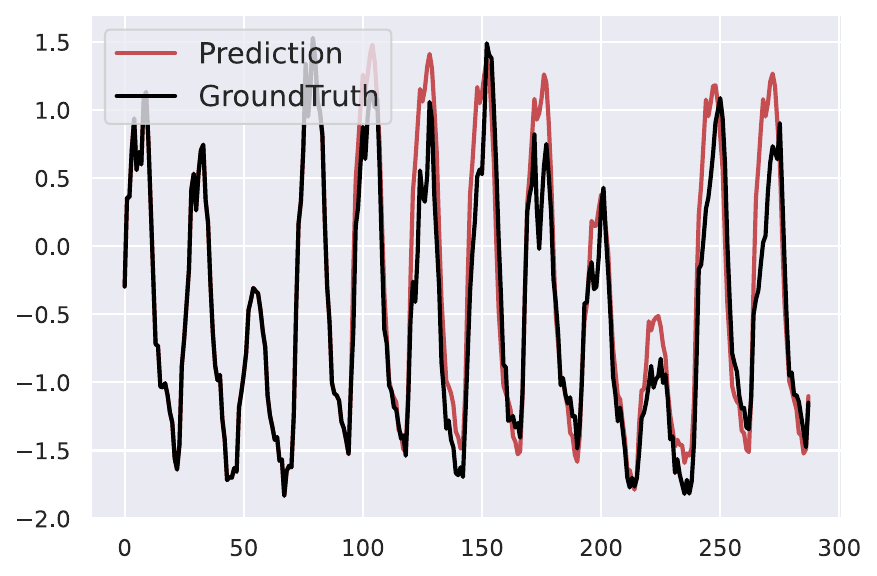}
\includegraphics[width=0.24\linewidth]{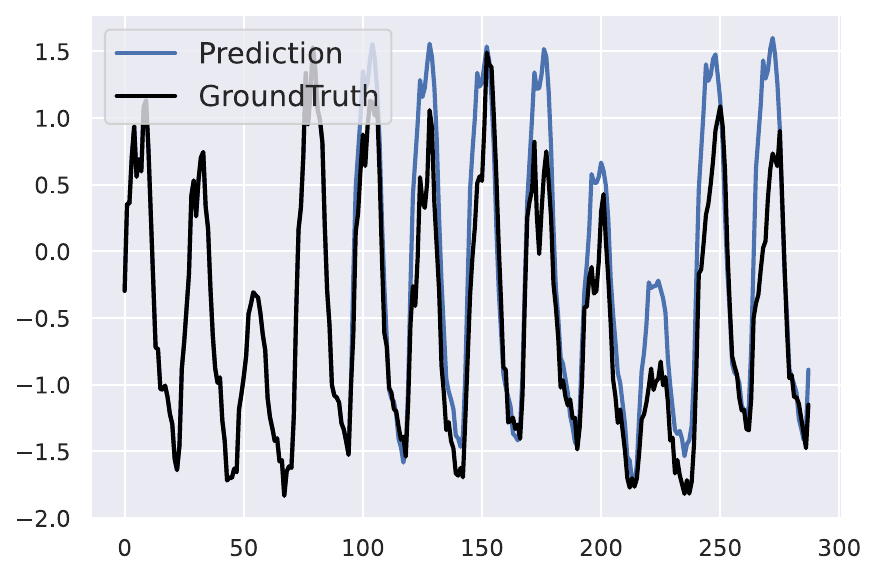}
}
\caption{The forecast sequence of DF (in blue) and \time\ (in red), with historical length $\mathrm{H}=96$.}\label{fig:case}
\end{center}
\end{figure}

\subsection{Learning objective comparison}
\label{sec:compete}
\begin{table}[t]
\centering
\begin{threeparttable}
\caption{Comparable results with other learning objectives for time-series forecasting.}\label{tab:loss_avg}
\renewcommand{\arraystretch}{1}
\setlength{\tabcolsep}{4.5pt} 
\scriptsize
\renewcommand{\multirowsetup}{\centering}
\begin{tabular}{c|l|cc|cc|cc|cc|cc|cc|cc}
\toprule
\multicolumn{2}{l}{Loss} & 
\multicolumn{2}{c}{\textbf{Ours}} &
\multicolumn{2}{c}{QDF} &
\multicolumn{2}{c}{Time-o1} &
\multicolumn{2}{c}{FreDF} &
\multicolumn{2}{c}{Koopman} &
\multicolumn{2}{c}{Soft-DTW} &
\multicolumn{2}{c}{DF} \\
\cmidrule(lr){3-4} \cmidrule(lr){5-6}\cmidrule(lr){7-8} \cmidrule(lr){9-10}\cmidrule(lr){11-12}\cmidrule(lr){13-14}\cmidrule(lr){15-16}
\multicolumn{2}{l}{Metrics}  & MSE & MAE  & MSE & MAE  & MSE & MAE  & MSE & MAE  & MSE & MAE  & MSE & MAE  & MSE & MAE  \\
\toprule

\multirow{4}{*}{{\rotatebox{90}{\scalebox{0.95}{TQNet}}}}
& ETTm1 & 0.373 & \bst{0.386} & \bst{0.371} & 0.389 & \subbst{0.372} & 0.390 & 0.375 & 0.390 & 0.499 & 0.595 & 0.394 & \subbst{0.387} & 0.377 & 0.393 \\
& ETTh1 & \bst{0.430} & 0.438 & \subbst{0.431} & \bst{0.431} & 0.432 & 0.437 & 0.432 & \subbst{0.432} & 0.442 & 0.451 & 0.453 & 0.438 & 0.441 & 0.434 \\
& ECL   & \bst{0.162} & \bst{0.253} & \subbst{0.165} & \subbst{0.257} & 0.167 & \subbst{0.257} & 0.168 & \subbst{0.257} & 0.166 & 0.258 & 0.524 & 0.623 & \subbst{0.165} & 0.259 \\
& Weather & \bst{0.241} & \bst{0.264} & \subbst{0.242} & \subbst{0.268} & 0.245 & 0.269 & 0.244 & \subbst{0.268} & 0.306 & 0.282 & 0.276 & 0.255 & \subbst{0.242} & \subbst{0.268} \\
\midrule

\multirow{4}{*}{{\rotatebox{90}{\scalebox{0.95}{PDF}}}}
& ETTm1 & \bst{0.378} & \bst{0.383} & \subbst{0.381} & \subbst{0.394} & 0.386 & 0.399 & 0.387 & 0.400 & 0.587 & 0.485 & 0.396 & 0.404 & 0.387 & 0.396 \\
& ETTh1 & \bst{0.429} & \subbst{0.432} & \subbst{0.436} & \bst{0.429} & 0.438 & 0.438 & 0.437 & 0.435 & 0.497 & 0.472 & 0.447 & 0.447 & 0.452 & 0.440 \\
& ECL   & \bst{0.189} & \bst{0.269} & \subbst{0.194} & 0.277 & 0.195 & 0.276 & \subbst{0.194} & \subbst{0.274} & 0.196 & 0.281 & 0.695 & 0.548 & 0.198 & 0.281 \\
& Weather & \bst{0.259} & \bst{0.276} & \bst{0.259} & \subbst{0.281} & \subbst{0.264} & 0.284 & 0.268 & 0.287 & 0.268 & 0.290 & 1.296 & 0.452 & 0.265 & 0.283 \\
\bottomrule
\end{tabular}
\begin{tablenotes}
\item  \tiny \textit{Note}:  \bst{Bold} and \subbst{underlined} denote best and second-best results, respectively. Follow the settings of QDF~\citep{wang2026iclrqdf}.
\end{tablenotes}
\end{threeparttable}
\end{table}

\time\ is compared with 11 learning objectives from four categories: shape alignment, likelihood maximization, distribution balancing, and decomposition. Table~\ref{tab:loss_avg} summarizes average performance on four datasets, with full horizon-level results in Appendix Table~\ref{tab:multistep_app_full}.

Shape-alignment losses (e.g., Soft-DTW~\citep{soft-dtw}, Koopman~\citep{koopman}) and likelihood objectives (e.g., Time-o1~\citep{wang2025nipstimeo1}, FreDF~\citep{wang2025fredf}, QDF~\citep{wang2026iclrqdf}) improve over MSE, but still treat multivariate outputs as scalar predictions. They do not explicitly enforce the cross-variable residual consistency motivated by Theorem~\ref{thm:spatiotemporal_nll}. By penalizing edge-wise residual discrepancies, \time\ ranks first in 20 settings (Appendix Table~\ref{tab:multistep_app_full}) and generalizes across TQNet~\citep{lin2025TQNet} and PDF~\citep{dai2024pdf} (Appendix Table~\ref{tab:loss-compare}).

\subsection{Generalization studies}\label{sec:generalize}

We test \time\ on TimeFilter, TQNet, CFPT~\citep{CFPT}, and iTransformer. As shown in \autoref{fig:backbone}, \time\ improves different architectures and datasets. On ECL, all four backbones improve, with iTransformer reducing MSE by 8.6\% and MAE by 7.9\%. On Weather, MSE drops by 1.3\%--2.3\% and MAE by 1.9\%--3.3\%, showing that \time\ is broadly plug-and-play.

\begin{figure}
\begin{center}
\subfigure[ECL with MSE]{\includegraphics[width=0.24\linewidth]{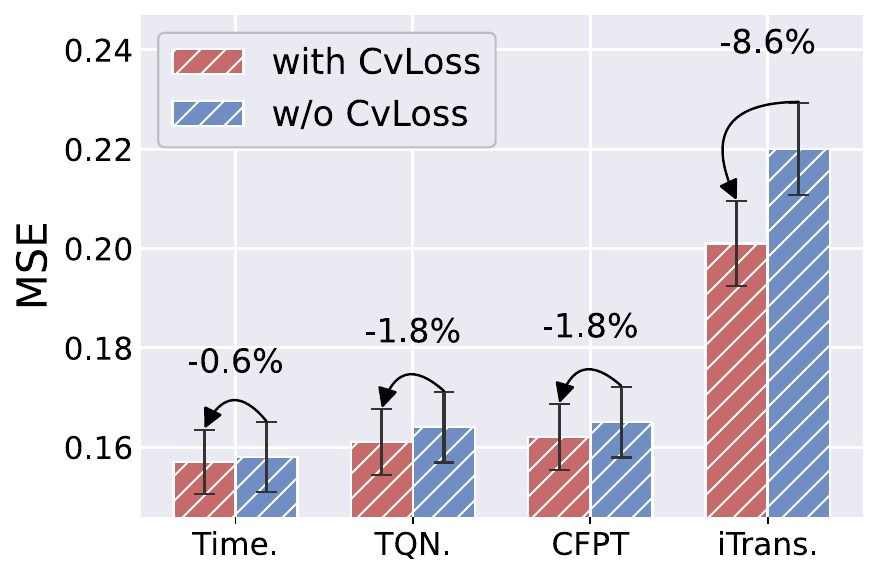}}
\subfigure[ECL with MAE]{\includegraphics[width=0.24\linewidth]{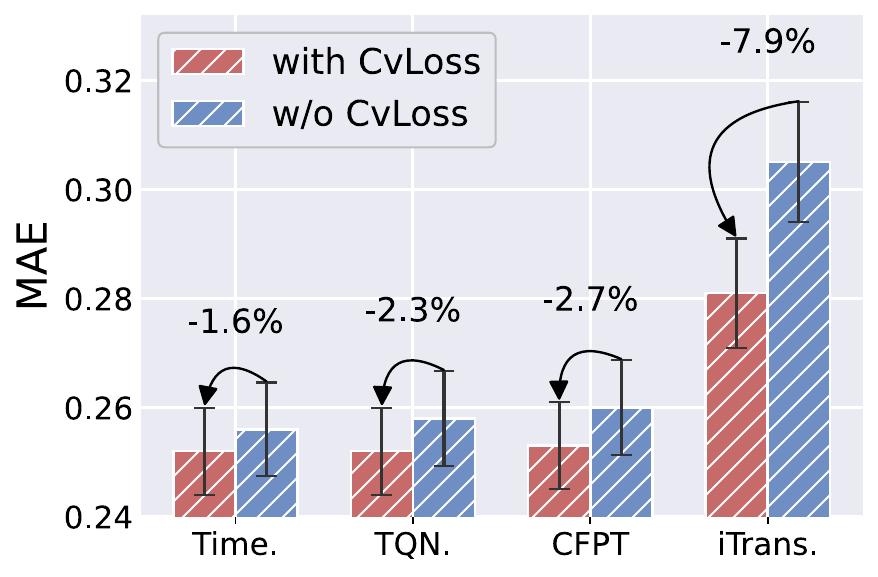}}
\subfigure[Weather with MSE]{\includegraphics[width=0.24\linewidth]{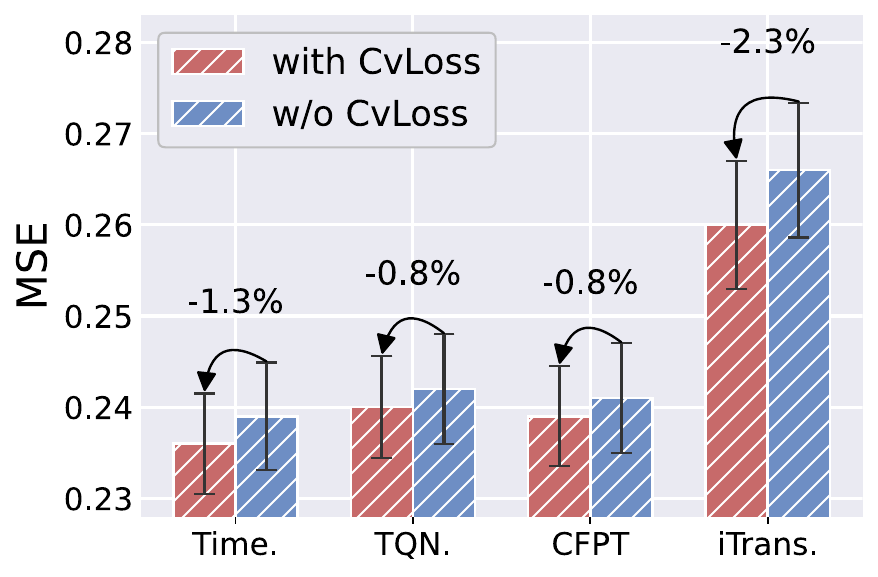}}
\subfigure[Weather with MAE]{\includegraphics[width=0.24\linewidth]{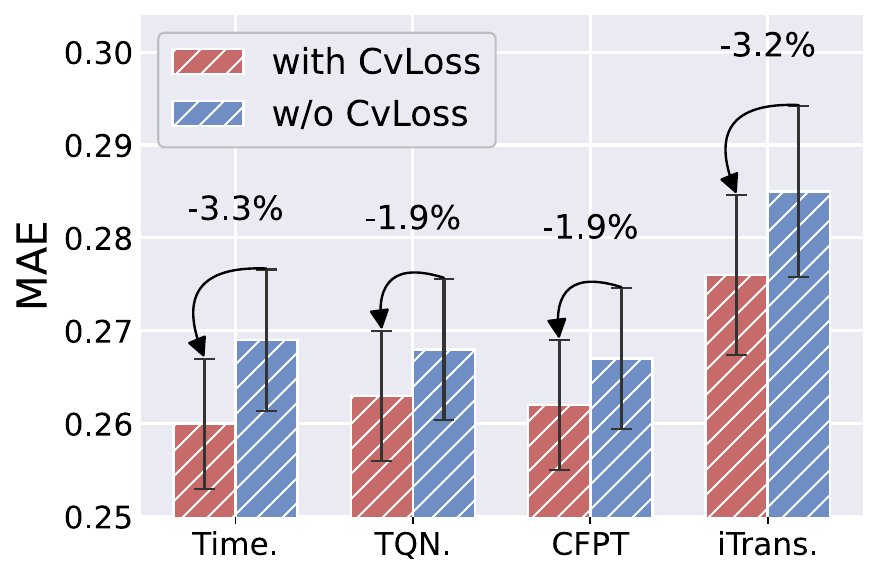}}
\caption{Improvement of \time\ applied to different forecast models, shown with colored bars for means over forecast lengths (96, 192, 336, 720) and error bars for 50\% confidence intervals. }
\label{fig:backbone}
\end{center}
\end{figure}

\begin{table}
\caption{Ablation study results.}\label{tab:system_ablation_app}
\setlength{\tabcolsep}{4pt}
\scriptsize
\centering
\begin{threeparttable}
\begin{tabular}{lcclccccccccccccccc}
    \toprule
    \multirow{2}{*}{Model} & \multirow{2}{*}{Sync.} & \multirow{2}{*}{Async.} &\multirow{2}{*}{Data} && \multicolumn{2}{c}{T=96} && \multicolumn{2}{c}{T=192} && \multicolumn{2}{c}{T=336} && \multicolumn{2}{c}{T=720} && \multicolumn{2}{c}{Avg} \\
    \cmidrule{6-7} \cmidrule{9-10} \cmidrule{12-13} \cmidrule{15-16} \cmidrule{18-19}
    &&&&& MSE  & MAE && MSE & MAE && MSE & MAE && MSE & MAE && MSE & MAE \\
    \midrule
    
\multirow{4}{*}{DF} & \multirow{4}{*}{\XSolidBrush} & \multirow{4}{*}{\XSolidBrush}
&   ETTm1 && 0.311 & 0.353 && 0.357 & 0.378 && 0.390 & 0.401 && 0.449 & 0.439 && 0.377 & 0.393 \\
&&& ETTh1 && 0.371 & 0.393 && 0.428 & 0.426 && 0.475 & 0.446 && \subbst{0.488} & \subbst{0.470} && \subbst{0.441} & \subbst{0.434} \\
&&& ECL && 0.134 & 0.229 && 0.154 & 0.247 && 0.170 & 0.265 && 0.200 & 0.293 && 0.165 & 0.259 \\
&&& Weather && 0.156 & 0.200 && 0.205 & 0.245 && \subbst{0.262} & 0.287 && 0.343 & 0.342 && 0.242 & 0.269 \\
\midrule

\multirow{4}{*}{{\time}$^\dagger$} & \multirow{4}{*}{\Checkmark} & \multirow{4}{*}{\XSolidBrush}
&   ETTm1 && 0.306 & 0.344 && 0.357 & \bst{0.372} && \subbst{0.388} & \bst{0.399} && 0.450 & 0.439 && 0.375 & 0.389 \\
&&& ETTh1 && 0.369 & 0.387 && \subbst{0.426} & 0.422 && 0.475 & 0.444 && 0.505 & 0.485 && 0.444 & 0.435 \\
&&& ECL && 0.134 & 0.228 && \subbst{0.152} & \subbst{0.243} && \subbst{0.169} & \subbst{0.261} && \subbst{0.199} & \subbst{0.289} && 0.164 & \subbst{0.255} \\
&&& Weather && 0.159 & 0.202 && 0.208 & 0.246 && 0.265 & 0.287 && 0.344 & 0.341 && 0.244 & 0.269 \\
\midrule

\multirow{4}{*}{{\time}$^\ddagger$} & \multirow{4}{*}{\XSolidBrush} & \multirow{4}{*}{\Checkmark}
&   ETTm1 && \subbst{0.305} & \subbst{0.343} && \subbst{0.353} & 0.373 && 0.390 & 0.400 && \subbst{0.448} & \subbst{0.431} && \subbst{0.374} & \subbst{0.387} \\
&&& ETTh1 && \subbst{0.369} & \subbst{0.387} && 0.428 & \subbst{0.422} && \subbst{0.473} & \subbst{0.444} && 0.505 & 0.485 && 0.444 & 0.435 \\
&&& ECL && \subbst{0.134} & \subbst{0.228} && 0.153 & 0.244 && 0.170 & 0.262 && 0.200 & 0.290 && \subbst{0.164} & 0.256 \\
&&& Weather && \subbst{0.156} & \subbst{0.196} && \subbst{0.203} & \subbst{0.241} && 0.263 & \subbst{0.286} && \subbst{0.343} & \subbst{0.340} && \subbst{0.241} & \subbst{0.266} \\
\midrule

\multirow{4}{*}{{\time}} & \multirow{4}{*}{\Checkmark} & \multirow{4}{*}{\Checkmark}
&   ETTm1 && \bst{0.305} & \bst{0.343} && \bst{0.352} & \subbst{0.373} && \bst{0.388} & \subbst{0.400} && \bst{0.444} & \bst{0.430} && \bst{0.372} & \bst{0.387} \\
&&& ETTh1 && \bst{0.368} & \bst{0.385} && \bst{0.426} & \bst{0.422} && \bst{0.471} & \bst{0.443} && \bst{0.486} & \bst{0.469} && \bst{0.438} & \bst{0.430} \\
&&& ECL && \bst{0.133} & \bst{0.225} && \bst{0.151} & \bst{0.241} && \bst{0.166} & \bst{0.258} && \bst{0.196} & \bst{0.287} && \bst{0.162} & \bst{0.253} \\
&&& Weather && \bst{0.154} & \bst{0.192} && \bst{0.203} & \bst{0.238} && \bst{0.262} & \bst{0.285} && \bst{0.342} & \bst{0.339} && \bst{0.240} & \bst{0.264} \\
    \bottomrule
\end{tabular}
\begin{tablenotes}
    \item  \tiny \textit{Note}:  \bst{Bold} and \subbst{underlined} denote best and second-best results, respectively. “Sync.” and “Async.” are abbreviations for synchronous and asynchronous.
\end{tablenotes}
\end{threeparttable}
\end{table}

\subsection{Ablation studies}\label{sec:ablation}

We ablate the edge set $\mathcal{E}$: \time$^\dagger$ includes synchronous cross-variable interactions at the same patch index; \time$^\ddagger$ includes lagged cross-variable interactions at different indices; and default \time\ includes both. Table~\ref{tab:system_ablation_app} shows that combining synchronous and lagged interactions gives the best overall forecasting performance.

\subsection{Hyperparameter sensitivity}\label{sec:hyper}
Tables~\ref{tab:sensi-TQNet} and \ref{tab:sensi-TimeFilter} report the sensitivity of CvLoss to the loss weight $\alpha$ on TQNet and TimeFilter. Increasing $\alpha$ from 0 yields consistent improvements, indicating the benefit of incorporating cross-variable structural dependencies into the learning objective across different backbones and datasets. For example, CvLoss reduces the MSE of TQNet from 0.379 to 0.365 on ETTh2, and that of TimeFilter from 0.159 to 0.157 on ECL. However, using the structural penalty ($\alpha=1$) is not always optimal; on ECL, both models perform best around $\alpha=0.5$ and degrade when $\alpha$ reaches 1. This suggests that CvLoss is most effective when combined with point-wise MSE, which anchors forecasting accuracy and prevents translation-invariant solutions under structural regularization.

\input{tab/table_hparam2}

\section{Conclusion}

This paper studies the overlooked output-side dependencies in multivariate time series forecasting. We show that the standard point-wise DF objective is mismatched with real-world spatiotemporal correlations, and introduce CvLoss to explicitly regularize cross-variable residual consistency. CvLoss is implemented as a graph total variation penalty over forecast patches, allowing both synchronous and lagged interactions to be modeled without changing the inference procedure. Extensive experiments show that CvLoss consistently improves SOTA forecasting backbones, outperforms competitive learning objectives, and stays efficient in practice. These results suggest that modeling the structural geometry of future variables is a simple but effective direction for improving multivariate forecasting.

\section{Limitations}\label{sec:limitations}

\time\ uses a fixed patching scheme and a uniform cross-variable graph. The complete graph includes both synchronous and lagged interactions but does not estimate heterogeneous edge strengths; adaptive graph weights and patch sizes offer directions for greater efficiency and interpretability. Random edge sampling provides a scalability option, while the reported accuracy results use the complete graph. Our evaluation concerns deterministic forecasting on regular benchmarks; probabilistic forecasting, irregular sampling~\citep{ji2026fine}, and missing-value settings remain future directions.

The Gaussian analysis motivates residual coupling, while Proposition~\ref{prop:precision} relates the squared graph penalty to a structured precision matrix. The implemented $\ell_1$ penalty instead defines a non-Gaussian pairwise Gibbs field. Appendix~\ref{app:norm_choice} compares both norms and supports the empirical choice of $\ell_1$. Top-$K$ selection in Appendix~\ref{app:complexity_and_selection} and the supports studied in Section~\ref{sec:ablation} examine graph variants, without learning a full weighted dependency structure.

\clearpage
\begin{ack}
\paragraph{Funding.} Hao Xue acknowledges support from the HKUST(GZ) Startup Fund (Grant No. G0101000338).
\paragraph{Competing interests.} The authors declare no competing interests.
\end{ack}

\bibliography{main}

@STRING{P_NIPS         = "Proc. Adv. Neural Inf. Process. Syst."}

@STRING{P_ICML         = "Proc. Int. Conf. Mach. Learn."}

@STRING{P_ICLR         = "Proc. Int. Conf. Learn. Represent."}

@STRING{P_ICASSP         = "Proc. IEEE Int. Conf. Acoust. Speech Signal Process."}

@inproceedings{
hu2025timefilter,
title={TimeFilter: Patch-Specific Spatial-Temporal Graph Filtration for Time Series Forecasting},
author={Yifan Hu and Guibin Zhang and Peiyuan Liu and Disen Lan and Naiqi Li and Dawei Cheng and Tao Dai and Shu-Tao Xia and Shirui Pan},
booktitle={Forty-second International Conference on Machine Learning},
year={2025}
}

@article{wang2024tssurvey,
  title={Deep Time Series Models: A Comprehensive Survey and Benchmark},
  author={Yuxuan Wang and Haixu Wu and Jiaxiang Dong and Yong Liu and Mingsheng Long and Jianmin Wang},
  booktitle={arXiv preprint arXiv:2407.13278},
  year={2024},
}

@article{
chen2023tsmixer,
title={{TSM}ixer: An All-{MLP} Architecture for Time Series Forecast-ing},
author={Si-An Chen and Chun-Liang Li and Sercan O Arik and Nathanael Christian Yoder and Tomas Pfister},
journal={Transactions on Machine Learning Research},
issn={2835-8856},
year={2023},
}

@inproceedings{Timesnet,
  title={TimesNet: Temporal 2D-Variation Modeling for General Time Series Analysis},
  author={Haixu Wu and Tengge Hu and Yong Liu and Hang Zhou and Jianmin Wang and Mingsheng Long},
  booktitle={International Conference on Learning Representations},
  year={2023},
}

@inproceedings{itransformer,
  title={iTransformer: Inverted Transformers Are Effective for Time Series Forecasting},
  author={Liu, Yong and Hu, Tengge and Zhang, Haoran and Wu, Haixu and Wang, Shiyu and Ma, Lintao and Long, Mingsheng},
  booktitle={International Conference on Learning Representations},
  year={2024}
}

@article{2026timemosaic,
    title={TimeMosaic: Temporal Heterogeneity Guided Time Series Forecasting via Adaptive Granularity Patch and Segment-wise Decoding},
    author={Ding, Kuiye and Fan, Fanda and Hou, Chunyi and Wang, Zheya and Wang, Lei and Yang, Zhengxin and Zhan, Jianfeng},
    journal={Proceedings of the AAAI Conference on Artificial Intelligence},
    year={2026},
    month={Mar.},
    volume={40},
    url={https://ojs.aaai.org/index.php/AAAI/article/view/39218},
    DOI={10.1609/aaai.v40i25.39218},
    number={25},
    pages={20790-20798}
}

@inproceedings{ding2025dualsg,
  title={DualSG: A Dual-Stream Explicit Semantic-Guided Multivariate Time Series Forecasting Framework},
  author={Ding, Kuiye and Fan, Fanda and Wang, Yao and Jian, Ruijie and Wang, Xiaorui and Gong, Luqi and Jiang, Yishan and Luo, Chunjie and Zhan, Jianfeng},
  booktitle={Proceedings of the 33rd ACM International Conference on Multimedia},
  pages={508--517},
  year={2025}
}

@inproceedings{
LIFT,
title={Rethinking Channel Dependence for Multivariate Time Series Forecasting: Learning from Leading Indicators},
author={Lifan Zhao and Yanyan Shen},
booktitle={The Twelfth International Conference on Learning Representations},
year={2024}
}

@inproceedings{Informer,
  title={Informer: Beyond efficient transformer for long sequence time-series forecasting},
  author    = {Haoyi Zhou and
               Shanghang Zhang and
               Jieqi Peng and
               Shuai Zhang and
               Jianxin Li and
               Hui Xiong and
               Wancai Zhang},
  booktitle={The Thirty-Fifth AAAI Conference on Artificial Intelligence},
  year={2021}
}

@inproceedings{Autoformer,
  title={Autoformer: Decomposition Transformers with {Auto-Correlation} for Long-Term Series Forecasting},
  author={Haixu Wu and Jiehui Xu and Jianmin Wang and Mingsheng Long},
  booktitle={Advances in Neural Information Processing Systems},
  year={2021}
}

@inproceedings{SCINet,
  title={SCINet: time series modeling and forecasting with sample convolution and interaction},
  author={Liu, Minhao and Zeng, Ailing and Chen, Muxi and Xu, Zhijian and Lai, Qiuxia and Ma, Lingna and Xu, Qiang},
  booktitle={Advances in Neural Information Processing Systems},
  year={2022}
}

@inproceedings{qiutfb,
title   = {TFB: Towards Comprehensive and Fair Benchmarking of Time Series Forecasting Methods},
author  = {Xiangfei Qiu and Jilin Hu and Lekui Zhou and Xingjian Wu and Junyang Du and Buang Zhang and Chenjuan Guo and Aoying Zhou and Christian S. Jensen and Zhenli Sheng and Bin Yang},
booktitle = {VLDB},
pages   = {2363--2377},
year    = {2024}
}

@inproceedings{
leddam,
title={Revitalizing Multivariate Time Series Forecasting: Learnable Decomposition with Inter-Series Dependencies and Intra-Series Variations Modeling},
author={Guoqi Yu and Jing Zou and Xiaowei Hu and Angelica I Aviles-Rivero and Jing Qin and Shujun Wang},
booktitle={Forty-first International Conference on Machine Learning},
year={2024}
}

@inproceedings{
wei2026tarfvae,
title={{TARFVAE}: Efficient One-Step Generative Time Series Forecasting via {TARFLOW} based {VAE}},
author={Jiawen Wei and Lan Jiang and Pengbo Wei and Ziwen Ye and Teng Song and Chen Chen and Guangrui Ma},
booktitle={The Thirty-ninth Annual Conference on Neural Information Processing Systems},
year={2026}
}

@inproceedings{softs,
  title={S{OFTS}: Efficient Multivariate Time Series Forecasting with Series-Core Fusion},
  author={Han, Lu and Chen, Xu-Yang and Ye, Han-Jia and Zhan, De-Chuan},
  booktitle={Advances in Neural Information Processing Systems},
  year={2024}
}

@inproceedings{Adam,
  author    = {Diederik P. Kingma and
               Jimmy Ba},
  title     = {Adam: {A} Method for Stochastic Optimization},
  booktitle = {International Conference on Learning Representations},
  pages     = {1-9},
  year      = {2015}
}

@inproceedings{lin2025TQNet,
  title={Temporal Query Network for Efficient Multivariate Time Series Forecasting}, 
  author={Lin, Shengsheng and Chen, Haojun and Wu, Haijie and Qiu, Chunyun and Lin, Weiwei},
  booktitle={Forty-second International Conference on Machine Learning},
  year={2025}
}

@inproceedings{liangsurvey,
author = {Liang, Yuxuan and Wen, Haomin and Nie, Yuqi and Jiang, Yushan and Jin, Ming and Song, Dongjin and Pan, Shirui and Wen, Qingsong},
title = {Foundation Models for Time Series Analysis: A Tutorial and Survey},
year = {2024},
booktitle = {Proceedings of the 30th ACM SIGKDD Conference on Knowledge Discovery and Data Mining},
}

@article{hu2025fintsb,
  title={FinTSB: A Comprehensive and Practical Benchmark for Financial Time Series Forecasting}, 
  author={Yifan Hu and Yuante Li and Peiyuan Liu and Yuxia Zhu and Naiqi Li and Tao Dai and Shu-tao Xia and Dawei Cheng and Changjun Jiang},
  journal={arXiv preprint arXiv:2502.18834},
  year={2025},
}

@article{hu2025finmamba,
  title={FinMamba: Market-Aware Graph Enhanced Multi-Level Mamba for Stock Movement Prediction}, 
  author={Hu, Yifan and Liu, Peiyuan and Li, Yuante and Cheng, Dawei and Li, Naiqi and Dai, Tao and Bao, Jigang and Xia Shu-Tao},
  journal={arXiv preprint arXiv:2502.06707},
  year={2025}
}

@article{zhao2024rethinking,
  title={Rethinking superpixel segmentation from biologically inspired mechanisms},
  author={Zhao, TingYu and Peng, Bo and Sun, Yuan and Yang, DaiPeng and Zhang, ZhenGuang and Wu, Xi},
  journal={Applied Soft Computing},
  volume={156},
  pages={111467},
  year={2024},
  publisher={Elsevier}
}

@ARTICLE{mcformer,
  author={Han, Wenyong and Zhu, Tao and Chen, Liming and Ning, Huansheng and Luo, Yang and Wan, Yaping},
  journal={IEEE Internet of Things Journal}, 
  title={MCformer: Multivariate Time Series Forecasting With Mixed-Channels Transformer}, 
  year={2024},
  volume={11},
  number={17},
  pages={28320-28329},
}

@inproceedings{wu2020connecting,
  title={Connecting the Dots: Multivariate Time Series Forecasting with Graph Neural Networks},
  author={Wu, Zonghan and Pan, Shirui and Long, Guodong and Jiang, Jing and Chang, Xiaojun and Zhang, Chengqi},
  booktitle={Proceedings of the 26th ACM SIGKDD International Conference on Knowledge Discovery \& Data Mining},
  year={2020}
}

@article{liu2025timebridge,
      title={TimeBridge: Non-Stationarity Matters for Long-term Time Series Forecasting}, 
      author={Liu, Peiyuan and Wu, Beiliang and Hu, Yifan and Li, Naiqi and Dai, Tao and Bao, Jigang and Xia, Shu-Tao},
      journal={International Conference on Machine Learning},
      year={2025},
}

@inproceedings{PatchTST,
  title={A Time Series is Worth 64 Words: Long-term Forecasting with Transformers},
  author={Nie, Yuqi and Nguyen, Nam H and Sinthong, Phanwadee and Kalagnanam, Jayant},
  booktitle={International Conference on Learning Representations},
  year={2023}
}

@inproceedings{DLinear,
  title={Are Transformers Effective for Time Series Forecasting?},
  author={Ailing Zeng and Muxi Chen and Lei Zhang and Qiang Xu},
  booktitle={Proc. AAAI Conf. Artif. Intell.},
  year={2023}
}

@inproceedings{wang2025fredf,
    title = {FreDF: Learning to Forecast in the Frequency Domain},
    author = {Wang, Hao and Pan, Licheng and Chen, Zhichao and Yang, Degui and Zhang, Sen and Yang, Yifei and Liu, Xinggao and Li, Haoxuan and Tao, Dacheng},
    booktitle = {ICLR},
    year = {2025},
}

@article{evgeniou2005learning,
  author  = {Theodoros Evgeniou and Charles A. Micchelli and Massimiliano Pontil},
  title   = {Learning Multiple Tasks with Kernel Methods},
  journal = {Journal of Machine Learning Research},
  year    = {2005},
  volume  = {6},
  number  = {21},
  pages   = {615--637}
}

@inproceedings{
patchloss,
title={Patch-wise Structural Loss for Time Series Forecasting},
author={Dilfira Kudrat and Zongxia Xie and Yanru Sun and Tianyu Jia and Qinghua Hu},
booktitle={Forty-second International Conference on Machine Learning},
year={2025}
}

@article{dai2024pdf,
  title={Periodicity Decoupling Framework for Long-term Series Forecasting},
  author={Dai, Tao and Wu, Beiliang and Liu, Peiyuan and Li, Naiqi and Bao, Jigang and Jiang, Yong and Xia, Shu-Tao},
  journal={International Conference on Learning Representations},
  year={2024}
}

@inproceedings{wang2026iclrqdf,
  title={Quadratic Direct Forecast for Training Multi-Step Time-Series Forecast Models},
  author={Wang, Hao and Pan, Licheng and Lu, Yuan and Chen, Zhichao and Liu, Tianqiao and He, Shuting and Chu, Zhixuan and Wen, Qingsong and Li, Haoxuan and Lin, Zhouchen},
  booktitle = {ICLR},
  pages     = {1-9},
  year      = {2026},
}

@inproceedings{wang2026iclrdistdf,
  title={DistDF: Time-Series Forecasting Needs Joint-Distribution Wasserstein Alignment},
  author={Wang, Hao and Pan, Licheng and Lu, Yuan and Chu, Zhixuan and Li, Xiaoxi and He, Shuting and Chen, Zhichao and Li, Haoxuan and Wen, Qingsong and Lin, Zhouchen},
  booktitle = P_ICLR,
  pages     = {1-9},
  year      = {2026},
}

@inproceedings{qiu2025duet,
 title     = {DUET: Dual Clustering Enhanced Multivariate Time Series Forecasting},
 author    = {Xiangfei Qiu and Xingjian Wu and Yan Lin and Chenjuan Guo and Jilin Hu and Bin Yang},
 booktitle = {SIGKDD},
 pages     = {1185-1196},
 year      = {2025}
}

@misc{kmbdf,
      title={Deep Time-series Forecasting Needs Kernelized Moment Balancing}, 
      author={Licheng Pan and Hao Wang and Haocheng Yang and Yuqi Li and Qingsong Wen and Xiaoxi Li and Zhichao Chen and Haoxuan Li and Zhixuan Chu and Yuan Lu},
      year={2026},
      eprint={2602.00717},
      archivePrefix={arXiv},
      primaryClass={cs.LG},
}

@misc{qiu2026survey,
      title={A Comprehensive Survey of Deep Learning for Multivariate Time Series Forecasting: A Channel Strategy Perspective}, 
      author={Xiangfei Qiu and Hanyin Cheng and Xingjian Wu and Junkai Lu and Jilin Hu and Chenjuan Guo and Christian S. Jensen and Bin Yang},
      year={2026},
      eprint={2502.10721},
      archivePrefix={arXiv},
      primaryClass={cs.LG}
}

@ARTICLE{TDAlign,
  author={Xiong, Qi and Tang, Kai and Ma, Minbo and Zhang, Ji and Xu, Jie and Li, Tianrui},
  journal={IEEE Transactions on Knowledge and Data Engineering}, 
  title={Modeling Temporal Dependencies within the Target for Long-Term Time Series Forecasting}, 
  year={2025},
  volume={37},
  number={12},
  pages={7300-7314},
  doi={10.1109/TKDE.2025.3609415}}

@inproceedings{GDTW,
  title={Gdtw: A novel differentiable dtw loss for time series tasks},
  author={Liu, Xiang and Li, Naiqi and Xia, Shu-Tao},
  booktitle=P_ICASSP,
  pages={2860--2864},
  year={2021},
  organization={IEEE}
}

@article{Dilate,
  title={Shape and time distortion loss for training deep time series forecasting models},
  author={Le Guen, Vincent and Thome, Nicolas},
  journal=P_NIPS,
  volume={32},
  year={2019}
}

@inproceedings{soft-dtw,
  title={Soft-dtw: a differentiable loss function for time-series},
  author={Cuturi, Marco and Blondel, Mathieu},
  booktitle=P_ICML,
  pages={894--903},
  year={2017},
  organization={PMLR}
}

@inproceedings{li2026fadti,
  title     = {FADTI: Fourier and Attention Driven Diffusion for Multivariate Time Series Imputation},
  author    = {Runze Li and Hanchen Wang and Wenjie Zhang and Binghao Li and Yu Zhang and Xuemin Lin and Ying Zhang},
  booktitle = {Proceedings of the 2026 IEEE International Conference on Data Mining (ICDM)},
  year      = {2026},
}

@article{wang2025nipstimeo1,
  title={Time-o1: Time-Series Forecasting Needs Transformed Label Alignment},
  author={Wang, Hao and Pan, Licheng and Chen, Zhichao and Chen, Xu and Dai, Qingyang and Wang, Lei and Li, Haoxuan and Lin, Zhouchen},
  journal=P_NIPS,
  year={2025}
}

@article{koopman,
  title={From Fourier to Koopman: Spectral methods for long-term time series prediction},
  author={Lange, Henning and Brunton, Steven L and Kutz, J Nathan},
  journal={Journal of Machine Learning Research},
  volume={22},
  number={41},
  pages={1--38},
  year={2021}
}

@inproceedings{CFPT,
  title={CFPT: Empowering Time Series Forecasting through Cross-Frequency Interaction and Periodic-Aware Timestamp Modeling},
  author={Kou, Feifei and Wang, Jiahao and Shi, Lei and Yao, Yuhan and Li, Yawen and Zhu, Suguo and Zhang, Zhongbao and Du, Junping},
  booktitle={Forty-second International Conference on Machine Learning},
  year={2025}
}

@inproceedings{crossformer,
  title={Crossformer: Transformer utilizing cross-dimension dependency for multivariate time series forecasting},
  author={Zhang, Yunhao and Yan, Junchi},
  booktitle={International Conference on Learning Representations},
  year={2023}
}

@inproceedings{fedformer,
  title={F{ED}former: Frequency enhanced decomposed transformer for long-term series forecasting},
  author={Zhou, Tian and Ma, Ziqing and Wen, Qingsong and Wang, Xue and Sun, Liang and Jin, Rong},
  booktitle={International Conference on Machine Learning},
  pages={27268--27286},
  year={2022},
  organization={PMLR}
}

@inproceedings{qiu2025DBLoss,
title     = {DBLoss: Decomposition-based Loss Function for Time Series Forecasting},
author    = {Xiangfei Qiu and Xingjian Wu and Hanyin Cheng and Xvyuan Liu and Chenjuan Guo and Jilin Hu and Bin Yang},
booktitle = {NeurIPS},
year      = {2025}
}

@article{hou2025graph,
  title={Graph-patchformer: Patch interaction transformer with adaptive graph learning for multivariate time series forecasting},
  author={Hou, Chunyi and Yu, Yongchuan and Ji, Jinquan and Zhang, Siyao and Shen, Xumeng and Yan, Jianzhuo},
  journal={Neural Networks},
  volume={194},
  pages={108140},
  year={2026},
  doi={10.1016/j.neunet.2025.108140}
}

@article{ji2026fine,
  title={Fine-grained dynamic graph learning for irregular multivariate time series},
  author={Ji, Jinquan and Hou, Chunyi and Cao, Yu and Ma, Yukun and Yan, Jianzhuo},
  journal={Information Processing \& Management},
  volume={63},
  number={6},
  pages={104740},
  year={2026},
  doi={10.1016/j.ipm.2026.104740}
}
\bibliographystyle{unsrtnat}

\newpage
\appendix
\section{Theoretical Justification}
\label{app:theory}

In this section, we provide the detailed derivation of Theorem~\ref{thm:spatiotemporal_nll}, which reveals the suboptimality of the standard DF objective under real-world spatiotemporal correlations, followed by the proof of Proposition~\ref{prop:precision}.

\subsection{Proof of Theorem~\ref{thm:spatiotemporal_nll}}

\textbf{Proof.} From a probabilistic perspective, finding the optimal parameters $\theta$ for the forecasting model $f_\theta$ is equivalent to maximizing the likelihood of the ground-truth target $Y$ given the input history $X$. Let
\[
\mathbf{e} = \mathrm{vec}(\hat{Y} - Y) \in \mathbb{R}^{TD}
\]
be the flattened forecast error vector.

Assume that the ground-truth residuals follow a zero-mean multivariate Gaussian distribution, i.e.,
\[
\mathbf{e} \sim \mathcal{N}(\mathbf{0}, \mathbf{\Sigma}_{ST}),
\]
where $\mathbf{\Sigma}_{ST} \in \mathbb{R}^{TD \times TD}$ is the true spatiotemporal covariance matrix. The probability density function of the residuals is given by:
\[
p(\mathbf{e}) =
\frac{1}{\sqrt{(2\pi)^{TD} \det(\mathbf{\Sigma}_{ST})}}
\exp\left(
-\frac{1}{2} \mathbf{e}^\top \mathbf{\Sigma}_{ST}^{-1} \mathbf{e}
\right)
\]

Minimizing the NLL of this distribution yields the optimal objective function. Taking the negative logarithm, we have:
\[
-\log p(\mathbf{e}) =
\frac{TD}{2} \log(2\pi)
+ \frac{1}{2} \log \det(\mathbf{\Sigma}_{ST})
+ \frac{1}{2} \mathbf{e}^\top \mathbf{\Sigma}_{ST}^{-1} \mathbf{e}
\]

Assuming the covariance matrix $\mathbf{\Sigma}_{ST}$ represents the intrinsic data distribution and is independent of the model parameters $\theta$, the first two terms are constants with respect to the optimization process. By dropping these constants and scaling the objective by $\frac{1}{TD}$ to obtain the average loss per prediction entry, we define the practical NLL objective as:
\[
\mathcal{L}_{\mathrm{nll}} = \frac{1}{2TD} \mathbf{e}^\top \mathbf{\Sigma}_{ST}^{-1} \mathbf{e}
\]

Now, consider the standard point-wise MSE objective used in the DF paradigm, defined as:
\[
\mathcal{L}_{\mathrm{df}} =
\frac{1}{T} \sum_{t=1}^{T} \frac{1}{D} \| Y_{t,:} - \hat{Y}_{t,:} \|_2^2
= \frac{1}{TD} \mathbf{e}^\top \mathbf{e}
\]

Notice that $\mathcal{L}_{\mathrm{df}}$ can be probabilistically interpreted as the negative log-likelihood under a strictly naive assumption: the residuals of all variables across all time steps are entirely independent and share a uniform variance $\sigma^2$. Under this assumption, the covariance matrix reduces to a spherical structure $\mathbf{\Sigma}_{ST} = \sigma^2 \mathbf{I}$, where $\mathbf{I} \in \mathbb{R}^{TD \times TD}$ is the identity matrix. In this special case, the NLL naturally degrades to:
\[
\mathcal{L}_{\mathrm{nll\_naive}} =
\frac{1}{2TD} \mathbf{e}^\top (\sigma^2 \mathbf{I})^{-1} \mathbf{e}
= \frac{1}{2\sigma^2} \left( \frac{1}{TD} \mathbf{e}^\top \mathbf{e} \right)
= \frac{1}{2\sigma^2} \mathcal{L}_{\mathrm{df}}
\]

To formally quantify the structural mismatch between the standard DF paradigm and the optimal estimation under true spatiotemporal correlations, we define the structural gap $\Delta$ as the difference between the true optimal objective $\mathcal{L}_{\mathrm{nll}}$ and the scaled DF objective $\frac{1}{2\sigma^2}\mathcal{L}_{\mathrm{df}}$:
\[
\Delta =
\mathcal{L}_{\mathrm{nll}} - \frac{1}{2\sigma^2}\mathcal{L}_{\mathrm{df}}
= \frac{1}{2TD} \mathbf{e}^\top \mathbf{\Sigma}_{ST}^{-1} \mathbf{e}
- \frac{1}{2TD \sigma^2} \mathbf{e}^\top \mathbf{I} \mathbf{e}
\]

Factoring out the common terms, we obtain:
\[
\Delta =
\frac{1}{2TD} \mathbf{e}^\top
\left(
\mathbf{\Sigma}_{ST}^{-1} - \frac{1}{\sigma^2}\mathbf{I}
\right)
\mathbf{e}
\]

If cross-variable or lagged residual dependencies yield non-zero off-diagonal entries in $\mathbf{\Sigma}_{ST}^{-1}$, the symmetric matrix $\mathbf{\Sigma}_{ST}^{-1}-\sigma^{-2}\mathbf{I}$ is non-zero. Its quadratic form is therefore not identically zero, so point-wise MSE and the true Gaussian NLL are different objectives. The gap can still be zero for particular residual vectors; the result concerns equality of the objective functions rather than a non-zero difference at every point. This concludes the proof.

\subsection{Proof of Proposition~\ref{prop:precision}}
\label{app:proposition}

Theorem~\ref{thm:spatiotemporal_nll} shows that MSE is the correct objective only when the residual precision matrix is a multiple of the identity, but it does not say which non-diagonal precision matrix should replace it, and estimating $\mathbf{\Sigma}_{ST}^{-1}$ from data is not feasible at the scales considered here. Proposition~\ref{prop:precision} closes this gap from the opposite direction: instead of estimating a precision matrix, we fix a one-parameter family of precision matrices whose off-diagonal support is a cross-variable graph, and show that the negative log-likelihood of that family is exactly MSE plus a squared graph total-variation penalty.

\textbf{Proof.} Recall from Section~\ref{cvl} that $A\in\mathbb{R}^{|\mathcal{E}|\times N}$ is the incidence matrix of the cross-variable graph under an arbitrary edge orientation, that $\tilde A=A\otimes I_L$ is its lifted counterpart, and that $\mathbf{r}=\mathrm{vec}(E)\in\mathbb{R}^{NL}$ collects the node-wise residuals $\mathbf{e}_1,\dots,\mathbf{e}_N\in\mathbb{R}^{L}$. By construction, the $k$-th block of $\tilde A\mathbf{r}$ associated with the edge $k=(i,j)$ equals $\mathbf{e}_i-\mathbf{e}_j$, so
\[
\|\tilde A\mathbf{r}\|_2^{2}
=
\sum_{(i,j)\in\mathcal{E}}\|\mathbf{e}_i-\mathbf{e}_j\|_2^{2} .
\]
Writing $\mathrm{Lap}(\mathcal{E})=\tilde A^{\top}\tilde A$ for the lifted graph Laplacian, the same quantity is the associated quadratic form,
\[
\mathbf{r}^{\top}\mathrm{Lap}(\mathcal{E})\,\mathbf{r}
=
\mathbf{r}^{\top}\tilde A^{\top}\tilde A\,\mathbf{r}
=
\|\tilde A\mathbf{r}\|_2^{2}
=
\sum_{(i,j)\in\mathcal{E}}\|\mathbf{e}_i-\mathbf{e}_j\|_2^{2} .
\]
Substituting the structured precision family of Eq.~\eqref{eq:precision_family} and expanding by linearity gives
\[
\mathbf{r}^{\top}\mathbf{\Sigma}^{-1}(\lambda)\,\mathbf{r}
=
\frac{1}{\sigma^{2}}\Bigl(\mathbf{r}^{\top}\mathbf{r}+\lambda\,\mathbf{r}^{\top}\mathrm{Lap}(\mathcal{E})\,\mathbf{r}\Bigr)
=
\frac{1}{\sigma^{2}}
\Bigl(
\|\mathbf{r}\|_2^{2}
+
\lambda\!\!\sum_{(i,j)\in\mathcal{E}}\!\!\|\mathbf{e}_i-\mathbf{e}_j\|_2^{2}
\Bigr),
\]
which is Eq.~\eqref{eq:precision_identity}. This concludes the proof. $\hfill $

\paragraph{Remarks.} Four consequences are worth recording, since they determine how the theory should and should not be read.

\begin{itemize}[leftmargin=*]
\item \textbf{The family is a valid precision family.} $\mathrm{Lap}(\mathcal{E})$ is positive semi-definite, so $\mathbf{I}+\lambda\,\mathrm{Lap}(\mathcal{E})\succ 0$ for every $\lambda\ge 0$ and $\mathbf{\Sigma}^{-1}(\lambda)$ is a genuine precision matrix. Its off-diagonal entries are $-\lambda/\sigma^{2}$ exactly on the node pairs in $\mathcal{E}$ and zero elsewhere, so $\mathcal{E}$ is the conditional-independence structure of the corresponding Gaussian field rather than an analogy.
\item \textbf{It contains the spherical case.} At $\lambda=0$ we recover $\mathbf{\Sigma}^{-1}=\sigma^{-2}\mathbf{I}$, which is precisely the situation in which the objective gap $\Delta$ of Theorem~\ref{thm:spatiotemporal_nll} vanishes. The theorem and the loss are therefore two ends of one family rather than unrelated statements.
\item \textbf{It matches the deployed weighting.} Dividing the objective $\mathcal{L}_{\alpha}=(1-\alpha)\mathcal{L}_{\mathrm{df}}+\alpha\mathcal{L}_{\mathrm{cv}}$ by $(1-\alpha)$ leaves the minimiser unchanged and yields $\mathcal{L}_{\mathrm{df}}+\lambda\mathcal{L}_{\mathrm{cv}}$ with $\lambda=\alpha/(1-\alpha)$. Values of $\alpha$ close to $1$ therefore correspond to a large $\lambda$ within the family above, not to a weight outside the stated range $(0,1)$.
\item \textbf{The complete graph is the weakest member.} Since $\mathcal{E}$ is the support of the off-diagonal entries, taking $\mathcal{E}$ to be the complete cross-variable graph imposes \emph{no} zero constraint on any cross-variable pair; sparser choices impose strictly more. The empty graph, i.e.\ plain MSE, is the member excluded by Theorem~\ref{thm:spatiotemporal_nll} whenever cross-variable dependencies are present.
\end{itemize}

Proposition~\ref{prop:precision} concerns the squared penalty. The objective we deploy in Eq.~\eqref{eq:cvl} uses the $\ell_1$ norm on the same edge differences, which yields the pairwise field
\[
p_{1}(\mathbf{r}\mid\mathcal{E})
\;\propto\;
\exp\Bigl[
-a\sum_{v=1}^{N}\|\mathbf{e}_v\|_2^{2}
-b\!\!\sum_{(i,j)\in\mathcal{E}}\!\!\|\mathbf{e}_i-\mathbf{e}_j\|_1
\Bigr],
\qquad
a=\frac{1-\alpha}{NL}>0,
\quad
b=\frac{\alpha}{|\mathcal{E}|L}>0 .
\]
This factorises into node and edge potentials and is therefore a pairwise Gibbs Markov random field with respect to the same graph, but it is not Gaussian, and we consequently restrict the precision-matrix reading to the $\ell_2$ case. We also do not claim that $\mathcal{E}$ recovers a sparse precision support: the default topology is complete, with random edge sampling used only for scalability. The $\ell_1$ form is adopted because squared penalties can be dominated by a small number of large edge-wise discrepancies arising from localized anomalies or variable-specific shocks. This is a task-specific empirical hypothesis about time-series residuals rather than a universal claim. A comparable separation between the norm used for analysis and the norm used in practice appears in FreDF~\citep{wang2025fredf}, whose analysis is $\ell_2$/Gaussian while its deployed frequency-domain loss is $\ell_1$. Appendix~\ref{app:norm_choice} tests the two norms directly.

\section{Reproduction Details}\label{sec:reproduce}
\subsection{Dataset descriptions}\label{sec:dataset}

Our empirical evaluation is conducted on a diverse collection of widely-used time series forecasting benchmarks. Each dataset presents distinct characteristics in terms of dimensionality and temporal resolution. A summary is provided in \autoref{tab:dataset}.

\begin{itemize}[leftmargin=*]
\item \textbf{ETT}~\citep{Informer}: This dataset includes seven electricity transformer indicators logged between July 2016 and July 2018. It is partitioned into four subsets according to sampling rates: ETTh1 and ETTh2 for hourly data, and ETTm1 and ETTm2 for 15-minute intervals.
\item \textbf{Weather}~\citep{Autoformer}: Encompasses 21 meteorological factors collected every 10 minutes throughout the year 2020 from the Max Planck Institute for Biogeochemistry.
\item \textbf{ECL}~\citep{Autoformer}: Records the hourly power usage patterns of 321 individual customers.
\item \textbf{Traffic}~\citep{Autoformer}: Monitors hourly road occupancy via 862 sensors located on freeways in the San Francisco Bay Area from 2015 to 2016.
\item \textbf{Solar-Energy}\footnote{\url{https://www.nrel.gov/grid/solar-power-data.html}}: 10-minute solar power production records from 137 photovoltaic plants throughout 2006.
\item \textbf{PEMS}~\citep{SCINet}: Derived from public California highway traffic records, aggregated at 5-minute steps. We adopt two standard subsets, PEMS03 and PEMS08, for our study.
\end{itemize}

\begin{table}
  \caption{Dataset description. }\label{tab:dataset}
  \centering
  \renewcommand{\multirowsetup}{\centering}
  \setlength{\tabcolsep}{8pt}
  \small
    \begin{threeparttable}
  \begin{tabular}{llllll}
    \toprule
    Dataset & D & Forecast length & Train / validation / test & Frequency& Domain \\
    \toprule
     ETTh1 & 7 & 96, 192, 336, 720 & 8545/2881/2881 & Hourly & Health\\
     \midrule
     ETTh2 & 7 & 96, 192, 336, 720 & 8545/2881/2881 & Hourly & Health\\
     \midrule
     ETTm1 & 7 & 96, 192, 336, 720 & 34465/11521/11521 & 15min & Health\\
     \midrule
     ETTm2 & 7 & 96, 192, 336, 720 & 34465/11521/11521 & 15min & Health\\
    \midrule
    Weather & 21 & 96, 192, 336, 720 & 36792/5271/10540 & 10min & Weather\\
    \midrule
    ECL & 321 & 96, 192, 336, 720 & 18317/2633/5261 & Hourly & Electricity \\
    \midrule
    Traffic & 862 & 96, 192, 336, 720 & 12185/1757/3509 & Hourly & Transportation \\
    \midrule
    Solar-Energy & 137  & 96, 192, 336, 720 & 36601/5161/10417 & 10min & Energy \\
    \midrule
    PEMS03 & 358 & 12, 24, 36, 48 & 15617/5135/5135 & 5min & Transportation\\
    \midrule
    PEMS04 & 307 & 12, 24, 48, 96 & 10172/3375/281 & 5min & Transportation\\
    \midrule
    PEMS07 & 883 & 12, 24, 48, 96 & 16911/5622/468 & 5min & Transportation\\
    \midrule
    PEMS08 & 170 & 12, 24, 36, 48 & 10690/3548/265 & 5min & Transportation\\
    \bottomrule
    \end{tabular}
    \begin{tablenotes}
    \item  \scriptsize \textit{Note}:  \textit{D} denotes the number of variates. \emph{Frequency} denotes the sampling interval of time points. \emph{Train, Validation, Test} denotes the number of samples employed in each split. The taxonomy aligns with~\citep{Timesnet}.
    \end{tablenotes}
    \end{threeparttable}
\end{table}

\paragraph{Why these datasets.} The twelve benchmarks above are the standard suite of~\citep{Timesnet, itransformer} and were selected before any results were obtained; we record the reasoning here because the choice interacts with what \time\ is designed to do. Three considerations apply. First, we require coverage of the variable-count range, since \time\ regularizes interactions between variables and its cost and its benefit both depend on $D$: the suite spans $D=7$ (the four ETT subsets) to $D=862$ (Traffic), with Weather ($21$), Solar ($137$), PEMS08 ($170$), PEMS04 ($307$), ECL ($321$), PEMS03 ($358$) and PEMS07 ($883$) in between. Second, we require coverage of the cross-variable correlation range, which Theorem~\ref{thm:spatiotemporal_nll} predicts should govern the size of the effect; the suite spans weakly correlated channels (ETT, average CCC $0.14$--$0.18$) to strongly correlated ones (Solar, $0.91$; PEMS07, $0.78$), as reported in Table~\ref{tab:learned_weights_corr}. Third, both the long-horizon regime ($T\in\{96,192,336,720\}$) and the short-horizon regime used for PEMS ($T\in\{12,24,36,48\}$) are represented, so the conclusions are not specific to one horizon scale. All four ETT subsets are included, at both the hourly (ETTh1, ETTh2) and the $15$-minute (ETTm1, ETTm2) sampling rates.

We deliberately exclude the Exchange dataset, following the protocol of~\citep{hu2025timefilter}. Exchange is close to a driftless random walk, a property documented since~\citep{SCINet}, and on such a series the trivial predictor that repeats the last observed value is competitive with deep forecasters. A benchmark on which the degenerate baseline is already near-optimal cannot discriminate between learning objectives, which is the comparison this paper makes. The point is compounded for our specific claim: Exchange consists of eight loosely coupled national exchange rates, so it carries little of the cross-variable structure that \time\ is designed to exploit, and including it would neither support nor challenge the mechanism under test. We note the exclusion explicitly rather than silently, since the omission is a choice about what the evidence can show.

Following established protocols~\citep{qiutfb, itransformer}, all datasets are segmented chronologically into training, validation, and testing portions. Regarding the ETT, Weather, ECL, Solar and Traffic datasets, we maintain a constant input window of 96 time steps and assess model performance across four distinct forecasting horizons: 96, 192, 336, and 720. For the PEMS datasets, while the historical look-back period remains 96, the evaluation focuses on shorter prediction intervals of 12, 24, 36, and 48 steps. Throughout our experimental phase, we disable the \textit{dropping-last trick} during the final test set assessment to ensure that every data point in the terminal batch is fully utilized.

\subsection{Implementation details of model training}
To ensure a strictly fair comparison, we build all experiments on top of publicly released benchmark implementations and retain the original model architectures and hyperparameter configurations. When integrating CvLoss, we do not modify the forecasting backbone, data preprocessing, optimizer, training schedule, or evaluation protocol; only the proposed objective term and its associated weight are added. We also disable the drop-last trick for all methods following~\citep{qiutfb}, so that every test sample is evaluated consistently. To establish a fair comparison, we reproduced all baseline models using their official, publicly available implementations, primarily sourcing from the TQNet~\citep{lin2025TQNet}, TimeFilter~\citep{hu2025timefilter}, iTransformer~\citep{itransformer}, TimeBridge~\citep{liu2025timebridge} and Time-Series-Library~\citep{wang2024tssurvey} repositories. The reproducibility of these baseline results was verified prior to our experiments. All models were trained to minimize the MSE loss function using the Adam optimizer~\citep{Adam}. All baselines apply to the publicly available parameters in their library without any changes. To prevent overfitting, we employed an early stopping mechanism that terminates training if the validation loss fails to improve for three consecutive epochs.

\paragraph{What is tuned, and on which split.}
\label{app:tuning}
Two quantities are tuned for \time: the loss weight $\alpha$ and the patch length $L$. Nothing else is. Both are selected on the \emph{validation} split and never on the test split, and all methods, baselines included, use early stopping on the validation loss with a patience of three epochs. Every number reported in the main text and in the appendix uses the fixed convex objective $\mathcal{L}_{\alpha}=(1-\alpha)\mathcal{L}_{\mathrm{df}}+\alpha\mathcal{L}_{\mathrm{cv}}$ with a single scalar $\alpha\in(0,1)$ chosen this way; $\alpha$ is never learned during training, and there is no configuration in which it varies across batches, variables or edges. The sensitivity tables report the full sweep over $\alpha$, whereas the main tables report the validation-selected setting; the two are consistent by construction, since the $\alpha=0$ rows of Tables~\ref{tab:sensi-iTransformer} and~\ref{tab:sensi-TimeBridge} reproduce the unregularised baselines exactly. The separate quantities in Table~\ref{tab:learned_weights_corr} are a diagnostic and are \emph{not} values of $\alpha$; this is stated with that table.

\paragraph{Baselines are tuned at least as favourably as our method.}
This deserves a precise statement, because a plug-in objective can appear to win simply by receiving more tuning effort than the models it is added to. The comparison here is arranged so that the opposite holds. Baselines are the authors' own public implementations, run with the hyperparameters published in their libraries and verified to reproduce beforehand; we change none of them. When \time\ is added, the backbone, data preprocessing, normalisation, optimizer, learning-rate schedule, batch size, early-stopping rule and evaluation protocol are all left exactly as the baseline defines them, and only the additional objective term and its weight are introduced. Our method therefore has strictly less freedom than the models it is compared against: two scalars against a fully tuned configuration. The drop-last trick is disabled for every model, ours included, following~\citep{qiutfb}, so no method benefits from a truncated final test batch.

\paragraph{No ground-truth future statistics enter the objective.}
\label{app:no_leakage}
In every configuration used for the reported results, the edge set $\mathcal{E}$ is fixed a priori, since it is the complete cross-variable graph of Section~\ref{cvl}, and therefore depends on no property of the future window. The quantity \time\ consumes is the residual $E=\hat Z-Z$, which is the same supervision signal that the standard DF objective already consumes in Eq.~\eqref{eq:df_loss}; no additional statistic of $Y$ is used. The one variant that inspects the ground-truth correlation matrix $\mathbf{C}(Y)$ is the top-$K$ edge selection of Appendix~\ref{app:complexity_and_selection}, which is used \emph{during training only} to concentrate computation on the worst-aligned variable pairs, and which produced none of the numbers in this paper. In all configurations the \time\ term leaves the computation graph at inference: it contributes no parameters, no buffers and no forward-pass computation to the deployed model, which is why the inference latency in Figure~\ref{fig:inference_latency} sits at parity with the unmodified backbone.

\subsection{Additional Visualization of Cross-Variable Correlation Structures}
\label{app:heatmap_vis}

\begin{figure}
\subfigure[Ground-truth structure.]{\includegraphics[width=0.23\linewidth]{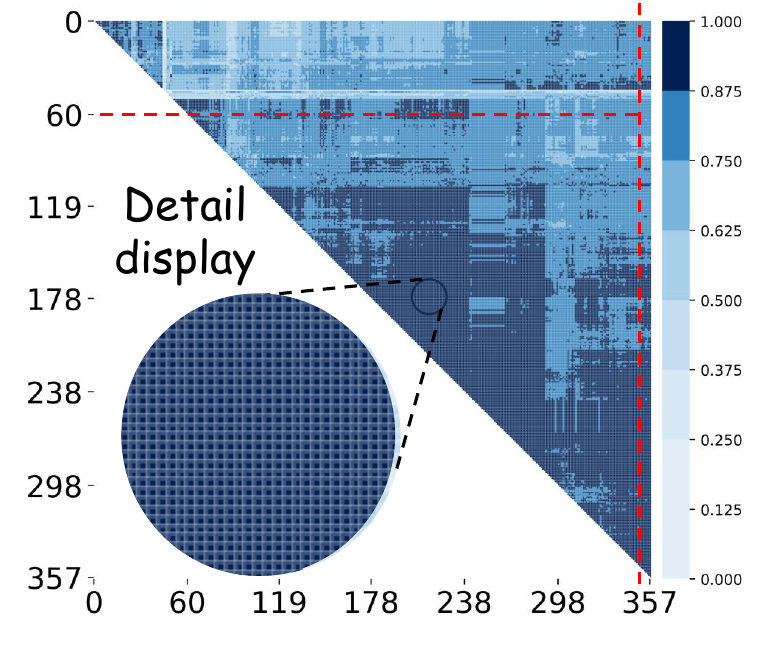}}
\hfill
    \raisebox{-0.0\height}{\rule{0.8pt}{2.6cm}} % Vertical line of 4cm height and 0.5pt width
\hfill
\subfigure[Baseline structure.]{\includegraphics[width=0.23\linewidth]{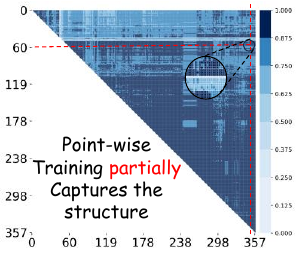}}
\hfill
    \raisebox{-0.0\height}{\rule{0.8pt}{2.6cm}} % Vertical line of 4cm height and 0.5pt width
\hfill
\subfigure[Improvement by CvLoss.]{\includegraphics[width=0.23\linewidth]{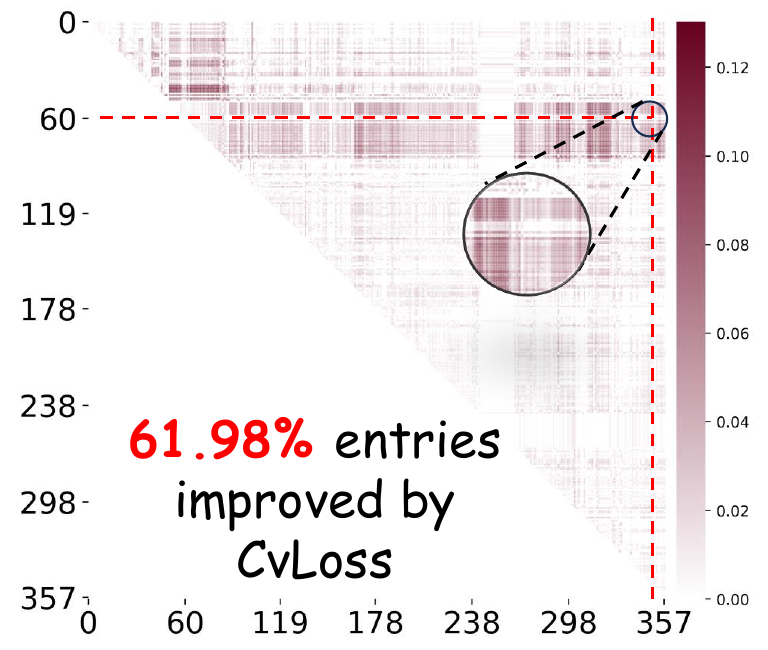}}
\hfill
    \raisebox{-0.0\height}{\rule{0.8pt}{2.6cm}} % Vertical line of 4cm height and 0.5pt width
\hfill
\subfigure[PEMS03 Snapshot.]{\includegraphics[width=0.2\linewidth]{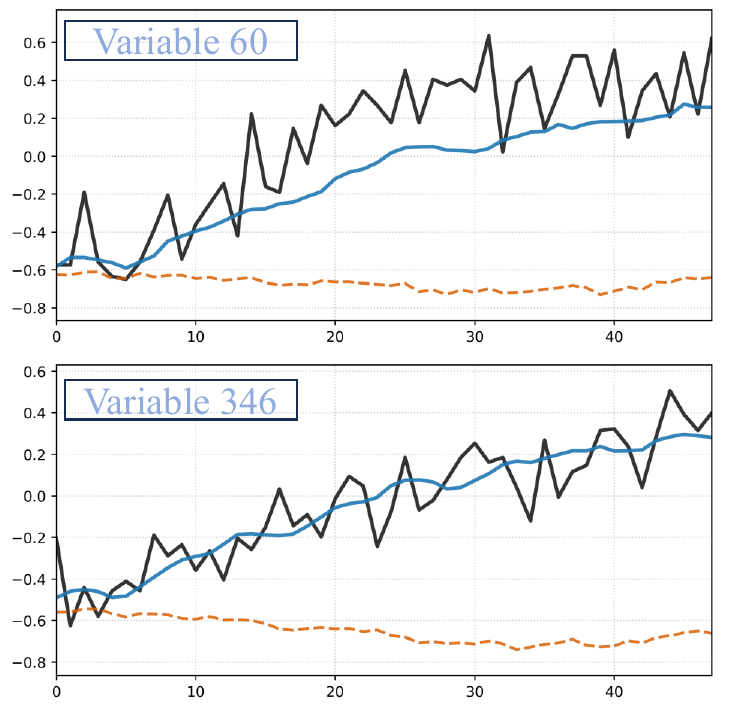}}

\caption{Motivating example on PEMS03 with input length $H=96$ and forecast horizon $T=48$. (a,b) Cross-variable correlation structures of the ground truth and the plain iTransformer~\citep{itransformer} prediction. (c) Entries where CvLoss reduces structural error over the baseline, covering 61.98\% of cross-variable pairs. (d) Forecast snapshot: orange denotes the baseline prediction, black denotes the ground truth, and blue denotes ours (CvLoss). See details in Appendix~\ref{app:heatmap_vis}.}

\label{fig:idea_pems03}
\end{figure}

To further examine how CvLoss affects the recovery of cross-variable dependencies, we provide additional visualizations on the ETTh2 and Weather datasets. These cases follow the same qualitative protocol as Figure~\ref{fig:idea}, while using input length $H=96$ and forecast horizon $T=96$.

Let $Y \in \mathbb{R}^{T \times D}$ denote the ground-truth future series and let $\hat{Y}^{\text{base}}, \hat{Y}^{\text{cv}} \in \mathbb{R}^{T \times D}$ denote the future series predicted by the baseline iTransformer and by iTransformer trained with CvLoss, respectively. For any future series $S \in \mathbb{R}^{T \times D}$, we define its cross-variable correlation structure as
\[
C(S) = \left[ \, \big| \mathrm{Corr}(S_{:,i}, S_{:,j}) \big| \, \right]_{i,j=1}^{D} \in \mathbb{R}^{D \times D},
\]
where $\mathrm{Corr}(\cdot,\cdot)$ denotes the Pearson correlation coefficient. Accordingly, we visualize
\[
C^{\mathrm{gt}} = C(Y), \qquad
C^{\mathrm{base}} = C(\hat{Y}^{\mathrm{base}}).
\]

To quantify structural deviation from the ground-truth pattern, we define the structural error matrices as
\[
\mathrm{SE}^{\mathrm{base}} = \big| C^{\mathrm{base}} - C^{\mathrm{gt}} \big|, \qquad
\mathrm{SE}^{\mathrm{cv}} = \big| C(\hat{Y}^{\mathrm{cv}}) - C^{\mathrm{gt}} \big|,
\]
where the absolute value is taken element-wise. The improvement induced by CvLoss is then measured by
\[
\Delta = \mathrm{SE}^{\mathrm{base}} - \mathrm{SE}^{\mathrm{cv}}.
\]
A positive entry $\Delta_{ij} > 0$ means that CvLoss produces a smaller structural error than the baseline on the $(i,j)$-th cross-variable entry, i.e., the recovered correlation is closer to the ground-truth structure. Therefore, panel (c) visualizes the positive part of $\Delta$, highlighting where CvLoss improves over the baseline in recovering cross-variable dependencies.

These supplementary cases complement the ECL example of Figure~\ref{fig:idea} in the main text and the PEMS03 example of Figure~\ref{fig:idea_pems03}, and verify that the structural benefit of CvLoss is not restricted to a single dataset. As shown in Figures~\ref{fig:etth2_heatmap} and~\ref{fig:weather_heatmap}, CvLoss consistently improves the recovery of cross-variable correlation structures across different data regimes, including both low-dimensional and higher-dimensional multivariate forecasting benchmarks. This observation is consistent with our main claim that CvLoss acts as a structural regularizer on multivariate outputs and improves the preservation of cross-variable dependencies beyond point-wise supervision alone.

\begin{figure}[t]
\centering
\subfigure[Ground-truth structure.]{\includegraphics[width=0.3\linewidth]{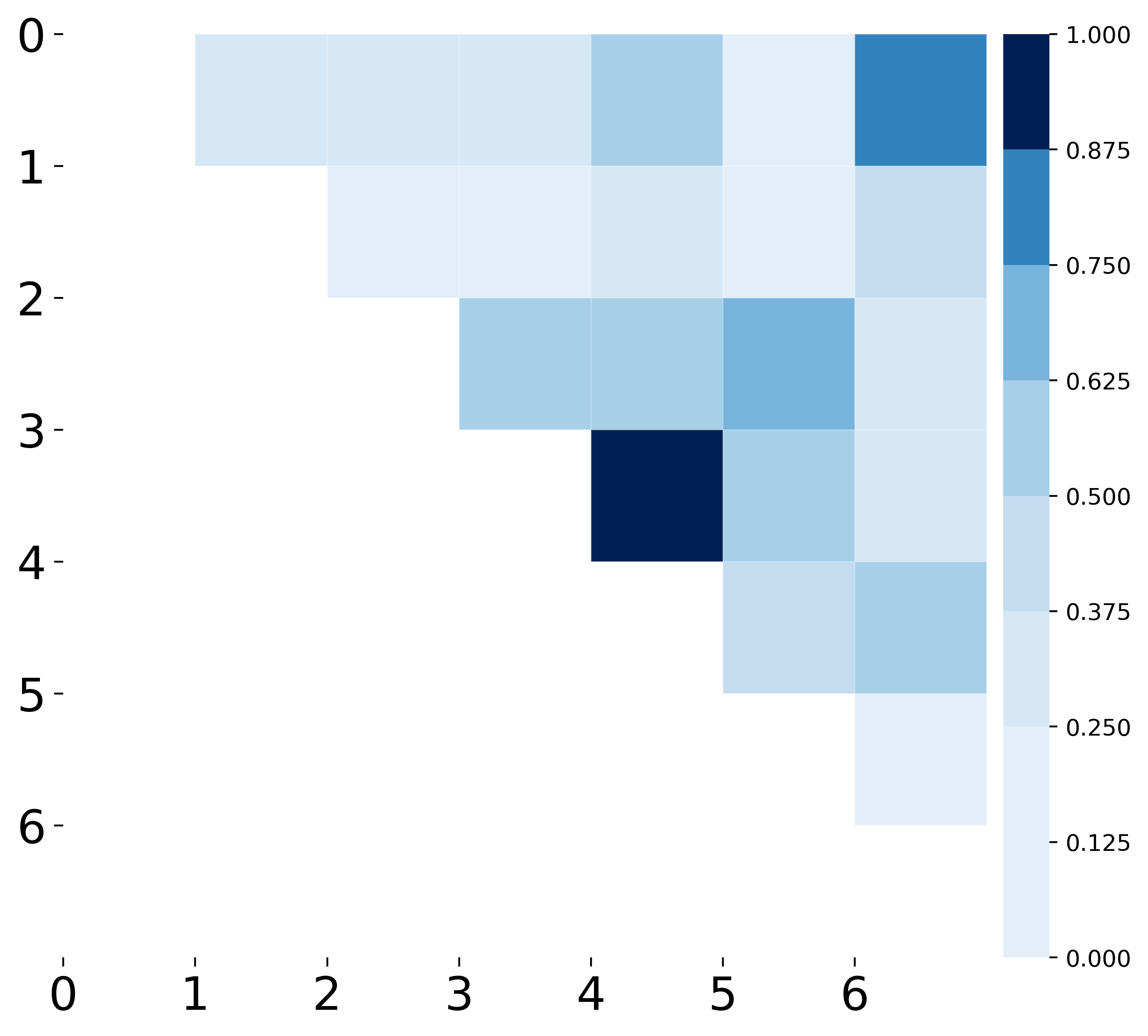}}
\hfill
    \raisebox{-0.0\height}{\rule{0.8pt}{3.5cm}}
\hfill
\subfigure[Baseline structure.]{\includegraphics[width=0.3\linewidth]{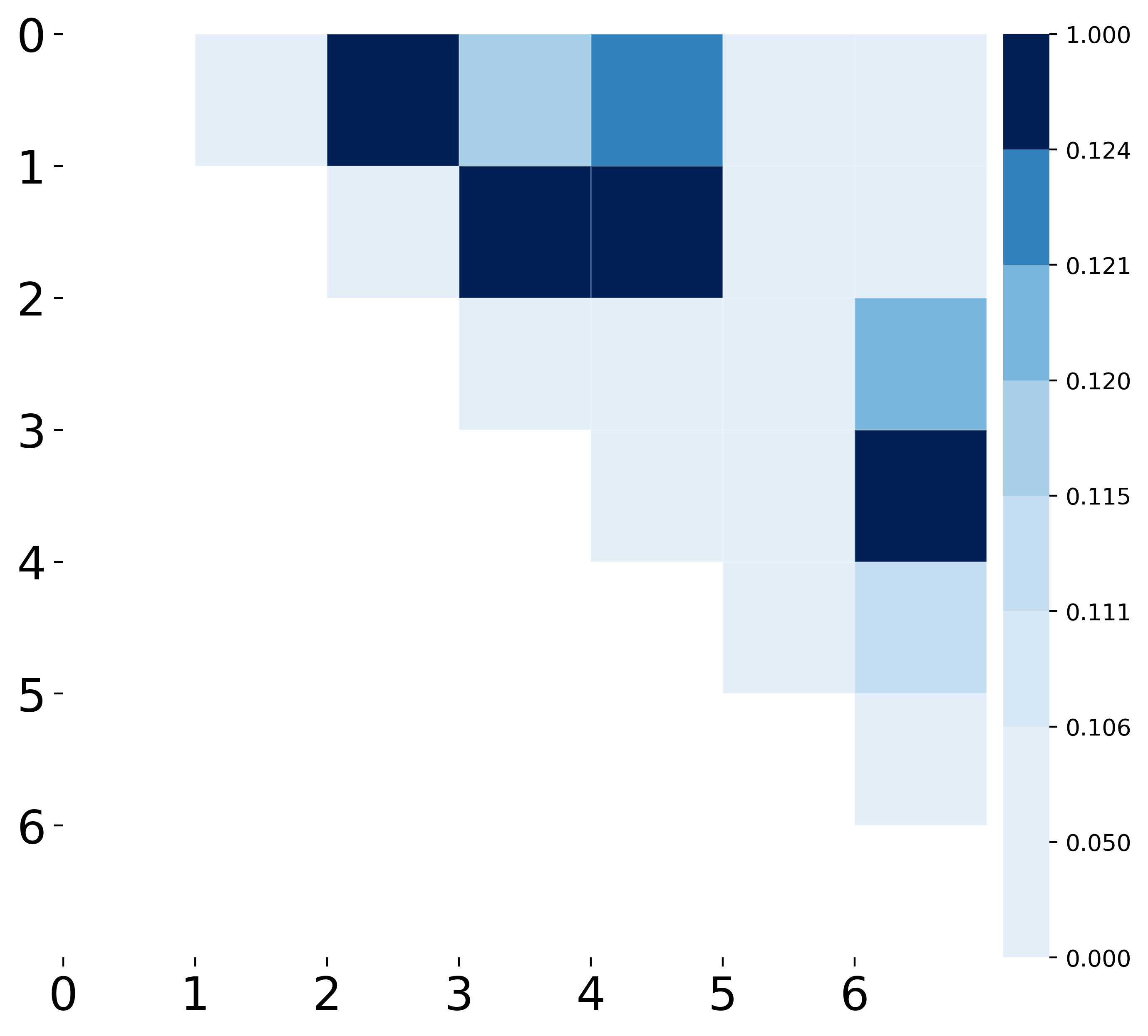}}
\hfill
    \raisebox{-0.0\height}{\rule{0.8pt}{3.5cm}}
\hfill
\subfigure[Improvement by CvLoss.]{\includegraphics[width=0.3\linewidth]{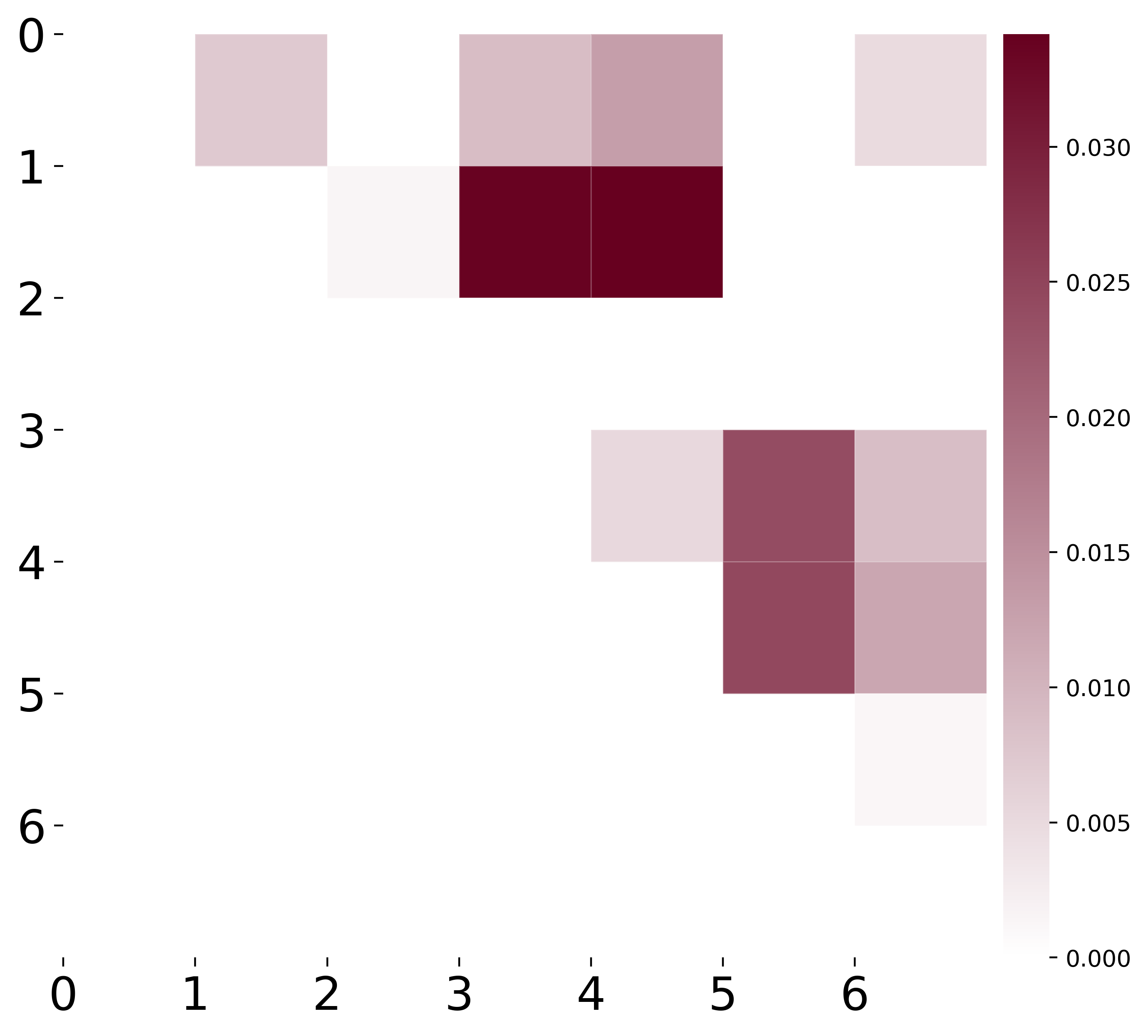}}
\caption{The cross-variable correlation structures of future series given $X$ on the ETTh2 dataset, with input length $H=96$ and forecast horizon $T=96$. The correlation matrices are computed from the ground-truth future series (a) and the future series predicted by iTransformer~\citep{itransformer} (b), while (c) shows the entries on which CvLoss yields smaller structural errors than iTransformer (62\%).}

\label{fig:etth2_heatmap}
\end{figure}

\begin{figure}[t]
\centering
\subfigure[Ground-truth structure.]{\includegraphics[width=0.3\linewidth]{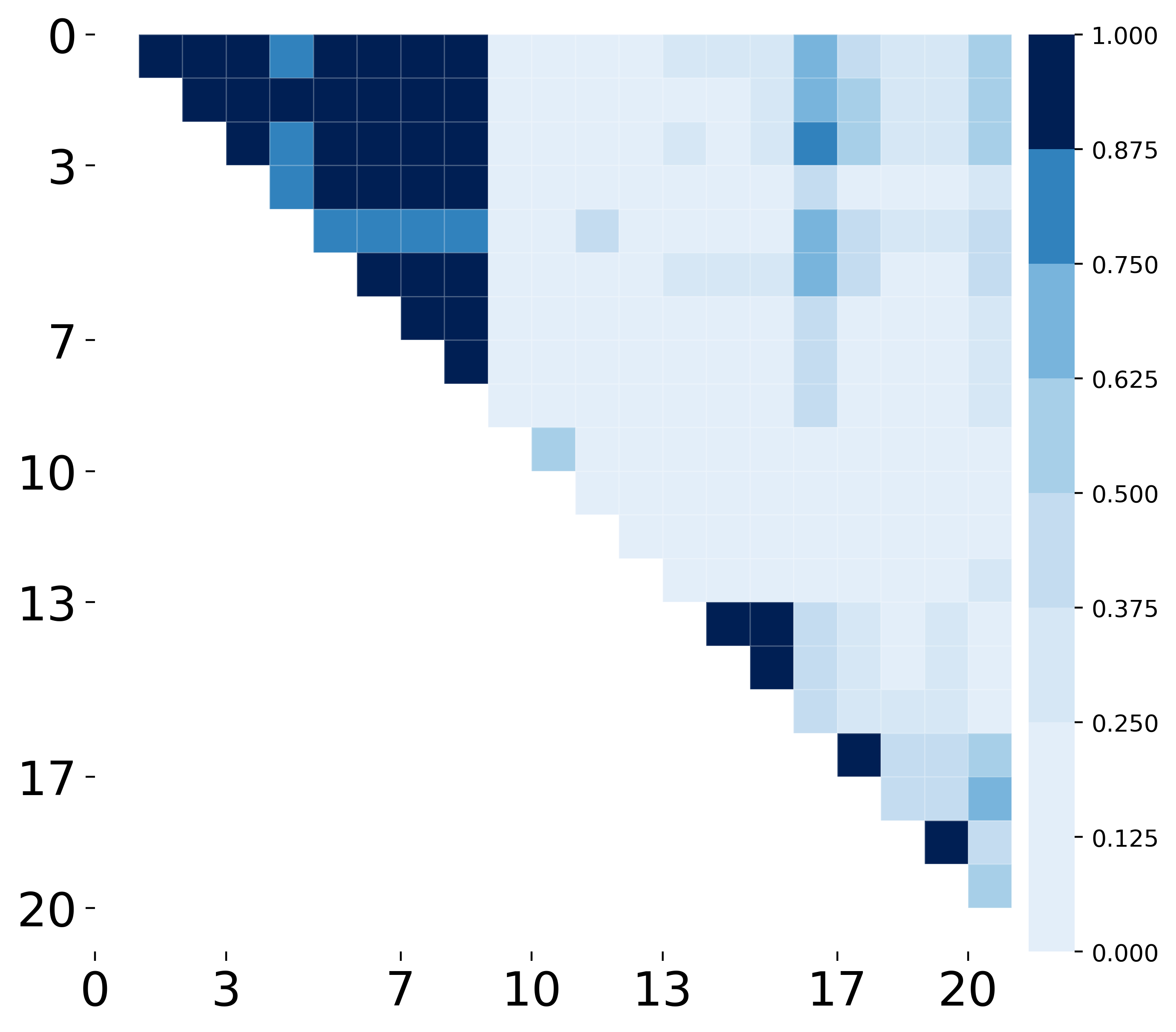}}
\hfill
    \raisebox{-0.0\height}{\rule{0.8pt}{3.5cm}}
\hfill
\subfigure[Baseline structure.]{\includegraphics[width=0.3\linewidth]{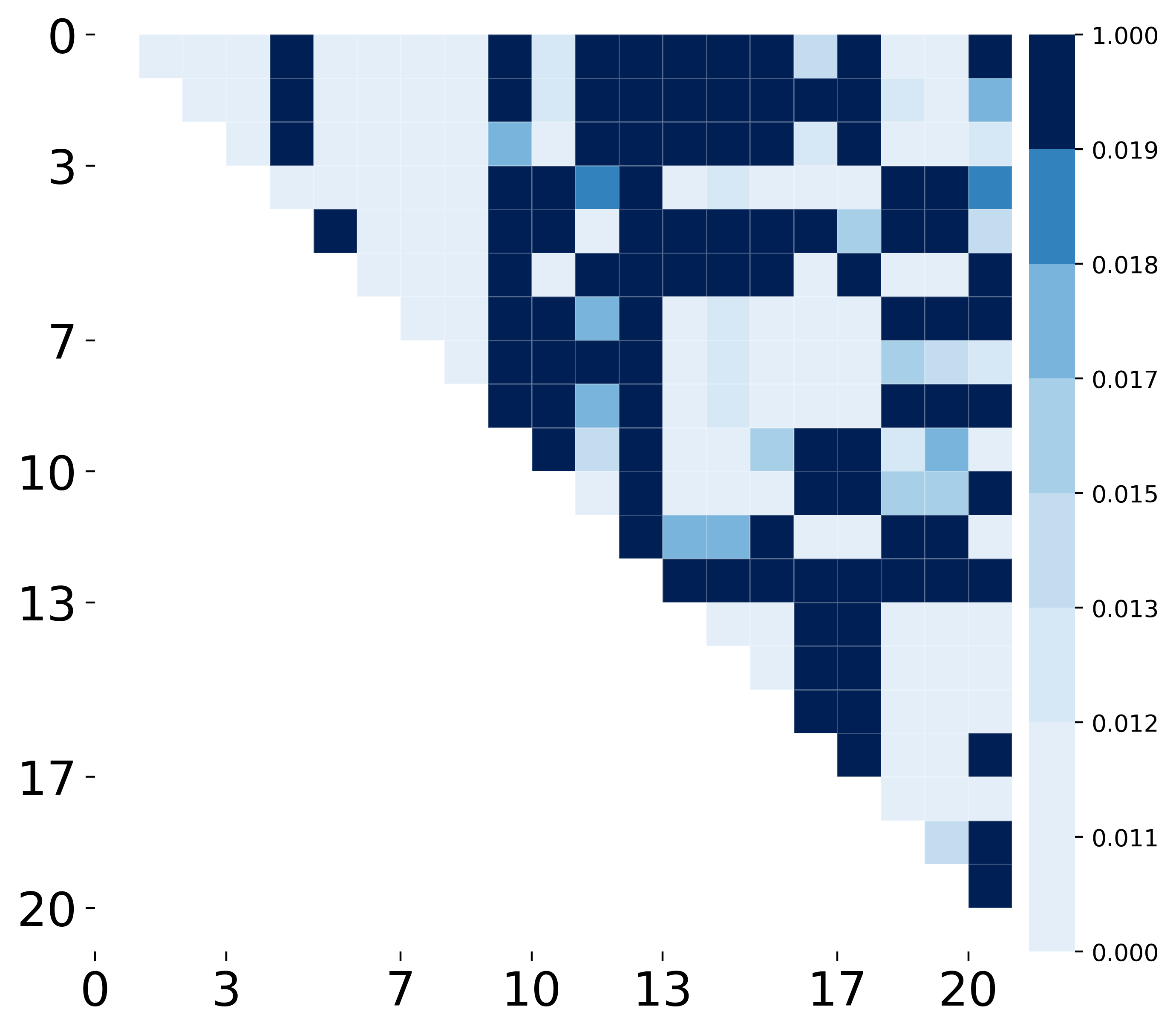}}
\hfill
    \raisebox{-0.0\height}{\rule{0.8pt}{3.5cm}}
\hfill
\subfigure[Improvement by CvLoss.]{\includegraphics[width=0.3\linewidth]{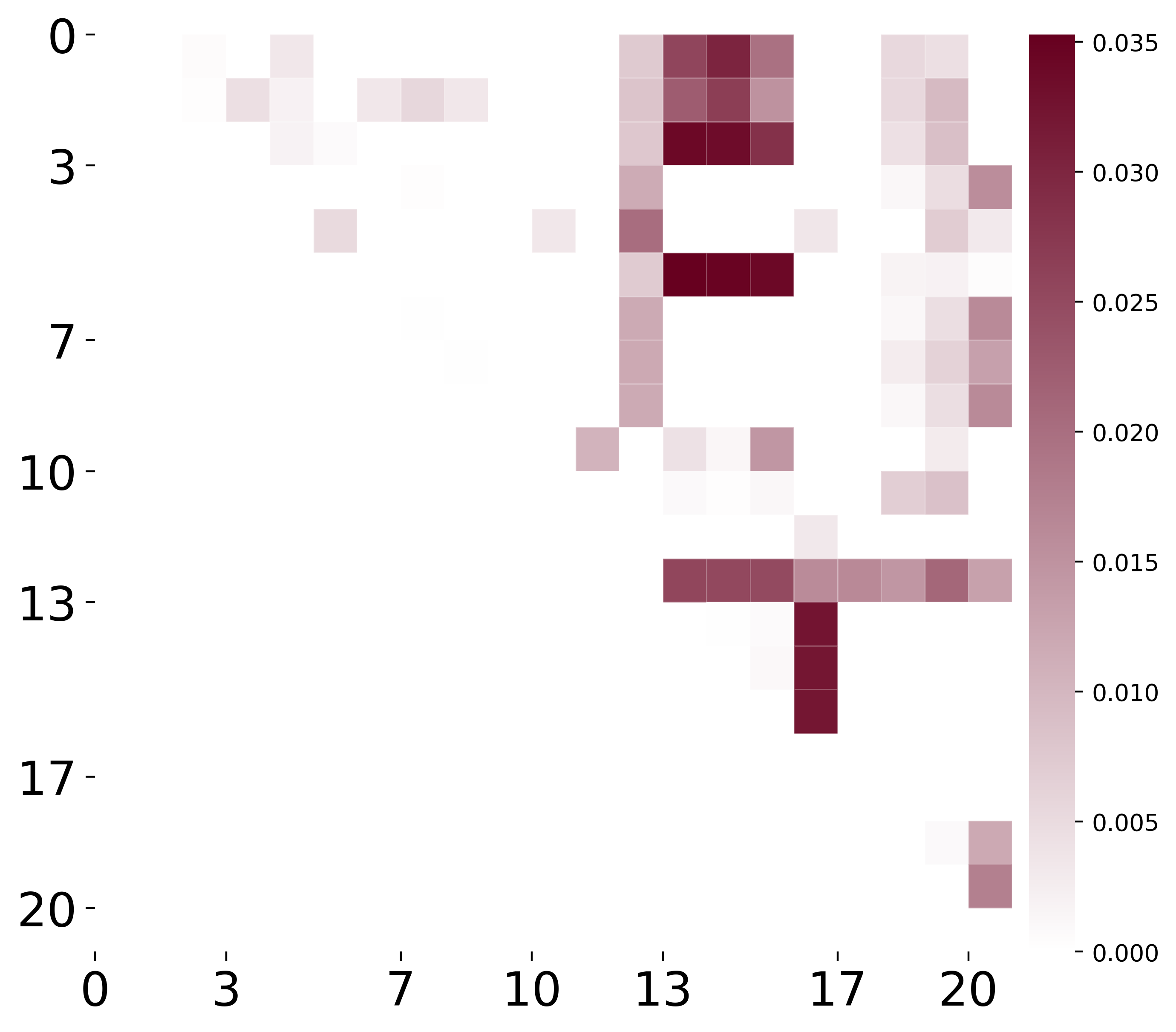}}
\caption{The cross-variable correlation structures of future series given $X$ on the Weather dataset, with input length $H=96$ and forecast horizon $T=96$. The correlation matrices are computed from the ground-truth future series (a) and the future series predicted by iTransformer~\citep{itransformer} (b), while (c) shows the entries on which CvLoss yields smaller structural errors than iTransformer(44\%).}
\label{fig:weather_heatmap}
\end{figure}

\begin{table}[htbp]
\caption{The comprehensive results on the long-term forecasting task. }\label{tab:longterm_app}
\renewcommand{\arraystretch}{1.05}
\setlength{\tabcolsep}{1.5pt}
\scriptsize
\centering
\begin{threeparttable}
\resizebox{\linewidth}{!}{
\begin{tabular}{lc|cc|cc|cc|cc|cc|cc|cc|cc|cc|cc|cc}

\toprule
  \multicolumn{2}{c}{\multirow{2}{*}{\scalebox{1.1}{Models}}} & \multicolumn{2}{c}{\time} & \multicolumn{2}{c}{TimeFilter} & \multicolumn{2}{c}{TQNet} & \multicolumn{2}{c}{iTransformer} & \multicolumn{2}{c}{Leddam} & \multicolumn{2}{c}{SOFTS} & \multicolumn{2}{c}{PatchTST} & \multicolumn{2}{c}{Crossformer} & \multicolumn{2}{c}{TimesNet} & \multicolumn{2}{c}{DLinear} & \multicolumn{2}{c}{FEDformer} \\ 
 
  \multicolumn{2}{c}{} & \multicolumn{2}{c}{\scalebox{0.8}{\textbf{(Ours)}}} & \multicolumn{2}{c}{\scalebox{0.8}{(2025)}} & \multicolumn{2}{c}{\scalebox{0.8}{(2025)}} & \multicolumn{2}{c}{\scalebox{0.8}{(2024)}} & \multicolumn{2}{c}{\scalebox{0.8}{(2024)}} & \multicolumn{2}{c}{\scalebox{0.8}{(2023)}} & \multicolumn{2}{c}{\scalebox{0.8}{(2023)}} & \multicolumn{2}{c}{\scalebox{0.8}{(2023)}} & \multicolumn{2}{c}{\scalebox{0.8}{(2023)}} & \multicolumn{2}{c}{\scalebox{0.8}{(2023)}} & \multicolumn{2}{c}{\scalebox{0.8}{(2022)}} \\

  \cmidrule(lr){3-4} \cmidrule(lr){5-6} \cmidrule(lr){7-8} \cmidrule(lr){9-10} \cmidrule(lr){11-12} \cmidrule(lr){13-14} \cmidrule(lr){15-16} \cmidrule(lr){17-18} \cmidrule(lr){19-20} \cmidrule(lr){21-22} \cmidrule(lr){23-24}

  \multicolumn{2}{c}{Metric} & \scalebox{0.85}{MSE} & \scalebox{0.85}{MAE} & \scalebox{0.85}{MSE} & \scalebox{0.85}{MAE} & \scalebox{0.85}{MSE} & \scalebox{0.85}{MAE} & \scalebox{0.85}{MSE} & \scalebox{0.85}{MAE} & \scalebox{0.85}{MSE} & \scalebox{0.85}{MAE} & \scalebox{0.85}{MSE} & \scalebox{0.85}{MAE} & \scalebox{0.85}{MSE} & \scalebox{0.85}{MAE} & \scalebox{0.85}{MSE} & \scalebox{0.85}{MAE} & \scalebox{0.85}{MSE} & \scalebox{0.85}{MAE} & \scalebox{0.85}{MSE} & \scalebox{0.85}{MAE} & \scalebox{0.85}{MSE} & \scalebox{0.85}{MAE} \\

\toprule

\multirow{5}{*}{\rotatebox[origin=c]{90}{ETTm1}} 
& 96   & \best{0.305} & \best{0.339} & 0.313 & 0.354 & \second{0.311} & \second{0.353} & 0.334 & 0.368 & 0.319 & 0.359 & 0.325 & 0.361 & 0.329 & 0.367 & 0.404 & 0.426 & 0.338 & 0.375 & 0.346 & 0.374 & 0.379 & 0.419 \\

& 192  & \best{0.353} & \best{0.367} & \second{0.356} & 0.380 & 0.357 & \second{0.378} & 0.377 & 0.391 & 0.369 & 0.383 & 0.375 & 0.389 & 0.367 & 0.385 & 0.450 & 0.451 & 0.374 & 0.387 & 0.382 & 0.391 & 0.426 & 0.441 \\

& 336  & \best{0.381} & \best{0.388} & \second{0.386} & 0.403 & 0.390 & \second{0.401} & 0.426 & 0.420 & 0.394 & 0.402 & 0.405 & 0.412 & 0.399 & 0.410 & 0.532 & 0.515 & 0.410 & 0.411 & 0.415 & 0.415 & 0.445 & 0.459 \\

& 720  & \second{0.450} & \best{0.430} & 0.452 & \second{0.437} & \best{0.449} & 0.439 & 0.491 & 0.459 & 0.460 & 0.442 & 0.466 & 0.447 & 0.454 & 0.439 & 0.666 & 0.589 & 0.478 & 0.450 & 0.473 & 0.451 & 0.543 & 0.490 \\

\cmidrule(lr){2-24}

& \emph{Avg} & \best{0.372} & \best{0.381} & \second{0.377} & 0.394 & \second{0.377} & \second{0.393} & 0.407 & 0.410 & 0.386 & 0.397 & 0.393 & 0.403 & 0.387 & 0.400 & 0.513 & 0.496 & 0.400 & 0.406 & 0.404 & 0.408 & 0.448 & 0.452 \\

\midrule

\multirow{5}{*}{\rotatebox[origin=c]{90}{ETTm2}} 

& 96   & \best{0.167} & \best{0.251} & \second{0.169} & \second{0.255} & 0.173 & \second{0.255} & 0.180 & 0.264 & 0.176 & 0.257 & 0.180 & 0.261 & 0.175 & 0.259 & 0.287 & 0.366 & 0.187 & 0.267 & 0.193 & 0.293 & 0.203 & 0.287 \\

& 192  & \best{0.232} & \best{0.294} & \second{0.235} & \second{0.299} & 0.243 & 0.300 & 0.250 & 0.309 & 0.243 & 0.303 & 0.246 & 0.306 & 0.241 & 0.302 & 0.414 & 0.492 & 0.249 & 0.309 & 0.284 & 0.361 & 0.269 & 0.328 \\

& 336  & \best{0.290} & \best{0.332} & \second{0.293} & \second{0.336} & 0.301 & 0.340 & 0.311 & 0.348 & 0.303 & 0.341 & 0.319 & 0.352 & 0.305 & 0.343 & 0.597 & 0.542 & 0.321 & 0.351 & 0.382 & 0.429 & 0.325 & 0.366 \\

& 720  & \best{0.389} & \best{0.391} & \second{0.390} & \second{0.393} & 0.395 & 0.394 & 0.412 & 0.407 & 0.400 & 0.398 & 0.405 & 0.401 & 0.402 & 0.400 & 1.730 & 1.042 & 0.408 & 0.403 & 0.558 & 0.525 & 0.421 & 0.415 \\

\cmidrule(lr){2-24}

& \emph{Avg} & \best{0.270} & \best{0.317} & \second{0.272} & \second{0.321} & 0.278 & 0.322 & 0.288 & 0.332 & 0.281 & 0.325 & 0.287 & 0.330 & 0.281 & 0.326 & 0.757 & 0.610 & 0.291 & 0.333 & 0.354 & 0.402 & 0.305 & 0.349 \\

\midrule

\multirow{5}{*}{\rotatebox[origin=c]{90}{ETTh1}} 

& 96   & \best{0.369} & \best{0.391} & \second{0.370} & 0.394 & 0.371 & \second{0.393} & 0.386 & 0.405 & 0.377 & 0.394 & 0.381 & 0.399 & 0.414 & 0.419 & 0.423 & 0.448 & 0.384 & 0.402 & 0.397 & 0.412 & 0.376 & 0.419 \\

& 192  & \best{0.412} & \best{0.420} & \second{0.413} & \best{0.420} & 0.428 & 0.426 & 0.441 & 0.436 & 0.424 & \second{0.422} & 0.435 & 0.431 & 0.460 & 0.445 & 0.471 & 0.474 & 0.436 & 0.429 & 0.446 & 0.441 & 0.420 & 0.448 \\

& 336  & \best{0.447} & \best{0.439} & \second{0.450} & \second{0.440} & 0.475 & 0.446 & 0.487 & 0.458 & 0.459 & 0.442 & 0.480 & 0.452 & 0.501 & 0.466 & 0.570 & 0.546 & 0.491 & 0.469 & 0.489 & 0.467 & 0.459 & 0.465 \\

& 720  & \best{0.447} & \best{0.457} & \second{0.448} & \best{0.457} & 0.488 & 0.470 & 0.503 & 0.491 & 0.463 & \second{0.459} & 0.499 & 0.488 & 0.500 & 0.488 & 0.653 & 0.621 & 0.521 & 0.500 & 0.513 & 0.510 & 0.506 & 0.507 \\

\cmidrule(lr){2-24}

& \emph{Avg} & \best{0.419} & \best{0.427} & \second{0.420} & \second{0.428} & 0.441 & 0.434 & 0.454 & 0.447 & 0.431 & 0.429 & 0.449 & 0.442 & 0.469 & 0.454 & 0.529 & 0.522 & 0.458 & 0.450 & 0.461 & 0.457 & 0.440 & 0.460 \\

\midrule

\multirow{5}{*}{\rotatebox[origin=c]{90}{ETTh2}} 

& 96   & \best{0.282} & \best{0.334} & \second{0.283} & 0.338 & 0.294 & 0.344 & 0.297 & 0.349 & 0.292 & 0.343 & 0.297 & 0.347 & 0.302 & 0.348 & 0.745 & 0.584 & 0.340 & 0.374 & 0.340 & 0.394 & 0.358 & 0.397 \\

& 192  & \best{0.356} & \best{0.383} & \second{0.362} & 0.392 & 0.368 & 0.393 & 0.380 & 0.400 & 0.367 & \second{0.389} & 0.373 & 0.394 & 0.388 & 0.400 & 0.877 & 0.656 & 0.402 & 0.414 & 0.482 & 0.479 & 0.429 & 0.439 \\

& 336  & \best{0.401} & \best{0.419} & \second{0.409} & 0.431 & 0.417 & 0.428 & 0.428 & 0.432 & 0.412 & \second{0.424} & 0.410 & 0.426 & 0.426 & 0.433 & 1.043 & 0.731 & 0.452 & 0.452 & 0.591 & 0.541 & 0.496 & 0.487 \\

& 720  & \best{0.397} & \best{0.426} & \second{0.408} & 0.434 & 0.434 & 0.446 & 0.427 & 0.445 & 0.419 & 0.438 & 0.411 & \second{0.433} & 0.431 & 0.446 & 1.104 & 0.763 & 0.462 & 0.468 & 0.839 & 0.661 & 0.463 & 0.474 \\

\cmidrule(lr){2-24}

& \emph{Avg} & \best{0.359} & \best{0.391} & \second{0.366} & 0.399 & 0.378 & 0.403 & 0.383 & 0.407 & 0.373 & 0.399 & 0.385 & 0.408 & 0.387 & 0.407 & 0.942 & 0.684 & 0.414 & 0.427 & 0.563 & 0.519 & 0.437 & 0.449 \\

\midrule

\multirow{5}{*}{\rotatebox[origin=c]{90}{Weather}} 

& 96   & \best{0.149} & \best{0.188} & \second{0.153} & \second{0.199} & 0.156 & 0.200 & 0.174 & 0.214 & 0.156 & 0.202 & 0.166 & 0.208 & 0.177 & 0.218 & 0.158 & 0.230 & 0.172 & 0.220 & 0.195 & 0.252 & 0.217 & 0.296 \\

& 192  & \best{0.198} & \best{0.236} & \second{0.202} & 0.246 & 0.205 & \second{0.245} & 0.221 & 0.254 & 0.207 & 0.250 & 0.217 & 0.253 & 0.225 & 0.259 & 0.206 & 0.277 & 0.219 & 0.261 & 0.237 & 0.295 & 0.276 & 0.336 \\

& 336  & \best{0.257} & \best{0.283} & \second{0.260} & 0.289 & 0.262 & \second{0.287} & 0.278 & 0.296 & 0.262 & 0.291 & 0.282 & 0.300 & 0.278 & 0.297 & 0.272 & 0.335 & 0.280 & 0.306 & 0.282 & 0.331 & 0.339 & 0.380 \\

& 720  & \best{0.340} & \best{0.334} & 0.345 & 0.344 & \second{0.343} & \second{0.342} & 0.358 & 0.349 & 0.343 & 0.343 & 0.356 & 0.351 & 0.354 & 0.348 & 0.398 & 0.418 & 0.365 & 0.359 & 0.345 & 0.382 & 0.403 & 0.428 \\

\cmidrule(lr){2-24}

& \emph{Avg} & \best{0.236} & \best{0.260} & \second{0.240} & 0.270 & 0.242 & \second{0.268} & 0.258 & 0.279 & 0.242 & 0.272 & 0.255 & 0.278 & 0.259 & 0.281 & 0.259 & 0.315 & 0.259 & 0.287 & 0.265 & 0.315 & 0.309 & 0.360 \\

\midrule

\multirow{5}{*}{\rotatebox[origin=c]{90}{Electricity}} 

& 96   & \best{0.131} & \best{0.225} & \second{0.133} & 0.230 & 0.134 & \second{0.229} & 0.148 & 0.240 & 0.141 & 0.235 & 0.143 & 0.233 & 0.195 & 0.285 & 0.219 & 0.314 & 0.168 & 0.272 & 0.210 & 0.302 & 0.193 & 0.308 \\

& 192  & \best{0.152} & \best{0.243} & \second{0.154} & 0.248 & \second{0.154} & \second{0.247} & 0.162 & 0.253 & 0.159 & 0.252 & 0.158 & 0.248 & 0.199 & 0.289 & 0.231 & 0.322 & 0.184 & 0.289 & 0.210 & 0.305 & 0.201 & 0.315 \\

& 336  & \best{0.162} & \best{0.257} & \second{0.164} & 0.262 & 0.170 & 0.265 & 0.178 & 0.269 & 0.173 & 0.268 & 0.178 & 0.269 & 0.215 & 0.305 & 0.246 & 0.337 & 0.198 & 0.300 & 0.223 & 0.319 & 0.214 & 0.329 \\

& 720  & \best{0.182} & \best{0.282} & \second{0.184} & \second{0.284} & 0.200 & 0.293 & 0.225 & 0.317 & 0.201 & 0.295 & 0.218 & 0.305 & 0.256 & 0.337 & 0.280 & 0.363 & 0.220 & 0.320 & 0.258 & 0.350 & 0.246 & 0.355 \\

\cmidrule(lr){2-24}

& \emph{Avg} & \best{0.157} & \best{0.252} & \second{0.159} & \second{0.256} & 0.164 & 0.258 & 0.178 & 0.270 & 0.169 & 0.263 & 0.174 & 0.264 & 0.216 & 0.304 & 0.244 & 0.334 & 0.193 & 0.295 & 0.225 & 0.319 & 0.214 & 0.327 \\

\midrule

\multirow{5}{*}{\rotatebox[origin=c]{90}{Traffic}} 

& 96   & \best{0.372} & \best{0.238} & \second{0.375} & \second{0.251} & 0.413 & 0.261 & 0.395 & 0.268 & 0.426 & 0.276 & 0.376 & \second{0.251} & 0.544 & 0.359 & 0.522 & 0.290 & 0.593 & 0.321 & 0.650 & 0.396 & 0.587 & 0.366 \\

& 192  & \best{0.395} & \best{0.248} & \second{0.396} & 0.262 & 0.432 & 0.271 & 0.417 & 0.276 & 0.458 & 0.289 & 0.398 & \second{0.261} & 0.540 & 0.354 & 0.530 & 0.293 & 0.617 & 0.336 & 0.598 & 0.370 & 0.604 & 0.373 \\

& 336  & \best{0.413} & \best{0.256} & 0.416 & 0.271 & 0.450 & 0.277 & 0.433 & 0.283 & 0.486 & 0.297 & \second{0.415} & \second{0.269} & 0.551 & 0.358 & 0.558 & 0.305 & 0.629 & 0.336 & 0.605 & 0.373 & 0.621 & 0.383 \\

& 720  & \best{0.446} & \best{0.275} & \best{0.446} & 0.290 & 0.486 & 0.295 & 0.467 & 0.302 & 0.498 & 0.313 & \second{0.447} & \second{0.287} & 0.586 & 0.375 & 0.589 & 0.328 & 0.640 & 0.350 & 0.645 & 0.394 & 0.626 & 0.382 \\

\cmidrule(lr){2-24}

& \emph{Avg} & \best{0.407} & \best{0.254} & \second{0.408} & 0.269 & 0.445 & 0.276 & 0.428 & 0.282 & 0.467 & 0.294 & 0.409 & \second{0.267} & 0.555 & 0.362 & 0.550 & 0.304 & 0.620 & 0.336 & 0.625 & 0.383 & 0.610 & 0.376 \\

\midrule

\multirow{5}{*}{\rotatebox[origin=c]{90}{Solar-Energy}} 

& 96   & \second{0.183} & \best{0.229} & 0.199 & 0.239 & \best{0.175} & 0.234 & 0.203 & 0.237 & 0.197 & 0.241 & 0.200 & \second{0.230} & 0.234 & 0.286 & 0.310 & 0.331 & 0.250 & 0.292 & 0.290 & 0.378 & 0.242 & 0.342 \\

& 192  & \second{0.215} & \best{0.250} & 0.227 & 0.261 & \best{0.198} & 0.255 & 0.233 & 0.261 & 0.231 & 0.263 & 0.229 & \second{0.253} & 0.267 & 0.310 & 0.734 & 0.725 & 0.296 & 0.318 & 0.320 & 0.398 & 0.285 & 0.380 \\

& 336  & \second{0.232} & \second{0.267} & 0.240 & 0.272 & \best{0.205} & \best{0.262} & 0.248 & 0.273 & 0.241 & 0.268 & 0.243 & 0.269 & 0.290 & 0.315 & 0.750 & 0.735 & 0.319 & 0.330 & 0.353 & 0.415 & 0.282 & 0.376 \\

& 720  & \second{0.241} & \best{0.270} & 0.247 & 0.277 & \best{0.210} & \second{0.270} & 0.249 & 0.275 & 0.250 & 0.281 & 0.245 & 0.272 & 0.289 & 0.317 & 0.769 & 0.765 & 0.338 & 0.337 & 0.357 & 0.413 & 0.357 & 0.427 \\

\cmidrule(lr){2-24}

& \emph{Avg} & \second{0.218} & \best{0.254} & 0.228 & 0.262 & \best{0.197} & \second{0.255} & 0.233 & 0.262 & 0.230 & 0.264 & 0.229 & 0.256 & 0.270 & 0.307 & 0.641 & 0.639 & 0.301 & 0.319 & 0.330 & 0.401 & 0.292 & 0.381 \\

\midrule

\multicolumn{2}{c|}{{{$1^{\text{st}}$ Count}}} & \best{34} & \best{39} & 1 & 2 & \second{6} & 1 & 0 & 0 & 0 & 0 & 0 & 0 & 0 & 0 & 0 & 0 & 0 & 0 & 0 & 0 & 0 & 0 \\

\bottomrule
\end{tabular}
}
\noindent\parbox{\linewidth}{\scriptsize \textit{Note}: We fix the input length as 96 following~\citep{itransformer}. \best{Bold} typeface highlights the top performance for each metric, while \second{underlined} text denotes the second-best results. \emph{Avg} indicates the results averaged over forecasting lengths: T=96, 192, 336 and 720.}
\end{threeparttable}
\end{table}

\begin{table*}[htbp]
\caption{Full results of short-term forecasting.}\label{tab::app_full_short_result}
\renewcommand{\arraystretch}{1.05}
\setlength{\tabcolsep}{1.5pt}
\scriptsize
\centering
\begin{threeparttable}
\resizebox{\linewidth}{!}{
\begin{tabular}{lc|cc|cc|cc|cc|cc|cc|cc|cc|cc|cc|cc}

\toprule
\multicolumn{2}{c}{\multirow{2}{*}{\scalebox{1.1}{Models}}} & \multicolumn{2}{c}{\time} & \multicolumn{2}{c}{TimeFilter} & \multicolumn{2}{c}{TQNet} & \multicolumn{2}{c}{iTransformer} & \multicolumn{2}{c}{Leddam} & \multicolumn{2}{c}{SOFTS} & \multicolumn{2}{c}{PatchTST} & \multicolumn{2}{c}{Crossformer} & \multicolumn{2}{c}{TimesNet} & \multicolumn{2}{c}{DLinear} & \multicolumn{2}{c}{FEDformer} \\ 

\multicolumn{2}{c}{} & \multicolumn{2}{c}{\scalebox{0.8}{\textbf{(Ours)}}} & \multicolumn{2}{c}{\scalebox{0.8}{(2025)}} & \multicolumn{2}{c}{\scalebox{0.8}{(2025)}} & \multicolumn{2}{c}{\scalebox{0.8}{(2024)}} & \multicolumn{2}{c}{\scalebox{0.8}{(2024)}} & \multicolumn{2}{c}{\scalebox{0.8}{(2023)}} & \multicolumn{2}{c}{\scalebox{0.8}{(2023)}} & \multicolumn{2}{c}{\scalebox{0.8}{(2023)}} & \multicolumn{2}{c}{\scalebox{0.8}{(2023)}} & \multicolumn{2}{c}{\scalebox{0.8}{(2023)}} & \multicolumn{2}{c}{\scalebox{0.8}{(2022)}} \\

\cmidrule(lr){3-4} \cmidrule(lr){5-6} \cmidrule(lr){7-8} \cmidrule(lr){9-10} \cmidrule(lr){11-12} \cmidrule(lr){13-14} \cmidrule(lr){15-16} \cmidrule(lr){17-18} \cmidrule(lr){19-20} \cmidrule(lr){21-22} \cmidrule(lr){23-24}

\multicolumn{2}{c}{Metric} & MSE & MAE & MSE & MAE & MSE & MAE & MSE & MAE & MSE & MAE & MSE & MAE & MSE & MAE & MSE & MAE & MSE & MAE & MSE & MAE & MSE & MAE \\

\toprule

\multirow{4}{*}{\rotatebox[origin=c]{90}{PEMS03}} 
& 12  & \best{0.060} & \best{0.160} & 0.063 & 0.165 & \second{0.060} & \second{0.161} & 0.071 & 0.174 & 0.068 & 0.174 & 0.064 & 0.165 & 0.099 & 0.216 & 0.090 & 0.203 & 0.085 & 0.192 & 0.122 & 0.243 & 0.126 & 0.251 \\
& 24  & \best{0.075} & \best{0.180} & 0.079 & 0.185 & \second{0.077} & \second{0.183} & 0.093 & 0.201 & 0.094 & 0.202 & 0.083 & 0.188 & 0.142 & 0.259 & 0.121 & 0.240 & 0.118 & 0.223 & 0.201 & 0.317 & 0.149 & 0.275 \\
& 48  & \best{0.101} & \best{0.210} & 0.110 & 0.222 & \second{0.103} & \second{0.213} & 0.125 & 0.236 & 0.140 & 0.254 & 0.114 & 0.223 & 0.211 & 0.319 & 0.202 & 0.317 & 0.155 & 0.260 & 0.333 & 0.425 & 0.227 & 0.348 \\
\cmidrule(lr){2-24}
& \emph{Avg} & \best{0.079} & \best{0.183} & 0.084 & 0.191 & \second{0.080} & \second{0.186} & 0.096 & 0.204 & 0.101 & 0.210 & 0.087 & 0.192 & 0.151 & 0.265 & 0.138 & 0.253 & 0.119 & 0.271 & 0.219 & 0.295 & 0.167 & 0.291 \\

\midrule

\multirow{4}{*}{\rotatebox[origin=c]{90}{PEMS04}} 
& 12  & \best{0.066} & \best{0.164} & 0.068 & 0.167 & \second{0.067} & \second{0.166} & 0.078 & 0.183 & 0.076 & 0.182 & 0.074 & 0.176 & 0.105 & 0.224 & 0.098 & 0.218 & 0.087 & 0.195 & 0.148 & 0.272 & 0.138 & 0.262 \\
& 24  & \best{0.075} & \best{0.178} & 0.080 & 0.183 & \second{0.077} & \second{0.181} & 0.095 & 0.205 & 0.097 & 0.209 & 0.088 & 0.194 & 0.153 & 0.257 & 0.131 & 0.256 & 0.103 & 0.215 & 0.224 & 0.340 & 0.177 & 0.293 \\
& 48  & \best{0.094} & \best{0.201} & 0.101 & 0.209 & \second{0.098} & \second{0.207} & 0.120 & 0.233 & 0.132 & 0.249 & 0.110 & 0.219 & 0.229 & 0.339 & 0.205 & 0.326 & 0.136 & 0.250 & 0.335 & 0.437 & 0.270 & 0.368 \\
\cmidrule(lr){2-24}
& \emph{Avg} & \best{0.078} & \best{0.181} & 0.083 & 0.186 & \second{0.081} & \second{0.185} & 0.098 & 0.207 & 0.102 & 0.213 & 0.091 & 0.196 & 0.162 & 0.273 & 0.145 & 0.267 & 0.109 & 0.220 & 0.236 & 0.350 & 0.195 & 0.308 \\

\midrule

\multirow{4}{*}{\rotatebox[origin=c]{90}{PEMS07}} 
& 12  & \best{0.051} & \best{0.140} & 0.055 & 0.150 & \second{0.052} & \second{0.144} & 0.067 & 0.165 & 0.066 & 0.164 & 0.057 & 0.152 & 0.095 & 0.207 & 0.094 & 0.200 & 0.082 & 0.181 & 0.115 & 0.242 & 0.109 & 0.225 \\
& 24  & \best{0.061} & \best{0.155} & 0.068 & 0.166 & \second{0.063} & \second{0.159} & 0.088 & 0.190 & 0.079 & 0.185 & 0.073 & 0.173 & 0.150 & 0.262 & 0.139 & 0.247 & 0.101 & 0.204 & 0.210 & 0.329 & 0.125 & 0.244 \\
& 48  & \best{0.077} & \best{0.174} & 0.089 & 0.193 & \second{0.080} & \second{0.180} & 0.110 & 0.215 & 0.115 & 0.228 & 0.096 & 0.195 & 0.253 & 0.340 & 0.311 & 0.369 & 0.134 & 0.238 & 0.398 & 0.458 & 0.165 & 0.288 \\
\cmidrule(lr){2-24}
& \emph{Avg} & \best{0.063} & \best{0.156} & 0.071 & 0.170 & \second{0.065} & \second{0.161} & 0.088 & 0.190 & 0.087 & 0.192 & 0.075 & 0.173 & 0.166 & 0.270 & 0.181 & 0.272 & 0.106 & 0.208 & 0.241 & 0.343 & 0.133 & 0.282 \\

\midrule

\multirow{4}{*}{\rotatebox[origin=c]{90}{PEMS08}} 
& 12  & \best{0.062} & \best{0.159} & \second{0.064} & \second{0.162} & 0.071 & 0.170 & 0.079 & 0.182 & 0.070 & 0.173 & 0.074 & 0.171 & 0.168 & 0.232 & 0.165 & 0.214 & 0.112 & 0.212 & 0.227 & 0.343 & 0.173 & 0.273 \\
& 24  & \best{0.076} & \best{0.178} & \second{0.079} & \second{0.182} & 0.095 & 0.195 & 0.115 & 0.219 & 0.091 & 0.200 & 0.104 & 0.201 & 0.224 & 0.281 & 0.215 & 0.260 & 0.141 & 0.238 & 0.318 & 0.409 & 0.210 & 0.310 \\
& 48  & \best{0.101} & \best{0.210} & \second{0.105} & \second{0.214} & 0.149 & 0.243 & 0.186 & 0.235 & 0.145 & 0.261 & 0.164 & 0.253 & 0.321 & 0.354 & 0.315 & 0.335 & 0.198 & 0.283 & 0.497 & 0.510 & 0.320 & 0.394 \\
\cmidrule(lr){2-24}
& \emph{Avg} & \best{0.080} & \best{0.182} & \second{0.083} & \second{0.186} & 0.105 & 0.203 & 0.127 & 0.212 & 0.102 & 0.211 & 0.114 & 0.208 & 0.238 & 0.289 & 0.232 & 0.270 & 0.150 & 0.244 & 0.347 & 0.421 & 0.234 & 0.326 \\

\bottomrule

\end{tabular}
}

\noindent\parbox{\linewidth}{\scriptsize \textit{Note}: The input sequence length is set to 96 for all baselines. All results are averaged across four different forecasting horizons: $T \in \{12, 24, 48\}$. The best and second-best results are highlighted in \best{bold} and \second{underlined}, respectively.}

\end{threeparttable}
\end{table*}

\begin{table}
    \centering
    \caption{Backbone-controlled comparison on four forecasting backbones: each pair contrasts the unmodified model with the identical model trained with \time. The \emph{Wins} row counts, for each backbone, the horizon--dataset--metric cells on which \time\ is better, out of $3\times4\times2=24$. These counts are produced by the evaluation script from the \emph{unrounded} outputs, not by comparing the three-decimal values printed here: seven cells are equal at the displayed precision but are improvements at the fourth decimal, so the totals of $24/24$ per backbone and $96/96$ overall are exact. We keep three decimals for consistency with the other tables rather than widening the display.
    }
    \resizebox{\textwidth}{!}
    {
       \begin{tabular}{cc|cccc|cccc|cccc|cccc}
        \toprule
        \multicolumn{2}{c|}{\multirow{2}{*}{Method}} & \multicolumn{2}{c}{TQNet} & \multicolumn{2}{c|}{\textbf{+ Ours}} & \multicolumn{2}{c}{TimeFilter} & \multicolumn{2}{c|}{\textbf{+ Ours}} & \multicolumn{2}{c}{CFPT} & \multicolumn{2}{c}{\textbf{+ Ours}} & \multicolumn{2}{c}{iTransformer} & \multicolumn{2}{c}{\textbf{+ Ours}}   \\
        \multicolumn{2}{c|}{} & MSE & MAE & MSE & MAE & MSE & MAE & MSE & MAE & MSE & MAE & MSE & MAE & MSE & MAE & MSE & MAE \\ \midrule
         
         \multirow{5}{*}{\rotatebox{90}{$ETTh2$}}  
         & 96  &0.294  &0.344  & \textbf{0.281}  & \textbf{0.332}  &0.283  &0.337  & \textbf{0.282}  & \textbf{0.334}  &0.287  &0.337  & \textbf{0.283}  & \textbf{0.333}  &0.297  &0.349  & \textbf{0.296}  & \textbf{0.342}  \\
         & 192 &0.368  &0.393  & \textbf{0.358}  & \textbf{0.380}  &0.362  &0.392  & \textbf{0.356}  & \textbf{0.383}  &0.368  &0.390  & \textbf{0.361}  & \textbf{0.383}  &0.380  &0.400  & \textbf{0.374}  & \textbf{0.391} \\
         & 336 &0.417  &0.428  & \textbf{0.404}  & \textbf{0.419}  &0.404  &0.424  & \textbf{0.401}  & \textbf{0.419}  &0.416  &0.428  & \textbf{0.409}  & \textbf{0.420}  &0.428  &0.432  & \textbf{0.415}  & \textbf{0.427}        \\
         & 720 &0.434  &0.446  & \textbf{0.413}  & \textbf{0.433}  &0.407  &0.433  & \textbf{0.397}  & \textbf{0.426}  &0.415  &0.436  & \textbf{0.392}  & \textbf{0.418}  &0.427  &0.445  & \textbf{0.421}  & \textbf{0.438}        \\
         & \cellcolor{tabhighlight} Avg & \cellcolor{tabhighlight} 0.378 & \cellcolor{tabhighlight} 0.403 & \cellcolor{tabhighlight} \textbf{0.364} & \cellcolor{tabhighlight} \textbf{0.391} & \cellcolor{tabhighlight} 0.364 & \cellcolor{tabhighlight} 0.397 & \cellcolor{tabhighlight} \textbf{0.359} & \cellcolor{tabhighlight} \textbf{0.391} & \cellcolor{tabhighlight} 0.372 & \cellcolor{tabhighlight} 0.398 & \cellcolor{tabhighlight} \textbf{0.361} & \cellcolor{tabhighlight} \textbf{0.389} & \cellcolor{tabhighlight} 0.383 & \cellcolor{tabhighlight} 0.407 & \cellcolor{tabhighlight} \textbf{0.377} & \cellcolor{tabhighlight} \textbf{0.399} \\ \midrule

         \multirow{5}{*}{\rotatebox{90}{$Weather$}}  
         & 96  &0.156  &0.200  & \textbf{0.154}  & \textbf{0.192}  &0.153  &0.199  & \textbf{0.149}  & \textbf{0.188}  &0.156  &0.201  & \textbf{0.152}  & \textbf{0.192}  &0.174  &0.214  & \textbf{0.168}  & \textbf{0.207}  \\
         & 192 &0.205  &0.245  & \textbf{0.203}  & \textbf{0.238}  &0.202  &0.246  & \textbf{0.198}  & \textbf{0.236}  &0.205  &0.245  & \textbf{0.203}  & \textbf{0.240}  &0.221  &0.254  & \textbf{0.221}  & \textbf{0.254} \\
         & 336 &0.262  &0.287  & \textbf{0.262}  & \textbf{0.285}  &0.260  &0.289  & \textbf{0.257}  & \textbf{0.283}  &0.261  &0.286  & \textbf{0.260}  & \textbf{0.281}  &0.278  &0.296  & \textbf{0.278}  & \textbf{0.296}        \\
         & 720 &0.343  &0.342  & \textbf{0.343}  & \textbf{0.339}  &0.342  &0.341  & \textbf{0.340}  & \textbf{0.334}  &0.343  &0.339  & \textbf{0.340}  & \textbf{0.336}  &0.358  &0.349  & \textbf{0.357}  & \textbf{0.347}        \\
         & \cellcolor{tabhighlight} Avg & \cellcolor{tabhighlight} 0.242 & \cellcolor{tabhighlight} 0.268 & \cellcolor{tabhighlight} \textbf{0.241} & \cellcolor{tabhighlight} \textbf{0.264} & \cellcolor{tabhighlight} 0.239 & \cellcolor{tabhighlight} 0.269 & \cellcolor{tabhighlight} \textbf{0.236} & \cellcolor{tabhighlight} \textbf{0.260} & \cellcolor{tabhighlight} 0.241 & \cellcolor{tabhighlight} 0.267 & \cellcolor{tabhighlight} \textbf{0.239} & \cellcolor{tabhighlight} \textbf{0.262} & \cellcolor{tabhighlight} 0.258 & \cellcolor{tabhighlight} 0.279 & \cellcolor{tabhighlight} \textbf{0.256} & \cellcolor{tabhighlight} \textbf{0.276} \\ \midrule

         \multirow{5}{*}{\rotatebox{90}{$ECL$}}  
         & 96  &0.134  &0.229  & \textbf{0.133}  & \textbf{0.225}  &0.133  &0.230  & \textbf{0.131}  & \textbf{0.225}  &0.137  &0.232  & \textbf{0.136}  & \textbf{0.227}  &0.148  &0.240  & \textbf{0.146}  & \textbf{0.235}  \\
         & 192 &0.154  &0.247  & \textbf{0.151}  & \textbf{0.241}  &0.154  &0.248  & \textbf{0.152}  & \textbf{0.243}  &0.153  &0.247  & \textbf{0.153}  & \textbf{0.242}  &0.162  &0.253  & \textbf{0.161}  & \textbf{0.249} \\
         & 336 &0.170  &0.265  & \textbf{0.166}  & \textbf{0.258}  &0.164  &0.261  & \textbf{0.162}  & \textbf{0.257}  &0.168  &0.266  & \textbf{0.166}  & \textbf{0.259}  &0.178  &0.269  & \textbf{0.174}  & \textbf{0.264}        \\
         & 720 &0.200  &0.293  & \textbf{0.196}  & \textbf{0.287}  &0.184  &0.284  & \textbf{0.182}  & \textbf{0.282}  &0.199  &0.294  & \textbf{0.196}  & \textbf{0.286}  &0.225  &0.317  & \textbf{0.205}  & \textbf{0.292}        \\
         & \cellcolor{tabhighlight} Avg & \cellcolor{tabhighlight} 0.164 & \cellcolor{tabhighlight} 0.258 & \cellcolor{tabhighlight} \textbf{0.162} & \cellcolor{tabhighlight} \textbf{0.253} & \cellcolor{tabhighlight} 0.158 & \cellcolor{tabhighlight} 0.256 & \cellcolor{tabhighlight} \textbf{0.157} & \cellcolor{tabhighlight} \textbf{0.252} & \cellcolor{tabhighlight} 0.165 & \cellcolor{tabhighlight} 0.260 & \cellcolor{tabhighlight} \textbf{0.163} & \cellcolor{tabhighlight} \textbf{0.254} & \cellcolor{tabhighlight} 0.178 & \cellcolor{tabhighlight} 0.270 & \cellcolor{tabhighlight} \textbf{0.172} & \cellcolor{tabhighlight} \textbf{0.260} \\ \midrule
        
        \rowc\rowcolor{blue!15}
        \multicolumn{2}{c|}{Wins} & \multicolumn{2}{c}{0} & \multicolumn{2}{c|}{\textbf{24}} & \multicolumn{2}{c}{0} & \multicolumn{2}{c|}{\textbf{24}} & \multicolumn{2}{c}{0} & \multicolumn{2}{c|}{\textbf{24}} & \multicolumn{2}{c}{0} & \multicolumn{2}{c}{\textbf{24}} \\
        \bottomrule
        \end{tabular}
    }
    \label{tab2}
\end{table}

\section{More Experimental Results}\label{sec:results_app}

\subsection{Overall performance}\label{sec:overall_app}
Additional overall results are reported in Tables~\ref{tab:longterm_app} and~\ref{tab::app_full_short_result}, where we provide performance under different forecast horizons.

\subsection{Patch length sensitivity}
Additional results for the patch length hyperparameter are reported in Figure~\ref{fig:sensi}, covering the ETTh2 and ETTm2 datasets under the forecast horizons $T=192$ and $T=336$.

\begin{figure}
\begin{center}
\subfigure[ETTh2 with MSE]{\includegraphics[width=0.24\linewidth]{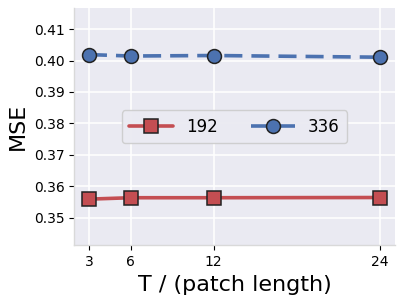}}
\subfigure[ETTh2 with MAE]{\includegraphics[width=0.24\linewidth]{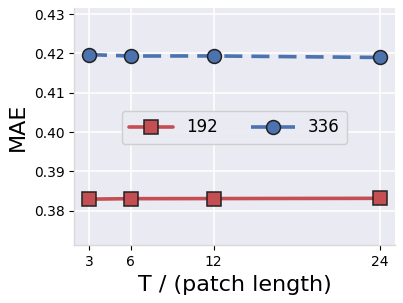}}
\subfigure[ETTm2 with MSE]{\includegraphics[width=0.24\linewidth]{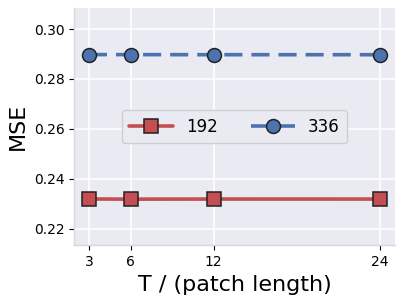}}
\subfigure[ETTm2 with MAE]{\includegraphics[width=0.24\linewidth]{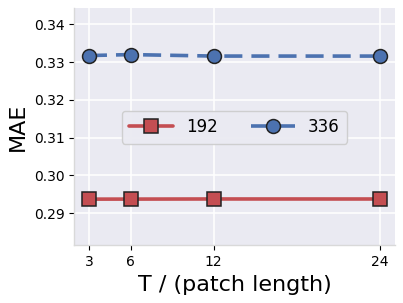}}
\caption{Sensitivity analysis of the patch length hyperparameter. The plots illustrate the forecasting performance in terms of MSE and MAE on the ETTh2 (a, b) and ETTm2 (c, d) datasets. The x-axis represents the ratio of the forecast horizon to the patch length (T / patch length), while the lines denote different forecast horizons ($T=192$ and $T=336$). The nearly horizontal curves demonstrate that the predictive accuracy remains consistent regardless of the specific patch length chosen, indicating that CvLoss is highly robust to this hyperparameter.}
\label{fig:sensi}
\end{center}
\end{figure}

\begin{table}[t]
\begin{minipage}[t]{0.495\textwidth}
\makeatletter\def\@captype{table}
\renewcommand{\arraystretch}{0.9} 
\setlength{\tabcolsep}{4.7pt} 
\scriptsize
\centering
\renewcommand{\multirowsetup}{\centering}
\begin{threeparttable}
\caption{Varying $\alpha$ results of iTransformer}
\label{tab:sensi-iTransformer}
\begin{tabular}{l|cc|cc|cc}
\hline
\multirow{2}{*}{$\alpha$} 
& \multicolumn{2}{c|}{ETTh2} 
& \multicolumn{2}{c|}{ECL} 
& \multicolumn{2}{c}{Weather} \\
\cline{2-7}
& MSE & MAE & MSE & MAE & MSE & MAE \\
\hline
0      & 0.384 & 0.407 & 0.176 & 0.267 & 0.259 & 0.280  \\
0.001  & 0.384 & 0.407 & 0.176 & 0.267 & 0.260 & 0.281  \\
0.002  & 0.384 & 0.407 & 0.176 & 0.267 & 0.260 & 0.281  \\
0.005  & 0.384 & 0.407 & 0.176 & 0.267 & 0.259 & 0.280  \\
0.01   & 0.384 & 0.407 & 0.176 & 0.267 & 0.262 & 0.282  \\
0.02   & 0.384 & 0.407 & 0.176 & 0.267 & 0.261 & 0.281  \\
0.05   & 0.384 & 0.407 & 0.176 & 0.267 & 0.262 & 0.282  \\
0.1    & 0.383 & 0.406 & 0.175 & 0.266 & \subbst{0.259} & 0.280  \\
0.2    & 0.383 & 0.406 & \subbst{0.174} & \subbst{0.264} & 0.259 & \subbst{0.279}  \\
0.5    & \subbst{0.379} & \subbst{0.403} & \bst{0.172} & \bst{0.260} & \bst{0.257} & 0.276  \\
1      & \bst{0.378} & \bst{0.399} & 0.179 & 0.265 & 0.258 & \bst{0.276}  \\
\hline
\end{tabular}
\begin{tablenotes}
\item \tiny \textit{Note}: \bst{Bold} and \subbst{underlined} denote the best and second-best results.
\end{tablenotes}
\end{threeparttable}
\end{minipage}
\hfill
\begin{minipage}[t]{0.495\textwidth}
\makeatletter\def\@captype{table}
\renewcommand{\arraystretch}{0.9} 
\setlength{\tabcolsep}{4.7pt} 
\scriptsize
\centering
\renewcommand{\multirowsetup}{\centering}
\begin{threeparttable}
\caption{Varying $\alpha$ results of TimeBridge.}
\label{tab:sensi-TimeBridge}
\begin{tabular}{l|cc|cc|cc}
\hline
\multirow{2}{*}{$\alpha$} 
& \multicolumn{2}{c|}{ETTh2} 
& \multicolumn{2}{c|}{ECL} 
& \multicolumn{2}{c}{Weather} \\
\cline{2-7}
& MSE & MAE & MSE & MAE & MSE & MAE \\
\hline
0      & 0.366 & 0.408 & 0.159 & 0.258 & 0.221 & 0.261  \\
0.001  & 0.366 & 0.408 & 0.161 & 0.259 & 0.221 & 0.261  \\
0.002  & 0.366 & 0.408 & 0.156 & 0.255 & 0.221 & 0.261  \\
0.005  & 0.366 & 0.408 & 0.159 & 0.258 & 0.221 & 0.261  \\
0.01   & 0.366 & 0.408 & 0.157 & 0.257 & 0.221 & 0.261  \\
0.02   & 0.366 & 0.408 & 0.157 & 0.257 & 0.221 & 0.261  \\
0.05   & 0.365 & 0.407 & 0.157 & 0.255 & 0.221 & 0.261  \\
0.1    & 0.363 & 0.406 & 0.156 & 0.254 & 0.220 & 0.260  \\
0.2    & 0.360 & 0.403 & 0.158 & 0.255 & 0.220 & 0.258  \\
0.5    & \subbst{0.351} & \subbst{0.396} & \bst{0.153} & \subbst{0.250} & \bst{0.220} & \subbst{0.255}  \\
1      & \bst{0.342} & \bst{0.386} & \subbst{0.155} & \bst{0.248} & \subbst{0.220} & \bst{0.251}  \\
\hline
\end{tabular}
\begin{tablenotes}
\item \tiny \textit{Note}: \bst{Bold} and \subbst{underlined} denote the best and second-best results.
\end{tablenotes}
\end{threeparttable}
\end{minipage}
\end{table}

\subsection{Comparison with different learning objectives}
\label{other_learning_objectives}

\paragraph{Baselines.} Since this paper focuses on the design of learning objectives, we select competitive learning objectives tailored for training forecast models as baselines. For the comprehensive evaluation in Table~\ref{tab:multistep_app_full}, we compare against 11 objectives, which can be categorized as follows: \ding{182} \textbf{shape-alignment objectives:} GDTW~\citep{GDTW}, Dilate~\citep{Dilate}, and Soft-DTW~\citep{soft-dtw}; \ding{183} \textbf{likelihood maximization objectives:} QDF~\citep{wang2026iclrqdf}, Time-o1~\citep{wang2025nipstimeo1}, Koopman~\citep{koopman}, FreDF~\citep{wang2025fredf}, and MSE (also denoted as DF); \ding{184} \textbf{distribution balancing objectives:} KMB-DF~\citep{kmbdf}, and DistDF~\citep{wang2026iclrdistdf}; and \ding{185} \textbf{decomposition-based objectives:} DBLoss~\citep{qiu2025DBLoss}. For the cross-model evaluation in Table~\ref{tab:loss-compare}, we select a representative subset of 7 objectives: QDF, Time-o1, FreDF, Koopman, Soft-DTW, DF, and DBLoss. The implementation of baselines follows the official codebase from~\citep{kmbdf}.

\paragraph{Implementation.} To thoroughly assess the effectiveness of our proposed objective, we evaluate it under two distinct experimental settings. 
In the first setting (Table~\ref{tab:multistep_app_full}), we employ CFPT~\citep{CFPT} as the default forecast model to compare all 10 learning objectives across six standard benchmarks (ETTh1, ETTh2, ETTm1, ETTm2, Weather, and ECL). 
In the second setting (Table~\ref{tab:loss-compare}), we adopt TQNet and PDF as the backbone models to verify the generalizability of our method across different architectures. This evaluation is conducted on four datasets (ETTh1, ETTm1, ECL, and Weather) against the aforementioned subset of 6 learning objectives. 
To ensure fair comparison, the drop-last trick is disabled for all baselines, as recommended in~\citep{qiutfb}. All objectives are trained with the Adam optimizer~\citep{Adam}. When integrating \time\ to train a forecast model, we retain all hyperparameters from the public benchmarks~\citep{CFPT}, only tuning $\alpha$ and the segment patch length.

\paragraph{On the weakest competing objectives.} Two of the eleven objectives, Koopman~\citep{koopman} and Soft-DTW~\citep{soft-dtw}, diverge sharply on some datasets. Koopman reaches $0.499$ MSE on ETTm1 with TQNet, and Soft-DTW reaches $4.502$ in Table~\ref{tab:multistep_app_full}, so a reader may reasonably suspect a reproduction failure that would inflate our win counts. It is not one. Our figures reproduce the published numbers of the objective-comparison protocol we adopt: on the PDF backbone, Table 2 of QDF~\citep{wang2026iclrqdf} reports Koopman at $0.587/0.485$ on ETTm1 and Soft-DTW at $0.695/0.548$ on ECL and $1.296/0.452$ on Weather, matching ours. These older shape-alignment objectives simply degrade on some multivariate benchmarks, as their own authors' evaluations show. None of our conclusions rests on them: neither ranks first in any setting, so removing both leaves the first-place count of \time\ in Table~\ref{tab:multistep_app_full} unchanged. The objectives that actually compete with \time\ are QDF, KMB-DF, DBLoss, Time-o1 and FreDF, and they differ from it by a few thousandths.

\begin{table*}
  \caption{Comparable results with different learning objectives.}\label{tab:loss-compare}
  \renewcommand{\arraystretch}{1} \setlength{\tabcolsep}{4.9pt} \scriptsize
  \centering
  \renewcommand{\multirowsetup}{\centering}
  \begin{threeparttable}
  \begin{tabular}{c|c|cc|cc|cc|cc|cc|cc|cc}
    \toprule
    \multicolumn{2}{l}{Loss} & 
    \multicolumn{2}{c}{\textbf{\time}} &
    \multicolumn{2}{c}{\textbf{QDF}} &
    \multicolumn{2}{c}{Time-o1} &
    \multicolumn{2}{c}{FreDF} &
    \multicolumn{2}{c}{Koopman} &
    \multicolumn{2}{c}{Soft-DTW} &
    \multicolumn{2}{c}{DF} \\
    \cmidrule(lr){3-4} \cmidrule(lr){5-6} \cmidrule(lr){7-8} \cmidrule(lr){9-10}\cmidrule(lr){11-12}\cmidrule(lr){13-14}\cmidrule(lr){15-16}
    \multicolumn{2}{l}{Metrics}  & MSE & MAE  & MSE & MAE & MSE & MAE & MSE & MAE  & MSE & MAE  & MSE & MAE  & MSE & MAE  \\
       \hline
    \rowcolor{blue!8}
    \multicolumn{16}{l}{\textbf{Forecast model:TQNet}}\\\hline

\multirow{5}{*}{{\rotatebox{90}{\scalebox{0.95}{ETTm1}}}}
& 96 & \textbf{0.305} & \textbf{0.343} & 0.307 & 0.349 & 0.309 & 0.351 & 0.314 & 0.355 & 0.806 & 0.578 & 0.315 & 0.353 & 0.311 & 0.353 \\
& 192 & \textbf{0.352} & \textbf{0.373} & \textbf{0.352} & 0.376 & 0.353 & 0.375 & 0.359 & 0.378 & 0.619 & 0.515 & 0.360 & 0.377 & 0.357 & 0.378 \\
& 336 & 0.388 & 0.400 & 0.383 & 0.398 & 0.383 & 0.398 & \textbf{0.382} & \textbf{0.396} & 0.507 & 0.468 & 0.398 & 0.402 & 0.390 & 0.401 \\
& 720 & 0.444 & \textbf{0.430} & \textbf{0.441} & 0.434 & 0.444 & 0.436 & 0.444 & 0.432 & 0.450 & 0.437 & 0.476 & 0.446 & 0.449 & 0.439 \\
\cmidrule(lr){2-16}
& Avg & 0.373 & \textbf{0.386} & \textbf{0.371} & 0.389 & 0.372 & 0.390 & 0.375 & 0.390 & 0.595 & 0.499 & 0.387 & 0.394 & 0.377 & 0.393 \\
\midrule
\multirow{5}{*}{{\rotatebox{90}{\scalebox{0.95}{ETTh1}}}}
& 96 & 0.368 & \textbf{0.385} & \textbf{0.365} & 0.389 & 0.381 & 0.395 & 0.369 & 0.391 & 0.415 & 0.425 & 0.379 & 0.390 & 0.371 & 0.393 \\
& 192 & 0.426 & 0.422 & 0.427 & \textbf{0.421} & 0.427 & 0.424 & \textbf{0.425} & 0.422 & 0.430 & 0.422 & 0.437 & 0.424 & 0.428 & 0.426 \\
& 336 & 0.471 & \textbf{0.443} & \textbf{0.466} & 0.449 & 0.471 & 0.444 & 0.467 & 0.445 & 0.474 & 0.445 & 0.488 & 0.453 & 0.475 & 0.446 \\
& 720 & 0.486 & 0.469 & \textbf{0.466} & 0.467 & 0.469 & \textbf{0.466} & 0.468 & 0.469 & 0.483 & 0.474 & 0.510 & 0.487 & 0.488 & 0.470 \\
\cmidrule(lr){2-16}
& Avg & 0.438 & \textbf{0.430} & \textbf{0.431} & 0.431 & 0.437 & 0.432 & 0.432 & 0.432 & 0.451 & 0.442 & 0.453 & 0.438 & 0.441 & 0.434 \\
\midrule
\multirow{5}{*}{{\rotatebox{90}{\scalebox{0.95}{ECL}}}}
& 96 & \textbf{0.133} & \textbf{0.225} & 0.135 & 0.229 & 0.136 & 0.228 & 0.136 & 0.228 & 0.137 & 0.231 & 0.162 & 0.258 & 0.134 & 0.229 \\
& 192 & \textbf{0.151} & \textbf{0.241} & 0.153 & 0.245 & 0.154 & 0.245 & 0.155 & 0.245 & 0.154 & 0.247 & 0.446 & 0.449 & 0.154 & 0.247 \\
& 336 & \textbf{0.166} & \textbf{0.258} & 0.169 & 0.262 & 0.171 & 0.262 & 0.172 & 0.263 & 0.171 & 0.264 & 0.912 & 0.675 & 0.170 & 0.265 \\
& 720 & \textbf{0.196} & \textbf{0.287} & 0.202 & 0.290 & 0.208 & 0.293 & 0.209 & 0.293 & 0.204 & 0.292 & 0.971 & 0.715 & 0.200 & 0.293 \\
\cmidrule(lr){2-16}
& Avg & \textbf{0.162} & \textbf{0.253} & 0.165 & 0.257 & 0.167 & 0.257 & 0.168 & 0.257 & 0.166 & 0.258 & 0.623 & 0.524 & 0.165 & 0.259 \\
\midrule
\multirow{5}{*}{{\rotatebox{90}{\scalebox{0.95}{Weather}}}}
& 96 & \textbf{0.154} & \textbf{0.192} & 0.158 & 0.201 & 0.159 & 0.201 & 0.158 & 0.199 & 0.223 & 0.268 & 0.161 & 0.202 & 0.156 & 0.200 \\
& 192 & \textbf{0.203} & \textbf{0.238} & 0.207 & 0.245 & 0.209 & 0.246 & 0.209 & 0.246 & 0.269 & 0.304 & 0.212 & 0.247 & 0.205 & 0.245 \\
& 336 & \textbf{0.262} & \textbf{0.285} & 0.263 & 0.286 & 0.268 & 0.290 & 0.266 & 0.288 & 0.291 & 0.309 & 0.270 & 0.289 & \textbf{0.262} & 0.287 \\
& 720 & \textbf{0.342} & \textbf{0.339} & \textbf{0.342} & \textbf{0.339} & 0.344 & 0.341 & 0.344 & 0.341 & 0.346 & 0.343 & 0.378 & 0.365 & 0.343 & 0.342 \\
\cmidrule(lr){2-16}
& Avg & \textbf{0.241} & \textbf{0.264} & 0.242 & 0.268 & 0.245 & 0.269 & 0.244 & 0.268 & 0.282 & 0.306 & 0.255 & 0.276 & 0.242 & 0.268 \\

\hline
    \rowcolor{blue!8}
    \multicolumn{16}{l}{\textbf{Forecast model:PDF}}\\\hline
    
\multirow{5}{*}{{\rotatebox{90}{\scalebox{0.95}{ETTm1}}}}
& 96 & \textbf{0.310} & \textbf{0.342} & 0.320 & 0.358 & 0.326 & 0.361 & 0.325 & 0.362 & 1.051 & 0.663 & 0.323 & 0.362 & 0.326 & 0.363 \\
& 192 & \textbf{0.360} & \textbf{0.369} & 0.361 & 0.380 & 0.371 & 0.386 & 0.372 & 0.388 & 0.420 & 0.414 & 0.371 & 0.388 & 0.365 & 0.381 \\
& 336 & \textbf{0.390} & \textbf{0.392} & \textbf{0.390} & 0.401 & 0.401 & 0.409 & 0.399 & 0.409 & 0.421 & 0.415 & 0.408 & 0.413 & 0.397 & 0.402 \\
& 720 & \textbf{0.451} & \textbf{0.430} & \textbf{0.451} & 0.437 & 0.448 & 0.439 & 0.453 & 0.443 & 0.456 & 0.448 & 0.480 & 0.454 & 0.458 & 0.437 \\
\cmidrule(lr){2-16}
& Avg & \textbf{0.378} & \textbf{0.383} & 0.381 & 0.394 & 0.386 & 0.399 & 0.387 & 0.400 & 0.587 & 0.485 & 0.396 & 0.404 & 0.387 & 0.396 \\
\midrule
\multirow{5}{*}{{\rotatebox{90}{\scalebox{0.95}{ETTh1}}}}
& 96 & \textbf{0.370} & \textbf{0.390} & 0.375 & 0.391 & 0.380 & 0.403 & 0.373 & 0.393 & 0.632 & 0.533 & 0.383 & 0.405 & 0.388 & 0.400 \\
& 192 & \textbf{0.419} & \textbf{0.419} & 0.423 & \textbf{0.419} & 0.422 & 0.425 & 0.423 & 0.426 & 0.424 & 0.429 & 0.430 & 0.432 & 0.440 & 0.428 \\
& 336 & \textbf{0.460} & \textbf{0.439} & 0.461 & \textbf{0.439} & 0.463 & 0.441 & 0.477 & 0.446 & 0.456 & 0.450 & 0.462 & 0.453 & 0.483 & 0.449 \\
& 720 & 0.479 & 0.474 & 0.484 & \textbf{0.468} & 0.485 & 0.483 & \textbf{0.475} & 0.476 & 0.476 & 0.478 & 0.511 & 0.496 & 0.495 & 0.482 \\
\cmidrule(lr){2-16}
& Avg & \textbf{0.432} & \textbf{0.429} & 0.436 & \textbf{0.429} & 0.438 & 0.438 & 0.437 & 0.435 & 0.497 & 0.472 & 0.447 & 0.447 & 0.452 & 0.440 \\
\midrule
\multirow{5}{*}{{\rotatebox{90}{\scalebox{0.95}{ECL}}}}
& 96 & 0.165 & \textbf{0.246} & 0.171 & 0.257 & 0.173 & 0.253 & \textbf{0.163} & \textbf{0.246} & 0.194 & 0.278 & 0.164 & 0.250 & 0.175 & 0.259 \\
& 192 & 0.174 & \textbf{0.255} & 0.177 & 0.261 & 0.181 & 0.262 & 0.179 & 0.261 & \textbf{0.173} & 0.260 & 0.387 & 0.410 & 0.182 & 0.266 \\
& 336 & \textbf{0.187} & \textbf{0.269} & 0.192 & 0.277 & 0.196 & 0.282 & 0.196 & 0.278 & 0.189 & 0.276 & 0.966 & 0.698 & 0.197 & 0.282 \\
& 720 & 0.230 & \textbf{0.307} & 0.234 & 0.312 & 0.229 & \textbf{0.307} & 0.237 & 0.312 & \textbf{0.228} & 0.310 & 1.263 & 0.834 & 0.237 & 0.315 \\
\cmidrule(lr){2-16}
& Avg & \textbf{0.189} & \textbf{0.269} & 0.194 & 0.277 & 0.195 & 0.276 & 0.194 & 0.274 & 0.196 & 0.281 & 0.695 & 0.548 & 0.198 & 0.281 \\
\midrule
\multirow{5}{*}{{\rotatebox{90}{\scalebox{0.95}{Weather}}}}
& 96 & 0.174 & \textbf{0.211} & 0.176 & 0.218 & 0.178 & 0.219 & \textbf{0.173} & 0.216 & 0.202 & 0.242 & 0.178 & 0.219 & 0.181 & 0.221 \\
& 192 & \textbf{0.224} & \textbf{0.254} & 0.225 & 0.260 & 0.236 & 0.267 & 0.235 & 0.268 & 0.225 & 0.258 & 0.232 & 0.262 & 0.234 & 0.263 \\
& 336 & \textbf{0.279} & \textbf{0.294} & 0.280 & 0.299 & 0.284 & 0.304 & 0.274 & 0.295 & 0.280 & 0.302 & 0.281 & 0.296 & 0.285 & 0.300 \\
& 720 & 0.358 & \textbf{0.345} & 0.357 & 0.347 & 0.357 & 0.348 & 0.356 & 0.350 & \textbf{0.353} & 0.347 & 4.502 & 1.036 & 0.360 & 0.348 \\
\cmidrule(lr){2-16}
& Avg & \textbf{0.259} & \textbf{0.276} & \textbf{0.259} & 0.281 & 0.264 & 0.284 & 0.268 & 0.287 & 0.268 & 0.290 & 1.296 & 0.452 & 0.265 & 0.283 \\

    \bottomrule
  \end{tabular}
  \end{threeparttable}
\end{table*}

\begin{table}[ht]
\caption{Full results on the multi-step forecasting task. The length of history window is set to 96 for all baselines. \texttt{Avg} indicates the results averaged over forecasting lengths: T=96, 192, 336 and 720 for ETT, ECL, and Weather. Follow the settings of QDF~\citep{kmbdf}}\label{tab:multistep_app_full}
\renewcommand{\arraystretch}{0.8}
\setlength{\tabcolsep}{2pt}
\scriptsize
\centering
\renewcommand{\multirowsetup}{\centering}
\resizebox{\textwidth}{!}{%
\begin{tabular}{c|c|cc|cc|cc|cc|cc|cc|cc|cc|cc|cc|cc|cc}
    \toprule
    \multicolumn{2}{l}{\multirow{2}{*}{\rotatebox{0}{\scaleb{Loss}}}} & 
    \multicolumn{2}{c}{\rotatebox{0}{\scaleb{\textbf{\time}}}} &
    \multicolumn{2}{c}{\rotatebox{0}{\scaleb{KMB-DF}}} &
    \multicolumn{2}{c}{\rotatebox{0}{\scaleb{DBLoss}}} &
    \multicolumn{2}{c}{\rotatebox{0}{\scaleb{QDF}}} &
    \multicolumn{2}{c}{\rotatebox{0}{\scaleb{DistDF}}} &
    \multicolumn{2}{c}{\rotatebox{0}{\scaleb{Time-o1}}} &
    \multicolumn{2}{c}{\rotatebox{0}{\scaleb{FreDF}}} &
    \multicolumn{2}{c}{\rotatebox{0}{\scaleb{Koopman}}} &
    \multicolumn{2}{c}{\rotatebox{0}{\scaleb{GDTW}}} &
    \multicolumn{2}{c}{\rotatebox{0}{\scaleb{Dilate}}} &
    \multicolumn{2}{c}{\rotatebox{0}{\scaleb{Soft-DTW}}} &
    \multicolumn{2}{c}{\rotatebox{0}{\scaleb{MSE}}} \\
    \multicolumn{2}{c}{} &
    \multicolumn{2}{c}{\scaleb{\textbf{(Ours)}}} & 
    \multicolumn{2}{c}{\scaleb{(2026)}} & 
    \multicolumn{2}{c}{\scaleb{(2025)}} & 
    \multicolumn{2}{c}{\scaleb{(2025)}} & 
    \multicolumn{2}{c}{\scaleb{(2025)}} & 
    \multicolumn{2}{c}{\scaleb{(2025)}} & 
    \multicolumn{2}{c}{\scaleb{(2025)}} & 
    \multicolumn{2}{c}{\scaleb{(2021)}} & 
    \multicolumn{2}{c}{\scaleb{(2021)}} &
    \multicolumn{2}{c}{\scaleb{(2019)}} &
    \multicolumn{2}{c}{\scaleb{(2017)}} &
    \multicolumn{2}{c}{\scaleb{(2002)}} \\
    \cmidrule(lr){3-4} \cmidrule(lr){5-6} \cmidrule(lr){7-8}\cmidrule(lr){9-10} \cmidrule(lr){11-12}\cmidrule(lr){13-14} \cmidrule(lr){15-16} \cmidrule(lr){17-18} \cmidrule(lr){19-20} \cmidrule(lr){21-22} \cmidrule(lr){23-24} \cmidrule(lr){25-26}
    \multicolumn{2}{l}{\rotatebox{0}{\scaleb{Metrics}}}  & \scalea{MSE} & \scalea{MAE}  & \scalea{MSE} & \scalea{MAE}  & \scalea{MSE} & \scalea{MAE}  & \scalea{MSE} & \scalea{MAE}  & \scalea{MSE} & \scalea{MAE}  & \scalea{MSE} & \scalea{MAE}  & \scalea{MSE} & \scalea{MAE}  & \scalea{MSE} & \scalea{MAE} & \scalea{MSE} & \scalea{MAE} & \scalea{MSE} & \scalea{MAE} & \scalea{MSE} & \scalea{MAE} & \scalea{MSE} & \scalea{MAE} \\
    \toprule
    
\multirow{5}{*}{{\rotatebox{90}{\scalebox{0.95}{ETTm1}}}}
& 96 & \scalea{\bst{0.309}} & \scalea{\bst{0.342}} & \scalea{0.315} & \scalea{0.352} & \scalea{\subbst{0.310}} & \scalea{\subbst{0.345}} & \scalea{0.318} & \scalea{0.359} & \scalea{0.318} & \scalea{0.355} & \scalea{0.321} & \scalea{0.353} & \scalea{0.320} & \scalea{0.354} & \scalea{0.326} & \scalea{0.363} & \scalea{0.324} & \scalea{0.362} & \scalea{0.323} & \scalea{0.361} & \scalea{0.322} & \scalea{0.359} & \scalea{0.321} & \scalea{0.358}  \\
& 192 & \scalea{\bst{0.350}} & \scalea{\subbst{0.369}} & \scalea{\subbst{0.351}} & \scalea{0.379} & \scalea{0.352} & \scalea{\bst{0.368}} & \scalea{0.355} & \scalea{0.379} & \scalea{0.356} & \scalea{0.380} & \scalea{0.358} & \scalea{0.380} & \scalea{0.356} & \scalea{0.379} & \scalea{0.359} & \scalea{0.383} & \scalea{0.368} & \scalea{0.392} & \scalea{0.363} & \scalea{0.385} & \scalea{0.366} & \scalea{0.385} & \scalea{0.357} & \scalea{0.382}  \\
& 336 & \scalea{0.385} & \scalea{\bst{0.392}} & \scalea{\bst{0.381}} & \scalea{\subbst{0.400}} & \scalea{0.391} & \scalea{\bst{0.392}} & \scalea{0.384} & \scalea{0.401} & \scalea{\subbst{0.383}} & \scalea{0.401} & \scalea{0.386} & \scalea{0.403} & \scalea{0.386} & \scalea{0.401} & \scalea{\subbst{0.383}} & \scalea{0.401} & \scalea{0.405} & \scalea{0.421} & \scalea{0.384} & \scalea{0.401} & \scalea{0.399} & \scalea{0.409} & \scalea{0.387} & \scalea{0.401}  \\
& 720 & \scalea{\bst{0.442}} & \scalea{\subbst{0.428}} & \scalea{\bst{0.442}} & \scalea{0.433} & \scalea{0.447} & \scalea{\bst{0.427}} & \scalea{\subbst{0.443}} & \scalea{0.435} & \scalea{0.444} & \scalea{0.434} & \scalea{0.447} & \scalea{0.436} & \scalea{0.444} & \scalea{0.434} & \scalea{0.446} & \scalea{0.436} & \scalea{0.496} & \scalea{0.473} & \scalea{0.446} & \scalea{0.437} & \scalea{0.490} & \scalea{0.461} & \scalea{0.446} & \scalea{0.435}  \\
\cmidrule(lr){2-26}
& Avg & \scalea{\bst{0.372}} & \scalea{\bst{0.383}} & \scalea{\bst{0.372}} & \scalea{\subbst{0.391}} & \scalea{\subbst{0.375}} & \scalea{\bst{0.383}} & \scalea{\subbst{0.375}} & \scalea{0.393} & \scalea{\subbst{0.375}} & \scalea{0.393} & \scalea{0.378} & \scalea{0.393} & \scalea{0.376} & \scalea{0.392} & \scalea{0.378} & \scalea{0.395} & \scalea{0.398} & \scalea{0.412} & \scalea{0.379} & \scalea{0.396} & \scalea{0.394} & \scalea{0.403} & \scalea{0.378} & \scalea{0.394}  \\
\midrule

\multirow{5}{*}{{\rotatebox{90}{\scalebox{0.95}{ETTm2}}}}
& 96 & \scalea{\bst{0.165}} & \scalea{\subbst{0.245}} & \scalea{\subbst{0.166}} & \scalea{0.249} & \scalea{\subbst{0.166}} & \scalea{\bst{0.244}} & \scalea{0.168} & \scalea{0.250} & \scalea{0.168} & \scalea{0.251} & \scalea{0.173} & \scalea{0.251} & \scalea{0.174} & \scalea{0.251} & \scalea{0.182} & \scalea{0.262} & \scalea{0.176} & \scalea{0.257} & \scalea{0.172} & \scalea{0.253} & \scalea{0.176} & \scalea{0.256} & \scalea{0.175} & \scalea{0.256}  \\
& 192 & \scalea{\subbst{0.230}} & \scalea{\bst{0.288}} & \scalea{\bst{0.229}} & \scalea{0.291} & \scalea{0.231} & \scalea{\bst{0.288}} & \scalea{0.232} & \scalea{0.292} & \scalea{0.232} & \scalea{0.293} & \scalea{0.233} & \scalea{0.291} & \scalea{0.233} & \scalea{\subbst{0.290}} & \scalea{0.235} & \scalea{0.294} & \scalea{0.243} & \scalea{0.305} & \scalea{0.236} & \scalea{0.296} & \scalea{0.243} & \scalea{0.299} & \scalea{0.234} & \scalea{0.295}  \\
& 336 & \scalea{\bst{0.287}} & \scalea{\bst{0.326}} & \scalea{0.289} & \scalea{0.330} & \scalea{0.291} & \scalea{\subbst{0.327}} & \scalea{\subbst{0.288}} & \scalea{0.329} & \scalea{0.289} & \scalea{0.330} & \scalea{0.290} & \scalea{0.329} & \scalea{0.290} & \scalea{\subbst{0.327}} & \scalea{0.291} & \scalea{0.331} & \scalea{0.303} & \scalea{0.344} & \scalea{0.290} & \scalea{0.331} & \scalea{0.306} & \scalea{0.342} & \scalea{0.289} & \scalea{0.330}  \\
& 720 & \scalea{\bst{0.383}} & \scalea{\subbst{0.385}} & \scalea{\subbst{0.384}} & \scalea{0.388} & \scalea{0.386} & \scalea{\bst{0.384}} & \scalea{0.386} & \scalea{0.387} & \scalea{0.385} & \scalea{0.387} & \scalea{0.386} & \scalea{\subbst{0.385}} & \scalea{0.386} & \scalea{0.386} & \scalea{0.388} & \scalea{0.390} & \scalea{0.434} & \scalea{0.419} & \scalea{0.393} & \scalea{0.390} & \scalea{0.438} & \scalea{0.422} & \scalea{0.385} & \scalea{0.389}  \\
\cmidrule(lr){2-26}
& Avg & \scalea{\bst{0.266}} & \scalea{\subbst{0.311}} & \scalea{\subbst{0.267}} & \scalea{0.315} & \scalea{0.268} & \scalea{\bst{0.310}} & \scalea{0.268} & \scalea{0.315} & \scalea{0.269} & \scalea{0.315} & \scalea{0.271} & \scalea{0.314} & \scalea{0.271} & \scalea{0.313} & \scalea{0.274} & \scalea{0.319} & \scalea{0.289} & \scalea{0.332} & \scalea{0.273} & \scalea{0.317} & \scalea{0.291} & \scalea{0.330} & \scalea{0.271} & \scalea{0.317}  \\
\midrule

\multirow{5}{*}{{\rotatebox{90}{\scalebox{0.95}{ETTh1}}}}
& 96 & \scalea{0.371} & \scalea{\bst{0.387}} & \scalea{0.372} & \scalea{0.389} & \scalea{\subbst{0.370}} & \scalea{\subbst{0.388}} & \scalea{0.373} & \scalea{0.392} & \scalea{0.373} & \scalea{0.391} & \scalea{\bst{0.369}} & \scalea{0.392} & \scalea{0.373} & \scalea{0.389} & \scalea{0.374} & \scalea{0.393} & \scalea{0.377} & \scalea{0.395} & \scalea{0.376} & \scalea{0.394} & \scalea{0.377} & \scalea{0.394} & \scalea{0.373} & \scalea{0.391}  \\
& 192 & \scalea{\subbst{0.424}} & \scalea{\bst{0.418}} & \scalea{\bst{0.423}} & \scalea{\subbst{0.420}} & \scalea{\subbst{0.424}} & \scalea{\bst{0.418}} & \scalea{0.426} & \scalea{0.421} & \scalea{0.426} & \scalea{\subbst{0.420}} & \scalea{\subbst{0.424}} & \scalea{0.424} & \scalea{\subbst{0.424}} & \scalea{\bst{0.418}} & \scalea{0.426} & \scalea{0.424} & \scalea{0.434} & \scalea{0.427} & \scalea{0.430} & \scalea{0.423} & \scalea{0.436} & \scalea{0.427} & \scalea{0.427} & \scalea{0.421}  \\
& 336 & \scalea{0.467} & \scalea{0.441} & \scalea{\bst{0.460}} & \scalea{\bst{0.438}} & \scalea{0.466} & \scalea{\subbst{0.439}} & \scalea{\subbst{0.464}} & \scalea{0.441} & \scalea{0.468} & \scalea{0.441} & \scalea{\subbst{0.464}} & \scalea{\subbst{0.439}} & \scalea{0.465} & \scalea{\subbst{0.439}} & \scalea{0.466} & \scalea{0.440} & \scalea{0.486} & \scalea{0.457} & \scalea{0.470} & \scalea{0.442} & \scalea{0.479} & \scalea{0.451} & \scalea{0.466} & \scalea{0.441}  \\
& 720 & \scalea{0.466} & \scalea{0.461} & \scalea{\bst{0.447}} & \scalea{\bst{0.455}} & \scalea{0.468} & \scalea{\subbst{0.458}} & \scalea{0.473} & \scalea{0.464} & \scalea{0.467} & \scalea{0.462} & \scalea{\subbst{0.462}} & \scalea{0.460} & \scalea{0.472} & \scalea{0.468} & \scalea{0.483} & \scalea{0.469} & \scalea{0.512} & \scalea{0.497} & \scalea{0.489} & \scalea{0.478} & \scalea{0.551} & \scalea{0.510} & \scalea{0.477} & \scalea{0.468}  \\
\cmidrule(lr){2-26}
& Avg & \scalea{0.432} & \scalea{\subbst{0.427}} & \scalea{\bst{0.426}} & \scalea{\bst{0.426}} & \scalea{0.432} & \scalea{\bst{0.426}} & \scalea{0.434} & \scalea{0.429} & \scalea{0.434} & \scalea{0.428} & \scalea{\subbst{0.430}} & \scalea{0.429} & \scalea{0.434} & \scalea{0.428} & \scalea{0.437} & \scalea{0.431} & \scalea{0.452} & \scalea{0.444} & \scalea{0.441} & \scalea{0.434} & \scalea{0.461} & \scalea{0.445} & \scalea{0.436} & \scalea{0.430}  \\
\midrule

\multirow{5}{*}{{\rotatebox{90}{\scalebox{0.95}{ETTh2}}}}
& 96 & \scalea{\bst{0.283}} & \scalea{\bst{0.333}} & \scalea{\subbst{0.285}} & \scalea{0.337} & \scalea{0.288} & \scalea{\bst{0.333}} & \scalea{0.288} & \scalea{0.337} & \scalea{0.287} & \scalea{0.337} & \scalea{0.286} & \scalea{\subbst{0.336}} & \scalea{0.289} & \scalea{0.337} & \scalea{0.287} & \scalea{0.338} & \scalea{0.289} & \scalea{0.340} & \scalea{0.287} & \scalea{0.338} & \scalea{0.292} & \scalea{0.342} & \scalea{0.287} & \scalea{0.337}  \\
& 192 & \scalea{\bst{0.361}} & \scalea{\bst{0.381}} & \scalea{\subbst{0.363}} & \scalea{0.387} & \scalea{0.365} & \scalea{\bst{0.381}} & \scalea{0.367} & \scalea{0.390} & \scalea{0.364} & \scalea{0.391} & \scalea{0.367} & \scalea{0.384} & \scalea{0.365} & \scalea{\subbst{0.382}} & \scalea{0.365} & \scalea{0.390} & \scalea{0.377} & \scalea{0.399} & \scalea{0.368} & \scalea{0.389} & \scalea{0.380} & \scalea{0.398} & \scalea{0.368} & \scalea{0.390}  \\
& 336 & \scalea{\bst{0.409}} & \scalea{\bst{0.420}} & \scalea{\subbst{0.411}} & \scalea{\subbst{0.425}} & \scalea{0.419} & \scalea{0.426} & \scalea{0.412} & \scalea{0.427} & \scalea{0.412} & \scalea{0.426} & \scalea{0.414} & \scalea{0.427} & \scalea{0.414} & \scalea{0.427} & \scalea{0.413} & \scalea{0.429} & \scalea{0.422} & \scalea{0.432} & \scalea{0.413} & \scalea{0.427} & \scalea{0.455} & \scalea{0.443} & \scalea{0.416} & \scalea{0.428}  \\
& 720 & \scalea{\bst{0.392}} & \scalea{\bst{0.418}} & \scalea{\subbst{0.396}} & \scalea{0.426} & \scalea{\bst{0.392}} & \scalea{\subbst{0.419}} & \scalea{0.401} & \scalea{0.428} & \scalea{0.398} & \scalea{0.427} & \scalea{0.402} & \scalea{0.428} & \scalea{0.406} & \scalea{0.430} & \scalea{0.405} & \scalea{0.430} & \scalea{0.457} & \scalea{0.463} & \scalea{0.415} & \scalea{0.437} & \scalea{0.445} & \scalea{0.453} & \scalea{0.415} & \scalea{0.436}  \\
\cmidrule(lr){2-26}
& Avg & \scalea{\bst{0.361}} & \scalea{\bst{0.388}} & \scalea{\subbst{0.364}} & \scalea{0.394} & \scalea{0.366} & \scalea{\subbst{0.390}} & \scalea{0.367} & \scalea{0.396} & \scalea{0.365} & \scalea{0.395} & \scalea{0.367} & \scalea{0.394} & \scalea{0.368} & \scalea{0.394} & \scalea{0.368} & \scalea{0.397} & \scalea{0.386} & \scalea{0.409} & \scalea{0.371} & \scalea{0.398} & \scalea{0.393} & \scalea{0.409} & \scalea{0.372} & \scalea{0.398}  \\
\midrule

\multirow{5}{*}{{\rotatebox{90}{\scalebox{0.95}{ECL}}}}
& 96 & \scalea{\bst{0.135}} & \scalea{\bst{0.227}} & \scalea{\subbst{0.136}} & \scalea{0.231} & \scalea{\bst{0.135}} & \scalea{\subbst{0.228}} & \scalea{0.137} & \scalea{0.232} & \scalea{0.137} & \scalea{0.232} & \scalea{0.141} & \scalea{0.235} & \scalea{0.139} & \scalea{0.232} & \scalea{0.138} & \scalea{0.232} & \scalea{0.882} & \scalea{0.777} & \scalea{0.137} & \scalea{0.232} & \scalea{\subbst{0.136}} & \scalea{0.232} & \scalea{0.137} & \scalea{0.232}  \\
& 192 & \scalea{\bst{0.152}} & \scalea{\bst{0.242}} & \scalea{\bst{0.152}} & \scalea{0.245} & \scalea{\bst{0.152}} & \scalea{\subbst{0.243}} & \scalea{\subbst{0.153}} & \scalea{0.246} & \scalea{\bst{0.152}} & \scalea{0.246} & \scalea{0.154} & \scalea{0.246} & \scalea{0.156} & \scalea{0.246} & \scalea{\subbst{0.153}} & \scalea{0.247} & \scalea{0.853} & \scalea{0.762} & \scalea{0.154} & \scalea{0.247} & \scalea{0.154} & \scalea{0.247} & \scalea{\subbst{0.153}} & \scalea{0.247}  \\
& 336 & \scalea{\subbst{0.166}} & \scalea{\bst{0.259}} & \scalea{0.168} & \scalea{0.265} & \scalea{\bst{0.165}} & \scalea{\subbst{0.260}} & \scalea{0.168} & \scalea{0.264} & \scalea{0.168} & \scalea{0.265} & \scalea{0.168} & \scalea{0.263} & \scalea{0.169} & \scalea{0.263} & \scalea{0.169} & \scalea{0.267} & \scalea{0.809} & \scalea{0.711} & \scalea{0.169} & \scalea{0.267} & \scalea{0.169} & \scalea{0.267} & \scalea{0.168} & \scalea{0.266}  \\
& 720 & \scalea{\subbst{0.195}} & \scalea{\bst{0.285}} & \scalea{0.198} & \scalea{0.293} & \scalea{\bst{0.193}} & \scalea{\bst{0.285}} & \scalea{0.199} & \scalea{0.293} & \scalea{0.201} & \scalea{0.294} & \scalea{0.198} & \scalea{0.291} & \scalea{0.199} & \scalea{\subbst{0.290}} & \scalea{0.201} & \scalea{0.296} & \scalea{0.900} & \scalea{0.769} & \scalea{0.202} & \scalea{0.297} & \scalea{0.203} & \scalea{0.298} & \scalea{0.199} & \scalea{0.294}  \\
\cmidrule(lr){2-26}
& Avg & \scalea{\bst{0.162}} & \scalea{\bst{0.253}} & \scalea{\subbst{0.163}} & \scalea{0.258} & \scalea{\bst{0.162}} & \scalea{\subbst{0.254}} & \scalea{0.164} & \scalea{0.259} & \scalea{0.164} & \scalea{0.259} & \scalea{0.165} & \scalea{0.259} & \scalea{0.166} & \scalea{0.258} & \scalea{0.165} & \scalea{0.260} & \scalea{0.861} & \scalea{0.755} & \scalea{0.165} & \scalea{0.260} & \scalea{0.165} & \scalea{0.261} & \scalea{0.165} & \scalea{0.260}  \\
\midrule

\multirow{5}{*}{{\rotatebox{90}{\scalebox{0.95}{Weather}}}}
& 96 & \scalea{\bst{0.152}} & \scalea{\bst{0.192}} & \scalea{\bst{0.152}} & \scalea{0.196} & \scalea{0.154} & \scalea{\subbst{0.193}} & \scalea{0.155} & \scalea{0.200} & \scalea{0.155} & \scalea{0.200} & \scalea{\subbst{0.153}} & \scalea{0.197} & \scalea{0.154} & \scalea{0.197} & \scalea{0.173} & \scalea{0.215} & \scalea{0.158} & \scalea{0.206} & \scalea{0.173} & \scalea{0.215} & \scalea{0.177} & \scalea{0.220} & \scalea{0.156} & \scalea{0.201}  \\
& 192 & \scalea{\bst{0.203}} & \scalea{\bst{0.238}} & \scalea{\subbst{0.204}} & \scalea{0.242} & \scalea{0.205} & \scalea{\subbst{0.239}} & \scalea{0.205} & \scalea{0.244} & \scalea{0.205} & \scalea{0.245} & \scalea{0.205} & \scalea{0.243} & \scalea{\subbst{0.204}} & \scalea{0.242} & \scalea{0.205} & \scalea{0.245} & \scalea{0.209} & \scalea{0.250} & \scalea{0.205} & \scalea{0.246} & \scalea{0.209} & \scalea{0.245} & \scalea{0.205} & \scalea{0.245}  \\
& 336 & \scalea{\bst{0.260}} & \scalea{\bst{0.280}} & \scalea{\bst{0.260}} & \scalea{\subbst{0.284}} & \scalea{\subbst{0.261}} & \scalea{\bst{0.280}} & \scalea{\subbst{0.261}} & \scalea{0.285} & \scalea{\subbst{0.261}} & \scalea{0.286} & \scalea{\subbst{0.261}} & \scalea{\subbst{0.284}} & \scalea{\subbst{0.261}} & \scalea{\subbst{0.284}} & \scalea{0.262} & \scalea{0.286} & \scalea{0.270} & \scalea{0.296} & \scalea{0.262} & \scalea{0.287} & \scalea{0.278} & \scalea{0.296} & \scalea{\subbst{0.261}} & \scalea{0.286}  \\
& 720 & \scalea{\subbst{0.340}} & \scalea{\subbst{0.336}} & \scalea{\bst{0.339}} & \scalea{0.337} & \scalea{\bst{0.339}} & \scalea{\bst{0.333}} & \scalea{0.343} & \scalea{0.338} & \scalea{0.343} & \scalea{0.341} & \scalea{0.343} & \scalea{0.340} & \scalea{0.341} & \scalea{0.338} & \scalea{0.345} & \scalea{0.342} & \scalea{0.354} & \scalea{0.352} & \scalea{0.345} & \scalea{0.342} & \scalea{0.385} & \scalea{0.363} & \scalea{0.343} & \scalea{0.339}  \\
\cmidrule(lr){2-26}
& Avg & \scalea{\bst{0.239}} & \scalea{\subbst{0.262}} & \scalea{\bst{0.239}} & \scalea{0.265} & \scalea{\subbst{0.240}} & \scalea{\bst{0.261}} & \scalea{0.241} & \scalea{0.267} & \scalea{0.241} & \scalea{0.268} & \scalea{\subbst{0.240}} & \scalea{0.266} & \scalea{\subbst{0.240}} & \scalea{0.265} & \scalea{0.246} & \scalea{0.272} & \scalea{0.248} & \scalea{0.276} & \scalea{0.246} & \scalea{0.273} & \scalea{0.262} & \scalea{0.281} & \scalea{0.241} & \scalea{0.267}  \\
\midrule

\multicolumn{2}{c|}{\scalea{{$1^{\text{st}}$ Count}}} & \scalea{\bst{20}} & \scalea{\bst{20}} & \scalea{\subbst{13}} & \scalea{3} & \scalea{7} & \scalea{\subbst{16}} & \scalea{0} & \scalea{0} & \scalea{1} & \scalea{0} & \scalea{1} & \scalea{0} & \scalea{0} & \scalea{1} & \scalea{0} & \scalea{0} & \scalea{0} & \scalea{0} & \scalea{0} & \scalea{0} & \scalea{0} & \scalea{0} & \scalea{0} & \scalea{0} \\
    \bottomrule
\end{tabular}%
}
\end{table}

\subsection{Choice of the edge-wise norm: \texorpdfstring{$\ell_1$}{L1} versus \texorpdfstring{$\ell_2$}{L2}}
\label{app:norm_choice}

Proposition~\ref{prop:precision} shows that MSE plus a \emph{squared} graph total-variation penalty is the negative log-likelihood of a Gaussian with a structured precision matrix, whereas the objective we deploy in Eq.~\eqref{eq:cvl} uses the $\ell_1$ norm on the same edge differences. The theory therefore backs the $\ell_2$ form, and the $\ell_1$ form is an empirical preference. That asymmetry could hide a substantive weakness. If $\ell_2$ were as good, the theory-backed variant would be available at no cost; if $\ell_2$ were far worse, the mechanism would not be what the theory describes. We therefore test the two norms directly under an identical protocol, varying nothing but the norm.

\begin{table}[htbp]
  \caption{Effect of the edge-wise norm. All settings are identical apart from the penalty: \emph{Base} is the unmodified backbone, $+\ell_2$ replaces $\|\cdot\|_1$ in Eq.~\eqref{eq:cvl} by the squared $\ell_2$ norm of Proposition~\ref{prop:precision}, and $+\ell_1$ is \time\ as deployed. Values are averaged over $T\in\{96,192,336,720\}$; the best entry in each row and metric is in \best{bold}.}
  \label{tab:norm_choice}
  \centering
  \setlength{\tabcolsep}{6pt}
  \small
  \begin{threeparttable}
  \begin{tabular}{llccc|ccc}
    \toprule
    \multirow{2}{*}{Dataset} & \multirow{2}{*}{Backbone} & \multicolumn{3}{c|}{MSE} & \multicolumn{3}{c}{MAE} \\
    \cmidrule(lr){3-5} \cmidrule(lr){6-8}
    & & Base & $+\ell_2$ & $+\ell_1$ & Base & $+\ell_2$ & $+\ell_1$ \\
    \midrule
    ECL     & TimeFilter & 0.158 & 0.157 & \best{0.156} & 0.256 & 0.254 & \best{0.252} \\
    ECL     & TQNet      & 0.165 & 0.164 & \best{0.162} & 0.259 & 0.257 & \best{0.254} \\
    Weather & TimeFilter & 0.241 & 0.241 & \best{0.238} & 0.271 & 0.270 & \best{0.262} \\
    Weather & TQNet      & 0.242 & 0.244 & \best{0.241} & 0.269 & 0.269 & \best{0.264} \\
    \bottomrule
  \end{tabular}
  \end{threeparttable}
\end{table}

Table~\ref{tab:norm_choice} reports the outcome. Relative to the unregularised base, the $\ell_2$ variant improves five of the eight averaged cells, matches the base in two, and slightly degrades Weather/TQNet MSE. Relation-aware coupling is therefore beneficial in its own right, which is the prediction Proposition~\ref{prop:precision} actually makes, and the benefit is not an artefact of the particular norm. The $\ell_1$ variant is nonetheless better than $\ell_2$ in all eight cells. We read this as consistent with the motivation given in Section~\ref{cvl}: a squared edge penalty is dominated by the few pairs with the largest structural mismatch, which in these series are typically localized anomalies rather than the systematic co-evolution the term is meant to enforce. The practical conclusion is that the theory identifies the right family and the norm is chosen within it on evidence, not that $\ell_1$ is derived.

\subsection{Generalization studies}\label{sec:generalize_app}
Additional experimental results of varying forecast models are available in Table~\ref{tab2}, Table~\ref{tab:loss-compare} and Fig ~\ref{fig:backbone_app}.

\subsection{Hyperparameter Sensitivity}
We evaluate the stability of \texttt{CvLoss} regarding the loss weight $\alpha$ and patch length $L$. Loss Weight $\alpha$: Tables \ref{tab:sensi-iTransformer} and \ref{tab:sensi-TimeBridge} show accuracy consistently improving as $\alpha$ increases. Optimal performance typically occurs at $\alpha \approx 0.5$ or $1$, where structural regularization complements point-wise accuracy. Patch Length $L$: Figure \ref{fig:sensi} demonstrates that forecasting performance remains stable across various $T/L$ ratios. The horizontal curves confirm that \texttt{CvLoss} is highly robust to patch size, requiring minimal tuning to achieve significant predictive gains.

\begin{figure}[tbp]
\begin{center}
\includegraphics[width=0.245\linewidth]{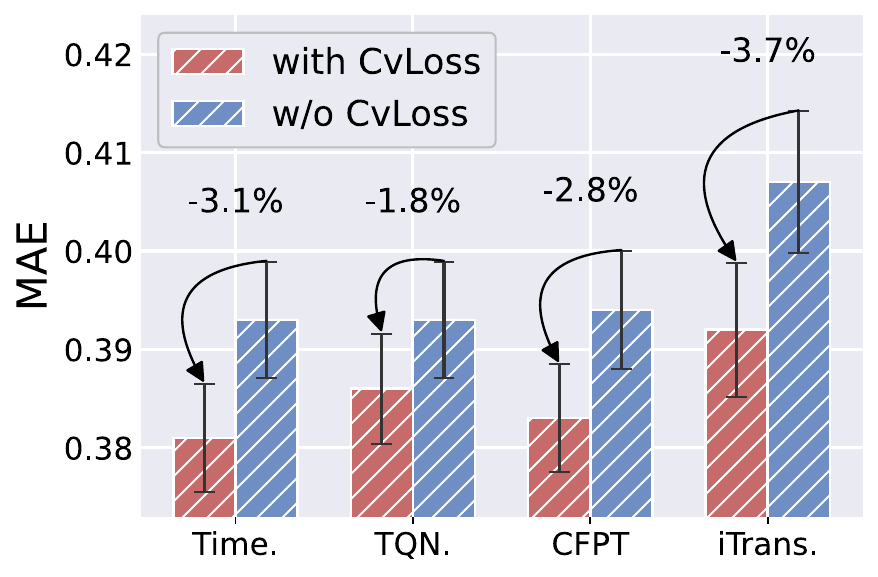}
\includegraphics[width=0.245\linewidth]{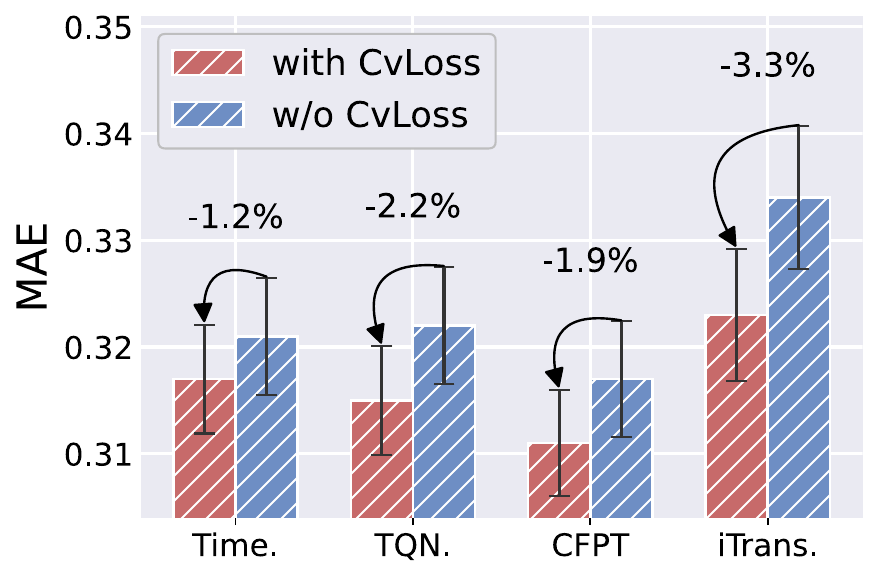}
\includegraphics[width=0.245\linewidth]{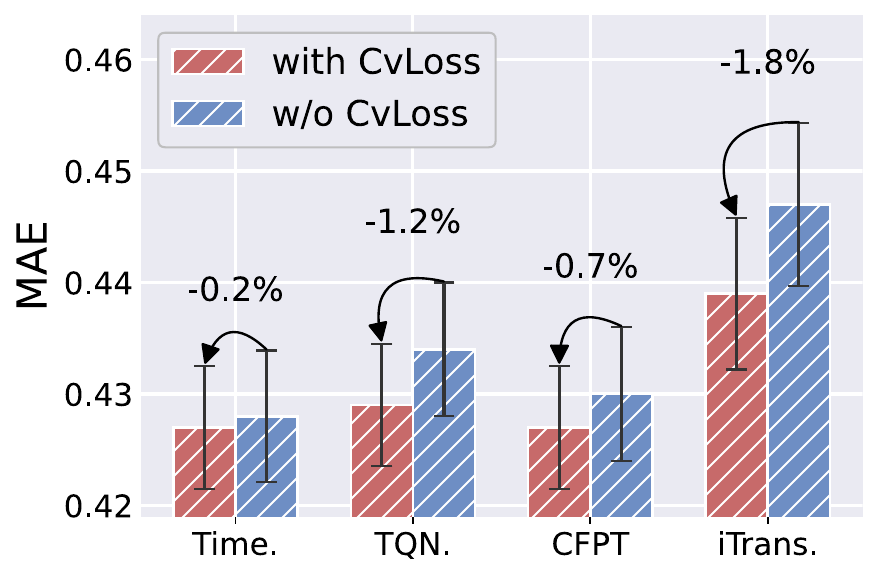}
\includegraphics[width=0.245\linewidth]{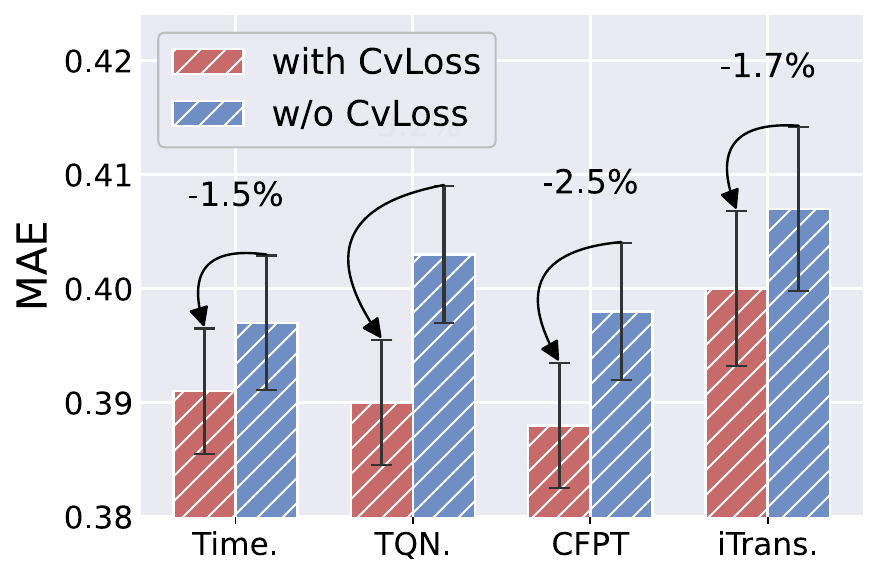}
\subfigure[ETTm1]{\includegraphics[width=0.245\linewidth]{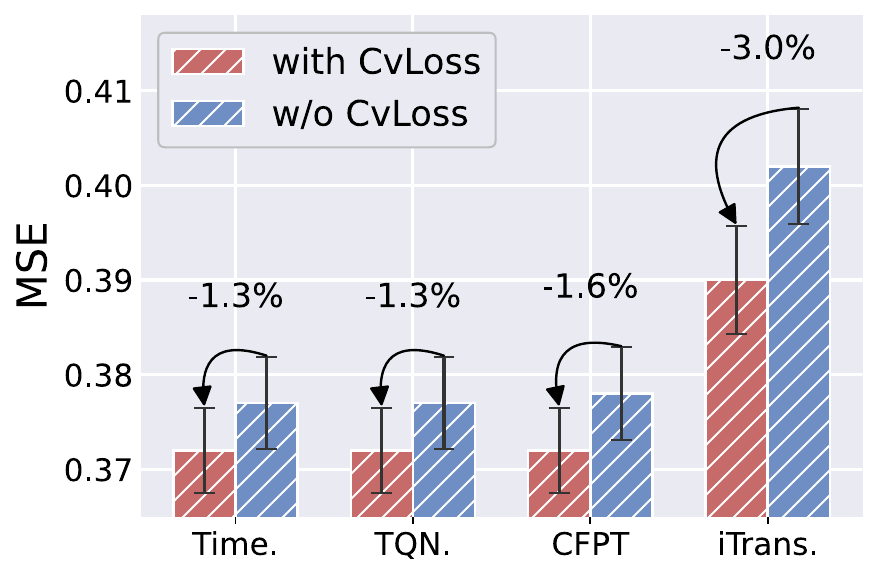}}
\subfigure[ETTm2]{\includegraphics[width=0.245\linewidth]{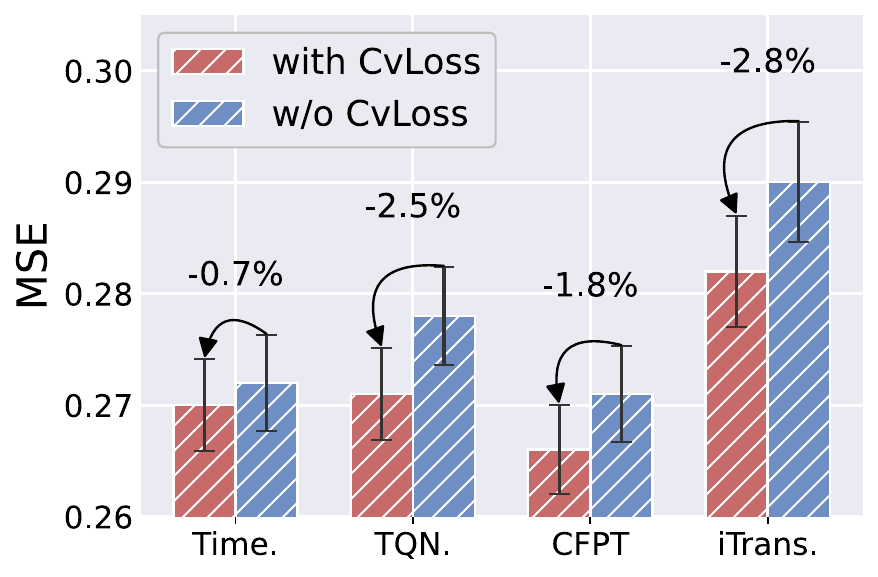}}
\subfigure[ETTh1]{\includegraphics[width=0.245\linewidth]{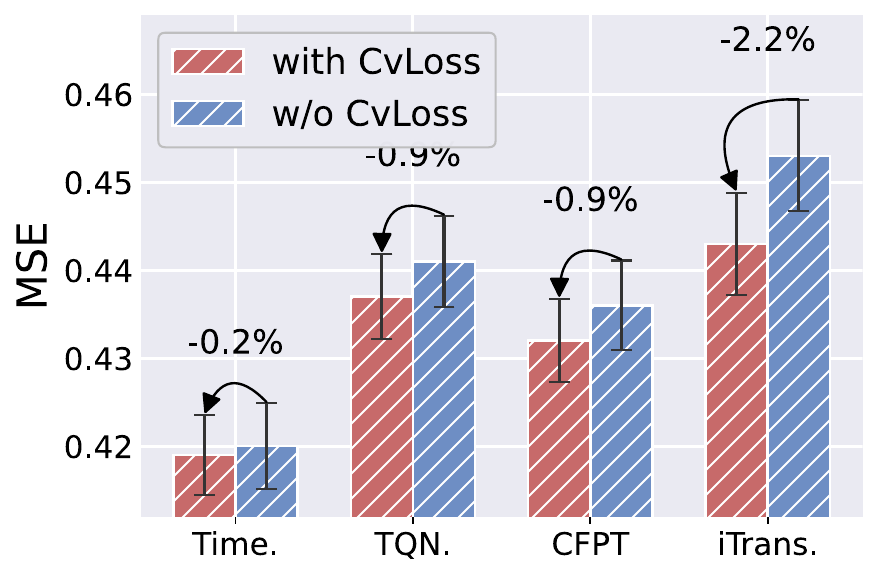}}
\subfigure[ETTh2]{\includegraphics[width=0.245\linewidth]{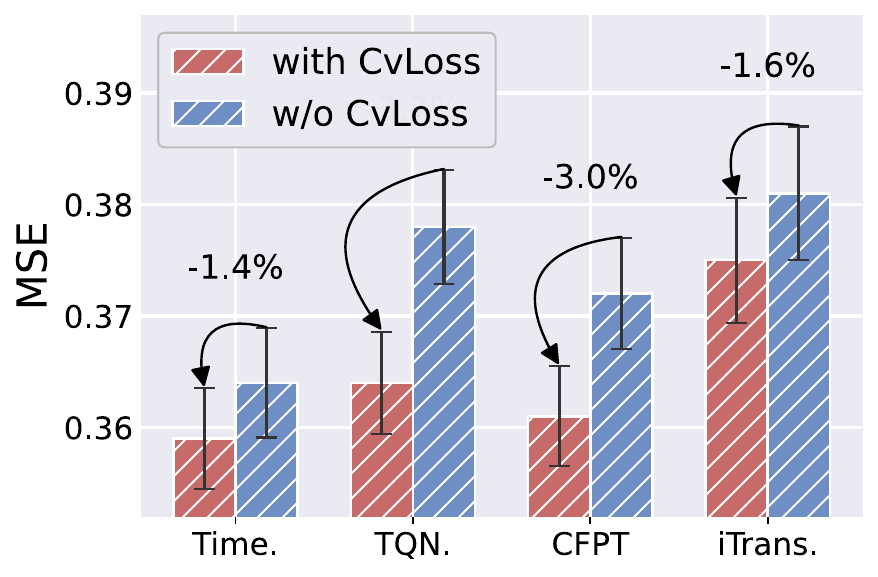}}
\caption{Performance of different forecast models with and without \time\. The forecast errors are averaged over forecast lengths and the error bars represent 50\% confidence intervals.}
\label{fig:backbone_app}
\end{center}
\end{figure}

\subsection{Case study with TQNet and TimeFilter of varying historical lengths}
\begin{table}
\centering
\caption{Varying input sequence length results on the Weather dataset.}\label{tab:vary_seq_len}
\renewcommand{\arraystretch}{1} \setlength{\tabcolsep}{10pt} \scriptsize
\centering
\renewcommand{\multirowsetup}{\centering}
\begin{tabular}{c|c|c|cc|cc|cc|cc}
    \toprule
    \multicolumn{3}{c|}{\rotatebox{0}{Models}} & \multicolumn{2}{c}{CvLoss} & \multicolumn{2}{c|}{TimeFilter} & \multicolumn{2}{c}{CvLoss} & \multicolumn{2}{c}{TQNet} \\
    \cmidrule(lr){4-5} \cmidrule(lr){6-7} \cmidrule(lr){8-9} \cmidrule(lr){10-11}
    \multicolumn{3}{c|}{\rotatebox{0}{Metrics}} & MSE & MAE & MSE & MAE & MSE & MAE & MSE & MAE \\
    \midrule
    \multirow{20}{*}{\rotatebox{90}{Historical sequence length}} 

& \multirow{5}{*}{96}
& 96 & 0.149 & 0.188 & 0.153 & 0.199 & 0.154 & 0.191 & 0.156 & 0.200 \\
&& 192 & 0.198 & 0.236 & 0.202 & 0.246 & 0.205 & 0.240 & 0.205 & 0.245 \\
&& 336 & 0.257 & 0.283 & 0.259 & 0.289 & 0.263 & 0.286 & 0.262 & 0.287 \\
&& 720 & 0.340 & 0.334 & 0.345 & 0.344 & 0.343 & 0.339 & 0.343 & 0.342 \\
\cmidrule(lr){3-11}
&& Avg & 0.236 & 0.260 & 0.240 & 0.269 & 0.241 & 0.264 & 0.242 & 0.268 \\
\cmidrule(lr){2-11}

& \multirow{5}{*}{192}
& 96 & 0.142 & 0.180 & 0.146 & 0.195 & 0.154 & 0.194 & 0.153 & 0.200 \\
&& 192 & 0.191 & 0.230 & 0.194 & 0.243 & 0.198 & 0.238 & 0.200 & 0.244 \\
&& 336 & 0.250 & 0.278 & 0.256 & 0.289 & 0.254 & 0.283 & 0.257 & 0.287 \\
&& 720 & 0.333 & 0.331 & 0.334 & 0.341 & 0.327 & 0.333 & 0.326 & 0.335 \\
\cmidrule(lr){3-11}
&& Avg & 0.229 & 0.255 & 0.232 & 0.267 & 0.233 & 0.262 & 0.234 & 0.266 \\
\cmidrule(lr){2-11}

& \multirow{5}{*}{336}
& 96 & 0.153 & 0.192 & 0.153 & 0.202 & 0.148 & 0.191 & 0.154 & 0.205 \\
&& 192 & 0.209 & 0.241 & 0.215 & 0.255 & 0.199 & 0.242 & 0.199 & 0.246 \\
&& 336 & 0.269 & 0.287 & 0.288 & 0.302 & 0.247 & 0.281 & 0.249 & 0.285 \\
&& 720 & 0.363 & 0.347 & 0.372 & 0.360 & 0.319 & 0.330 & 0.325 & 0.337 \\
\cmidrule(lr){3-11}
&& Avg & 0.248 & 0.267 & 0.257 & 0.280 & 0.228 & 0.261 & 0.232 & 0.268 \\
\cmidrule(lr){2-11}

& \multirow{5}{*}{720}
& 96 & 0.151 & 0.195 & 0.159 & 0.213 & 0.151 & 0.198 & 0.160 & 0.214 \\
&& 192 & 0.216 & 0.254 & 0.219 & 0.266 & 0.203 & 0.245 & 0.210 & 0.258 \\
&& 336 & 0.278 & 0.301 & 0.290 & 0.312 & 0.262 & 0.294 & 0.260 & 0.299 \\
&& 720 & 0.357 & 0.349 & 0.392 & 0.377 & 0.318 & 0.332 & 0.328 & 0.343 \\
\cmidrule(lr){3-11}
&& Avg & 0.250 & 0.275 & 0.265 & 0.292 & 0.233 & 0.267 & 0.240 & 0.278 \\
    \bottomrule
\end{tabular}
\end{table}

Additional experimental results of varying historical lengths are available in \autoref{tab:vary_seq_len}, complementing the fixed length of 96 used in the main text. The forecast models selected include TimeFilter~\citep{hu2025timefilter}, which is the recent state-of-the-art forecast model, and TQNet~\citep{lin2025TQNet}. The results demonstrate that \time\ consistently improves both forecast models across different historical sequence lengths.

\subsection{Random Seed Sensitivity}

Additional experimental results of random seed sensitivity are available in Table \ref{tab:cvloss_ablation_singlecol}, where we report the mean and standard deviation of results obtained from experiments conducted with seven different random seeds (2020--2026). The results indicate minimal sensitivity of the proposed method to random initialization, as all standard deviations remain below 0.004.

\begin{table}[t]
\centering
\caption{CvLoss ablation on ECL and Weather. Results are reported as mean$_{\pm\mathrm{std}}$ over seeds 2020--2026. Lower is better.}
\label{tab:cvloss_ablation_singlecol}
\scriptsize
\setlength{\tabcolsep}{3pt}
\renewcommand{\arraystretch}{1.05}
\resizebox{\columnwidth}{!}{%
\begin{tabular}{cc|cc|cc|cc|cc}
\toprule
\multirow{2}{*}{\textbf{Hor.}} & \multirow{2}{*}{\textbf{Met.}}
& \multicolumn{4}{c|}{\textbf{ECL}}
& \multicolumn{4}{c}{\textbf{Weather}} \\
\cmidrule(lr){3-6} \cmidrule(lr){7-10}
& & \textbf{TimeFilter} & \textbf{+CvLoss} & \textbf{TQNet} & \textbf{+CvLoss}
  & \textbf{TimeFilter} & \textbf{+CvLoss} & \textbf{TQNet} & \textbf{+CvLoss} \\
\midrule

\multirow{2}{*}{96}
& MSE
& $0.133_{\pm0.001}$ & $\mathbf{0.131}_{\pm0.000}$ & $0.134_{\pm0.000}$ & $\mathbf{0.133}_{\pm0.000}$
& $0.155_{\pm0.002}$ & $\mathbf{0.150}_{\pm0.001}$ & $0.157_{\pm0.001}$ & $\mathbf{0.156}_{\pm0.002}$ \\
& MAE
& $0.230_{\pm0.001}$ & $\mathbf{0.225}_{\pm0.000}$ & $0.230_{\pm0.000}$ & $\mathbf{0.225}_{\pm0.000}$
& $0.202_{\pm0.002}$ & $\mathbf{0.189}_{\pm0.001}$ & $0.200_{\pm0.001}$ & $\mathbf{0.194}_{\pm0.003}$ \\
\midrule

\multirow{2}{*}{192}
& MSE
& $0.154_{\pm0.002}$ & $\mathbf{0.152}_{\pm0.000}$ & $0.153_{\pm0.000}$ & $\mathbf{0.151}_{\pm0.000}$
& $0.204_{\pm0.002}$ & $\mathbf{0.201}_{\pm0.002}$ & $0.205_{\pm0.000}$ & $\mathbf{0.204}_{\pm0.001}$ \\
& MAE
& $0.248_{\pm0.001}$ & $\mathbf{0.243}_{\pm0.000}$ & $0.246_{\pm0.000}$ & $\mathbf{0.241}_{\pm0.000}$
& $0.248_{\pm0.002}$ & $\mathbf{0.239}_{\pm0.002}$ & $0.245_{\pm0.001}$ & $\mathbf{0.239}_{\pm0.000}$ \\
\midrule

\multirow{2}{*}{336}
& MSE
& $0.164_{\pm0.002}$ & $\mathbf{0.161}_{\pm0.001}$ & $0.169_{\pm0.000}$ & $\mathbf{0.166}_{\pm0.000}$
& $0.261_{\pm0.002}$ & $\mathbf{0.259}_{\pm0.002}$ & $0.263_{\pm0.001}$ & $\mathbf{0.262}_{\pm0.001}$ \\
& MAE
& $0.261_{\pm0.002}$ & $\mathbf{0.257}_{\pm0.000}$ & $0.264_{\pm0.000}$ & $\mathbf{0.257}_{\pm0.000}$
& $0.290_{\pm0.001}$ & $\mathbf{0.285}_{\pm0.002}$ & $0.288_{\pm0.001}$ & $\mathbf{0.285}_{\pm0.001}$ \\
\midrule

\multirow{2}{*}{720}
& MSE
& $0.184_{\pm0.003}$ & $\mathbf{0.181}_{\pm0.001}$ & $0.204_{\pm0.004}$ & $\mathbf{0.200}_{\pm0.003}$
& $0.345_{\pm0.001}$ & $\mathbf{0.341}_{\pm0.002}$ & $0.343_{\pm0.001}$ & $0.343_{\pm0.000}$ \\
& MAE
& $0.284_{\pm0.002}$ & $\mathbf{0.282}_{\pm0.002}$ & $0.297_{\pm0.003}$ & $\mathbf{0.292}_{\pm0.003}$
& $0.344_{\pm0.001}$ & $\mathbf{0.336}_{\pm0.001}$ & $0.342_{\pm0.000}$ & $\mathbf{0.339}_{\pm0.000}$ \\
\midrule

\multirow{2}{*}{Avg}
& MSE
& $0.158_{\pm0.002}$ & $\mathbf{0.156}_{\pm0.001}$ & $0.165_{\pm0.001}$ & $\mathbf{0.162}_{\pm0.001}$
& $0.241_{\pm0.001}$ & $\mathbf{0.238}_{\pm0.001}$ & $0.242_{\pm0.000}$ & $\mathbf{0.241}_{\pm0.001}$ \\
& MAE
& $0.256_{\pm0.001}$ & $\mathbf{0.252}_{\pm0.001}$ & $0.259_{\pm0.001}$ & $\mathbf{0.254}_{\pm0.001}$
& $0.271_{\pm0.001}$ & $\mathbf{0.262}_{\pm0.001}$ & $0.269_{\pm0.000}$ & $\mathbf{0.264}_{\pm0.001}$ \\
\bottomrule
\end{tabular}%
}
\end{table}

\subsection{Paired effect sizes}
\label{app:paired}

The standard deviations in Table~\ref{tab:cvloss_ablation_singlecol} are \emph{marginal}: each summarises the spread of one arm across seeds. Comparing two methods by asking whether their marginal intervals overlap is the wrong instrument here, and it is a conservative one, because the two arms are not independent samples. Both are trained with the same seed indices, the same data ordering and the same initialisation scheme, so a large part of the seed-to-seed variance is shared and cancels in the difference. This subsection therefore reports the paired quantity directly.

For each cell and each seed $s\in\{2020,\dots,2026\}$ we form the paired difference
\begin{equation}
d_s = \mathcal{M}_{\mathrm{base},s} - \mathcal{M}_{\time,s},
\end{equation}
where $\mathcal{M}$ is the reported error metric, so that $d_s>0$ means \time\ reduces the error on that seed. Because both arms use the same seed indices, the mean paired difference coincides with the difference of the two reported seven-seed means, $\bar d=\frac{1}{7}\sum_{s}d_s=\bar{\mathcal{M}}_{\mathrm{base}}-\bar{\mathcal{M}}_{\time}$. We summarise each cell by $\bar d$, by the relative reduction $\bar d/\bar{\mathcal{M}}_{\mathrm{base}}$, and by the paired $95\%$ interval $\bar d \pm t_{0.975,6}\,\mathrm{SD}(d_s)/\sqrt{7}$ with $t_{0.975,6}=2.447$, computed from the unrounded per-seed outputs.

Of the $40$ cells covered by Table~\ref{tab:cvloss_ablation_singlecol}, $27$ are already separated by their marginal intervals. Table~\ref{tab:paired_effects} reports the paired analysis for the remaining $13$, which are exactly the cases the marginal criterion leaves unresolved.

\begin{table}[htbp]
  \caption{Paired seven-seed effects for the $13$ cells whose marginal intervals overlap in Table~\ref{tab:cvloss_ablation_singlecol}. $\bar d$ is the mean paired reduction in the stated metric, \emph{Rel.} the corresponding relative reduction, and the interval is the paired $95\%$ confidence interval for $\bar d$. Positive values favour \time. No additional training was performed: these are the same seven completed runs re-analysed in paired form.}
  \label{tab:paired_effects}
  \centering
  \setlength{\tabcolsep}{6pt}
  \small
  \begin{threeparttable}
  \begin{tabular}{llccccc}
    \toprule
    Dataset & Backbone & $T$ & Metric & $\bar d$ & Rel. & $95\%$ CI \\
    \midrule
    Weather & TQNet      & 96   & MSE & 0.0011 & 0.70\% & [\,0.0002,\ 0.0020\,] \\
    ECL     & TimeFilter & 192  & MSE & 0.0018 & 1.17\% & [\,0.0003,\ 0.0033\,] \\
    Weather & TimeFilter & 192  & MSE & 0.0024 & 1.18\% & [\,0.0004,\ 0.0044\,] \\
    ECL     & TimeFilter & 336  & MSE & 0.0026 & 1.59\% & [\,0.0004,\ 0.0048\,] \\
    Weather & TimeFilter & 336  & MSE & 0.0019 & 0.73\% & [\,0.0001,\ 0.0037\,] \\
    Weather & TQNet      & 336  & MSE & 0.0008 & 0.30\% & [\,$-$0.0001,\ 0.0017\,] \\
    ECL     & TimeFilter & 720  & MSE & 0.0027 & 1.47\% & [\,0.0003,\ 0.0051\,] \\
    ECL     & TQNet      & 720  & MSE & 0.0041 & 2.01\% & [\,0.0005,\ 0.0077\,] \\
    Weather & TQNet      & 720  & MSE & 0.0004 & 0.12\% & [\,$-$0.0001,\ 0.0009\,] \\
    ECL     & TimeFilter & 720  & MAE & 0.0021 & 0.74\% & [\,0.0003,\ 0.0039\,] \\
    ECL     & TQNet      & 720  & MAE & 0.0048 & 1.62\% & [\,0.0004,\ 0.0092\,] \\
    ECL     & TimeFilter & Avg  & MSE & 0.0020 & 1.27\% & [\,0.0008,\ 0.0032\,] \\
    Weather & TQNet      & Avg  & MSE & 0.0009 & 0.37\% & [\,0.0001,\ 0.0017\,] \\
    \bottomrule
  \end{tabular}
  \end{threeparttable}
\end{table}

All $13$ paired means favour \time, and $11$ of the $13$ paired intervals exclude zero. The two that do not, Weather/TQNet at $T=336$ and $T=720$, both MSE, are positive but small, and we report them as inconclusive rather than as wins; they are the two settings in which the effect of \time\ is genuinely at the resolution of the experiment. Across all $40$ cells the median relative reduction is $1.66\%$ with an interquartile range of $1.13$ percentage points, which we regard as the honest summary of the size of the effect on these two datasets.

Two remarks on how these numbers should not be used. First, we do not report a pooled significance test over the full grid of metric--horizon--dataset cells. Such a test would treat the cells as independent observations, which they are not: MSE and MAE are computed from the same runs, and horizons within a dataset share the data, the split and the configuration. A rank-based test over the same cells would inherit the same dependence, so substituting one statistic for another would not repair the assumption. Second, none of the paper's claims requires it. The evidence is the direction and consistency of the controlled comparisons themselves, summarised in Table~\ref{tab:controlled}, together with the paired intervals above, and these are descriptive statements about the evaluated settings rather than inferences to a population of tasks.

\subsection{Complexity}\label{sec:comp_app}

We evaluate the computational characteristics of CvLoss from two perspectives: 
(i) the training and inference overhead introduced by the original formulation, and 
(ii) the scalability improvement brought by random edge sampling.

We first measure the runtime overhead of CvLoss on ETTh2 and ECL with the batch size fixed to 32. 
As shown in Fig.~\ref{fig:training_latency}, CvLoss introduces additional computation only during training, since it regularizes cross-variable dependence through an extra loss term. 
On the small-scale ETTh2 benchmark, this overhead is marginal in both the forward and backward passes. 
On the larger ECL benchmark, the overhead becomes more visible, especially in backpropagation, which is expected because modeling cross-variable interactions is more expensive when the number of variables is large. 
Importantly, although the prediction length $T$ varies from 32 to 1024, the runtime remains in a similar range rather than increasing sharply, indicating that the added cost is well controlled across different horizons.

At inference time, CvLoss is removed from the computation graph and therefore does not participate in prediction. 
As shown in Fig.~\ref{fig:inference_latency}, the inference latency of the model trained with CvLoss remains highly comparable to that of the baseline iTransformer on both datasets. 
The ratio plots further show that the latency fluctuates around parity, confirming that CvLoss introduces no practical deployment-time overhead.

While the above results demonstrate that CvLoss is inexpensive at inference and manageable during training, the original formulation still builds a fully connected cross-variable graph.
When the number of patches is $P$ and the variable dimension is $D$, the graph contains $P \times D$ nodes, and the number of pairwise interactions scales as $O(P^2D^2)$.
This quadratic growth becomes a scalability bottleneck on high-dimensional datasets such as ECL.
To address this issue, we further introduce \emph{random edge sampling}, which limits the number of evaluated interactions to at most 1000 per batch.
In practice, this truncates the cost of CvLoss from rapidly increasing quadratic complexity to an effectively constant level.
We emphasise that random edge sampling is a scalability device studied in this subsection alone: it produced none of the accuracy numbers reported elsewhere in the paper, all of which use the complete cross-variable graph without sampling or selection.

Fig.~\ref{fig:cvl_runtime} compares the training runtime before and after sampling. 
On ETTh2, where the variable dimension is small, the full CvLoss remains tractable, but the sampled version still makes the runtime flatter as the number of patches increases. 
On ECL, the benefit is much more pronounced: the runtime of the full CvLoss grows rapidly with the number of patches, whereas the sampled version remains almost unchanged. 
This confirms that random edge sampling effectively removes the main scalability bottleneck of the original formulation.

Fig.~\ref{fig:cvl_gain} further shows that random edge sampling substantially improves training efficiency without sacrificing predictive performance. As the number of patches increases, the speedup grows rapidly, while the relative MSE change remains close to zero.

Fig.~\ref{fig:cvl_runtime} summarizes this observation from an efficiency--accuracy perspective. The results show that large computational gains can be achieved with negligible forecasting degradation, suggesting that a small number of randomly sampled interactions is sufficient to preserve the regularization effect of CvLoss in practice.

\begin{figure}[t]
\begin{center}
\subfigure[ETTh2: forward (left) and backward (right).]{
\includegraphics[width=0.23\linewidth]{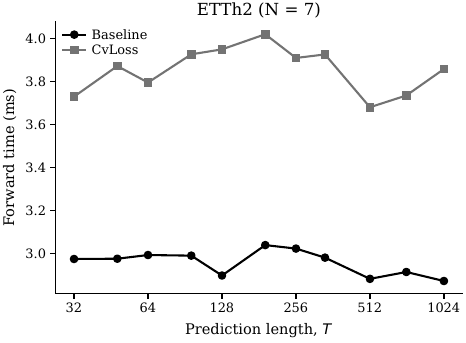}
\includegraphics[width=0.23\linewidth]{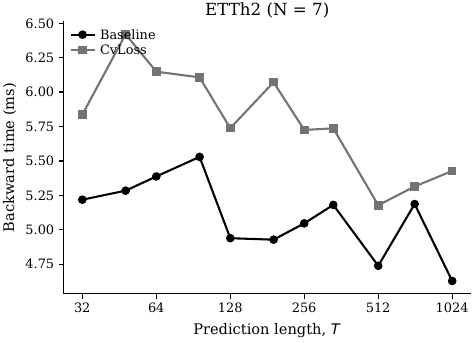}
}
\subfigure[ECL: forward (left) and backward (right).]{
\includegraphics[width=0.23\linewidth]{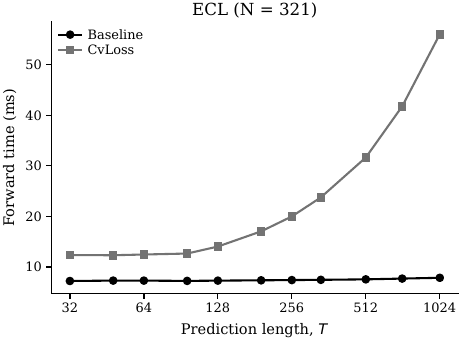}
\includegraphics[width=0.23\linewidth]{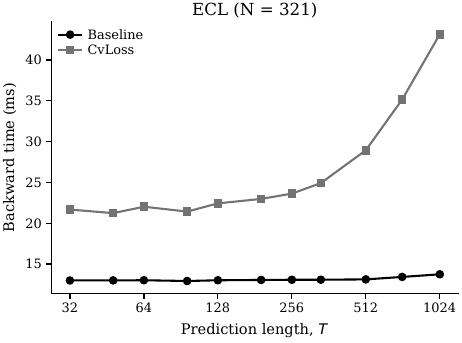}
}
\caption{Training latency on ETTh2 and ECL. CvLoss adds little overhead on ETTh2 and a larger, mainly backward-pass overhead on ECL.}
\label{fig:training_latency}
\end{center}
\end{figure}

\begin{figure}[t]
\begin{center}
\subfigure[ETTh2: inference latency (left) and latency ratio (right).]{
\includegraphics[width=0.23\linewidth]{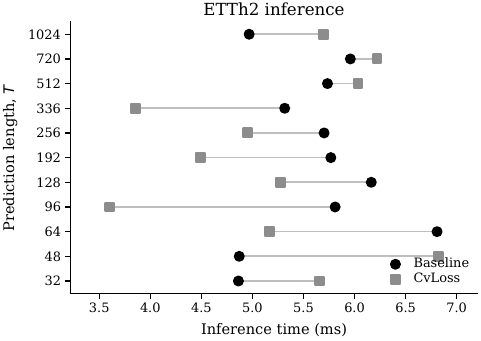}
\includegraphics[width=0.23\linewidth]{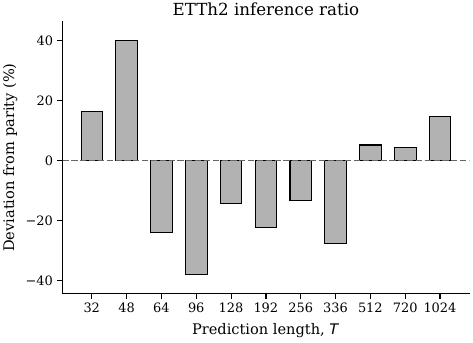}
}
\subfigure[ECL: inference latency (left) and latency ratio (right).]{
\includegraphics[width=0.23\linewidth]{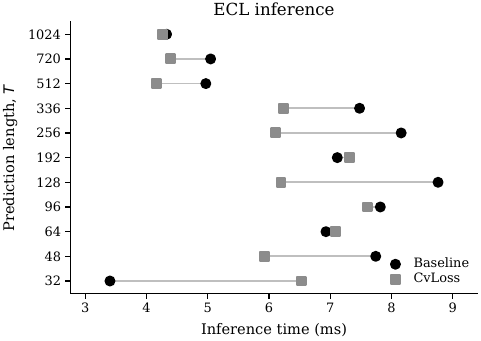}
\includegraphics[width=0.23\linewidth]{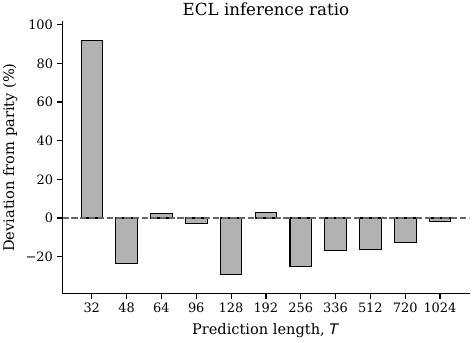}
}
\caption{Inference latency on ETTh2 and ECL. CvLoss remains comparable to the baseline, with ratios close to parity.}
\label{fig:inference_latency}
\end{center}
\end{figure}

\begin{figure}[t]
\begin{center}
\subfigure[ETTh2 runtime.]{
\includegraphics[width=0.4\linewidth]{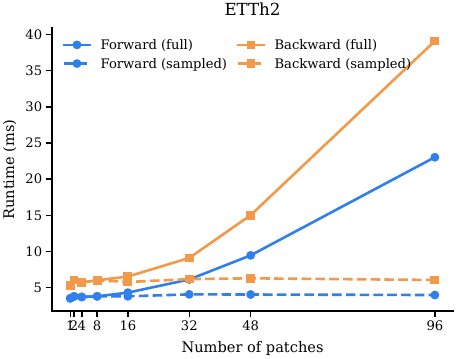}
}
\subfigure[ECL runtime.]{
\includegraphics[width=0.4\linewidth]{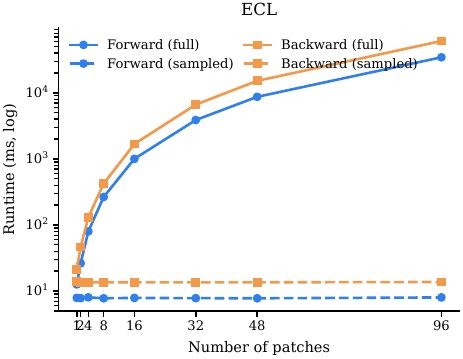}
}
\caption{Training runtime versus the number of patches. Dashed lines denote random edge sampling with at most 1000 edges.}
\label{fig:cvl_runtime}
\end{center}
\end{figure}

\begin{figure}[t]
\begin{center}
\subfigure[ECL speedup.]{
\includegraphics[width=0.31\linewidth]{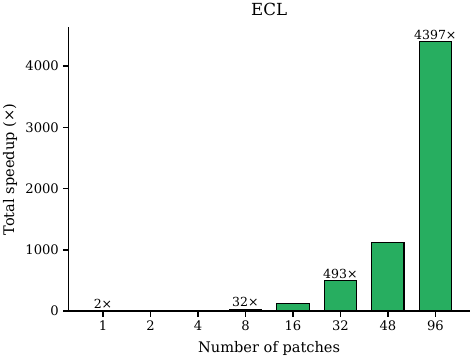}
}
\subfigure[MSE change.]{
\includegraphics[width=0.31\linewidth]{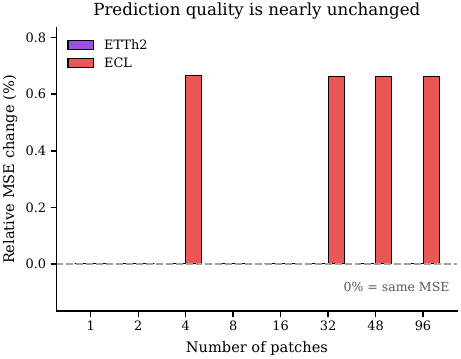}
}
\subfigure[Trade-off.]{
\includegraphics[width=0.31\linewidth]{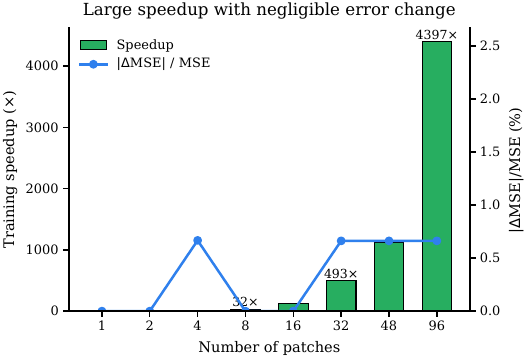}
}
\caption{Optimization benefit of random edge sampling. The speedup grows rapidly with the number of patches, while the relative MSE change remains close to zero. The trade-off plot summarizes that large speedups are achieved with negligible performance degradation.}
\label{fig:cvl_gain}
\end{center}
\end{figure}

\begin{table}[htbp]
\caption{Comprehensive long-term forecasting results grouped by backbone models in Benchmark TS-Library~\citep{Timesnet}.} \label{tab:comprehensive_all_variants}
\renewcommand{\arraystretch}{1.1}
\setlength{\tabcolsep}{2.2pt}
\scriptsize
\centering
\resizebox{\linewidth}{!}{
\begin{tabular}{lc|cc|cc|cc|cc|cc|cc|cc|cc|cc}
\toprule
\multicolumn{2}{c|}{\multirow{2}{*}{Models}} & \multicolumn{6}{c|}{iTransformer} & \multicolumn{6}{c|}{DLinear} & \multicolumn{6}{c}{PatchTST} \\
\cmidrule(lr){3-8} \cmidrule(lr){9-14} \cmidrule(lr){15-20}
\multicolumn{2}{c|}{} & \multicolumn{2}{c|}{Original} & \multicolumn{2}{c|}{DBLoss} & \multicolumn{2}{c|}{\cellcolor{tabhighlight}\time} & \multicolumn{2}{c|}{Original} & \multicolumn{2}{c|}{DBLoss} & \multicolumn{2}{c|}{\cellcolor{tabhighlight}\time} & \multicolumn{2}{c|}{Original} & \multicolumn{2}{c|}{DBLoss} & \multicolumn{2}{c}{\cellcolor{tabhighlight}\time} \\ 
\multicolumn{2}{c|}{} & \multicolumn{2}{c|}{(2024)} & \multicolumn{2}{c|}{(2025)} & \multicolumn{2}{c|}{\cellcolor{tabhighlight}(Ours)} & \multicolumn{2}{c|}{(2023)} & \multicolumn{2}{c|}{(2025)} & \multicolumn{2}{c|}{\cellcolor{tabhighlight}(Ours)} & \multicolumn{2}{c|}{(2023)} & \multicolumn{2}{c|}{(2025)} & \multicolumn{2}{c}{\cellcolor{tabhighlight}(Ours)} \\
\cmidrule(lr){3-4} \cmidrule(lr){5-6} \cmidrule(lr){7-8} \cmidrule(lr){9-10} \cmidrule(lr){11-12} \cmidrule(lr){13-14} \cmidrule(lr){15-16} \cmidrule(lr){17-18} \cmidrule(lr){19-20}
\multicolumn{2}{c|}{Metric} & MSE & MAE & MSE & MAE & \cellcolor{tabhighlight}MSE & \cellcolor{tabhighlight}MAE & MSE & MAE & MSE & MAE & \cellcolor{tabhighlight}MSE & \cellcolor{tabhighlight}MAE & MSE & MAE & MSE & MAE & \cellcolor{tabhighlight}MSE & \cellcolor{tabhighlight}MAE \\
\midrule
\multirow{5}{*}{\rotatebox[origin=c]{90}{ETTm1}} 
& 96  & 0.347 & 0.378 & \second{0.333} & \second{0.356} & \cellcolor{tabhighlight}\best{0.317} & \cellcolor{tabhighlight}\best{0.349} & 0.346 & 0.374 & \second{0.337} & \second{0.359} & \cellcolor{tabhighlight}\best{0.330} & \cellcolor{tabhighlight}\best{0.355} & \second{0.324} & \second{0.365} & 0.347 & 0.368 & \cellcolor{tabhighlight}\best{0.311} & \cellcolor{tabhighlight}\best{0.347} \\
& 192 & \second{0.375} & 0.389 & 0.380 & \second{0.381} & \cellcolor{tabhighlight}\best{0.369} & \cellcolor{tabhighlight}\best{0.379} & 0.382 & 0.391 & \second{0.378} & \second{0.380} & \cellcolor{tabhighlight}\best{0.374} & \cellcolor{tabhighlight}\best{0.378} & \second{0.372} & \second{0.392} & 0.397 & 0.394 & \cellcolor{tabhighlight}\best{0.364} & \cellcolor{tabhighlight}\best{0.377} \\
& 336 & \second{0.409} & 0.412 & 0.415 & \second{0.404} & \cellcolor{tabhighlight}\best{0.398} & \cellcolor{tabhighlight}\best{0.397} & 0.415 & 0.415 & \second{0.408} & \best{0.401} & \cellcolor{tabhighlight}\best{0.407} & \cellcolor{tabhighlight}\second{0.402} & \second{0.398} & \second{0.409} & 0.434 & 0.415 & \cellcolor{tabhighlight}\best{0.392} & \cellcolor{tabhighlight}\best{0.397} \\
& 720 & \second{0.477} & 0.450 & 0.481 & \best{0.441} & \cellcolor{tabhighlight}\best{0.474} & \cellcolor{tabhighlight}\second{0.442} & 0.473 & 0.451 & \best{0.467} & \best{0.437} & \cellcolor{tabhighlight}\second{0.468} & \cellcolor{tabhighlight}\second{0.440} & \second{0.457} & \second{0.444} & 0.484 & 0.446 & \cellcolor{tabhighlight}\best{0.455} & \cellcolor{tabhighlight}\best{0.438} \\
\cmidrule(lr){2-20}
\rowcolor{blue!15} \cellcolor{white} & \emph{Avg} & \second{0.402} & 0.407 & \second{0.402} & \second{0.395} & \best{0.390} & \best{0.392} & 0.404 & 0.408 & \second{0.397} & \second{0.394} & \best{0.395} & \best{0.394} & \second{0.388} & \second{0.402} & 0.415 & 0.406 & \cellcolor{tabhighlight}\best{0.381} & \cellcolor{tabhighlight}\best{0.390} \\
\midrule
\multirow{5}{*}{\rotatebox[origin=c]{90}{ETTm2}} 
& 96  & 0.185 & 0.271 & \second{0.181} & \second{0.259} & \cellcolor{tabhighlight}\best{0.176} & \cellcolor{tabhighlight}\best{0.256} & 0.193 & 0.293 & \second{0.183} & \best{0.262} & \cellcolor{tabhighlight}\best{0.180} & \cellcolor{tabhighlight}\second{0.262} & 0.186 & 0.268 & \second{0.180} & \second{0.257} & \cellcolor{tabhighlight}\best{0.178} & \cellcolor{tabhighlight}\best{0.255} \\
& 192 & 0.252 & 0.312 & \second{0.246} & \second{0.302} & \cellcolor{tabhighlight}\best{0.241} & \cellcolor{tabhighlight}\best{0.298} & 0.285 & 0.361 & \second{0.257} & \second{0.317} & \cellcolor{tabhighlight}\best{0.246} & \cellcolor{tabhighlight}\best{0.313} & \second{0.246} & 0.305 & 0.248 & \second{0.302} & \cellcolor{tabhighlight}\best{0.240} & \cellcolor{tabhighlight}\best{0.296} \\
& 336 & 0.317 & 0.353 & \second{0.307} & \second{0.340} & \cellcolor{tabhighlight}\best{0.302} & \cellcolor{tabhighlight}\best{0.337} & 0.385 & 0.429 & \second{0.308} & \second{0.352} & \cellcolor{tabhighlight}\best{0.305} & \cellcolor{tabhighlight}\best{0.346} & 0.311 & 0.348 & \second{0.308} & \second{0.343} & \cellcolor{tabhighlight}\best{0.300} & \cellcolor{tabhighlight}\best{0.337} \\
& 720 & \best{0.406} & 0.402 & 0.409 & \best{0.398} & \cellcolor{tabhighlight}\second{0.407} & \cellcolor{tabhighlight}\second{0.400} & 0.556 & 0.523 & \best{0.413} & \best{0.422} & \cellcolor{tabhighlight}\second{0.414} & \cellcolor{tabhighlight}\second{0.425} & 0.418 & 0.414 & \second{0.407} & \second{0.397} & \cellcolor{tabhighlight}\best{0.402} & \cellcolor{tabhighlight}\best{0.397} \\
\cmidrule(lr){2-20}
\rowcolor{blue!15} \cellcolor{white} & \emph{Avg} & 0.290 & 0.334 & \second{0.286} & \second{0.325} & \best{0.282} & \best{0.323} & 0.355 & 0.402 & \second{0.290} & \second{0.338} & \best{0.286} & \best{0.337} & 0.290 & 0.334 & \second{0.286} & \second{0.325} & \cellcolor{tabhighlight}\best{0.280} & \cellcolor{tabhighlight}\best{0.321} \\
\midrule
\multirow{5}{*}{\rotatebox[origin=c]{90}{ETTh1}} 
& 96  & \second{0.388} & 0.403 & 0.389 & \second{0.401} & \cellcolor{tabhighlight}\best{0.372} & \cellcolor{tabhighlight}\best{0.394} & 0.396 & 0.411 & \second{0.391} & \second{0.404} & \cellcolor{tabhighlight}\best{0.380} & \cellcolor{tabhighlight}\best{0.389} & \second{0.377} & \second{0.397} & 0.383 & 0.396 & \cellcolor{tabhighlight}\best{0.376} & \cellcolor{tabhighlight}\best{0.394} \\
& 192 & 0.440 & 0.434 & \second{0.439} & \second{0.431} & \cellcolor{tabhighlight}\best{0.433} & \cellcolor{tabhighlight}\best{0.427} & 0.445 & 0.440 & \second{0.440} & \second{0.433} & \cellcolor{tabhighlight}\best{0.433} & \cellcolor{tabhighlight}\best{0.425} & \second{0.426} & 0.432 & 0.429 & \best{0.426} & \cellcolor{tabhighlight}\best{0.424} & \cellcolor{tabhighlight}\second{0.427} \\
& 336 & 0.483 & 0.458 & \second{0.479} & \second{0.453} & \cellcolor{tabhighlight}\best{0.474} & \cellcolor{tabhighlight}\best{0.451} & 0.487 & 0.465 & \best{0.478} & \best{0.454} & \cellcolor{tabhighlight}\second{0.485} & \cellcolor{tabhighlight}\second{0.458} & 0.469 & \second{0.457} & \best{0.462} & \best{0.443} & \cellcolor{tabhighlight}\second{0.468} & \cellcolor{tabhighlight}\second{0.453} \\
& 720 & 0.500 & 0.491 & \best{0.473} & \best{0.471} & \cellcolor{tabhighlight}\second{0.491} & \cellcolor{tabhighlight}\second{0.484} & 0.513 & 0.510 & \best{0.489} & \second{0.490} & \cellcolor{tabhighlight}\second{0.512} & \cellcolor{tabhighlight}\best{0.509} & 0.519 & 0.504 & \best{0.452} & \best{0.459} & \cellcolor{tabhighlight}\second{0.517} & \cellcolor{tabhighlight}\second{0.499} \\
\cmidrule(lr){2-20}
\rowcolor{blue!15} \cellcolor{white} & \emph{Avg} & 0.453 & 0.447 & \second{0.445} & \best{0.439} & \best{0.443} & \best{0.439} & 0.460 & 0.457 & \best{0.449} & \best{0.445} & \second{0.453} & \best{0.445} & 0.448 & 0.448 & \best{0.431} & \best{0.431} & \cellcolor{tabhighlight}\second{0.446} & \cellcolor{tabhighlight}\second{0.443} \\
\midrule
\multirow{5}{*}{\rotatebox[origin=c]{90}{ETTh2}} 
& 96  & 0.297 & 0.348 & \second{0.292} & \best{0.341} & \cellcolor{tabhighlight}\best{0.289} & \cellcolor{tabhighlight}\best{0.341} & 0.341 & 0.395 & \second{0.300} & \second{0.351} & \cellcolor{tabhighlight}\best{0.293} & \cellcolor{tabhighlight}\best{0.346} & 0.309 & 0.359 & \best{0.282} & \best{0.333} & \cellcolor{tabhighlight}\second{0.287} & \cellcolor{tabhighlight}\second{0.339} \\
& 192 & 0.379 & 0.400 & \second{0.378} & \second{0.394} & \cellcolor{tabhighlight}\best{0.375} & \cellcolor{tabhighlight}\best{0.392} & 0.482 & 0.479 & \second{0.389} & \second{0.408} & \cellcolor{tabhighlight}\best{0.377} & \cellcolor{tabhighlight}\best{0.399} & 0.381 & 0.407 & \best{0.364} & \best{0.386} & \cellcolor{tabhighlight}\second{0.367} & \cellcolor{tabhighlight}\second{0.389} \\
& 336 & \second{0.422} & 0.433 & 0.424 & \second{0.430} & \cellcolor{tabhighlight}\best{0.415} & \cellcolor{tabhighlight}\best{0.428} & 0.593 & 0.542 & \second{0.447} & \second{0.453} & \cellcolor{tabhighlight}\best{0.434} & \cellcolor{tabhighlight}\best{0.444} & 0.412 & 0.429 & \best{0.404} & \second{0.419} & \cellcolor{tabhighlight}\second{0.404} & \cellcolor{tabhighlight}\best{0.418} \\
& 720 & \second{0.428} & 0.449 & 0.429 & \second{0.444} & \cellcolor{tabhighlight}\best{0.419} & \cellcolor{tabhighlight}\best{0.440} & 0.840 & 0.661 & \best{0.590} & \best{0.539} & \cellcolor{tabhighlight}\second{0.600} & \cellcolor{tabhighlight}\second{0.549} & 0.436 & 0.456 & \best{0.411} & \best{0.433} & \cellcolor{tabhighlight}\second{0.421} & \cellcolor{tabhighlight}\second{0.441} \\
\cmidrule(lr){2-20}
\rowcolor{blue!15} \cellcolor{white} & \emph{Avg} & 0.381 & 0.407 & \second{0.381} & \second{0.402} & \best{0.375} & \best{0.400} & 0.564 & 0.519 & \second{0.431} & \second{0.438} & \best{0.426} & \best{0.434} & 0.384 & 0.413 & \best{0.365} & \best{0.393} & \cellcolor{tabhighlight}\second{0.370} & \cellcolor{tabhighlight}\second{0.396} \\
\midrule
\multirow{5}{*}{\rotatebox[origin=c]{90}{Weather}} 
& 96  & \second{0.184} & \second{0.224} & 0.191 & \second{0.224} & \cellcolor{tabhighlight}\best{0.176} & \cellcolor{tabhighlight}\best{0.211} & \second{0.196} & 0.256 & 0.201 & \best{0.245} & \cellcolor{tabhighlight}\best{0.195} & \cellcolor{tabhighlight}\second{0.246} & \second{0.176} & 0.218 & 0.179 & \second{0.213} & \cellcolor{tabhighlight}\best{0.173} & \cellcolor{tabhighlight}\best{0.211} \\
& 192 & \second{0.231} & \second{0.262} & 0.242 & 0.268 & \cellcolor{tabhighlight}\best{0.223} & \cellcolor{tabhighlight}\best{0.253} & 0.239 & 0.299 & \second{0.236} & \best{0.278} & \cellcolor{tabhighlight}\best{0.235} & \cellcolor{tabhighlight}\second{0.288} & \second{0.221} & 0.256 & 0.224 & \second{0.253} & \cellcolor{tabhighlight}\best{0.219} & \cellcolor{tabhighlight}\best{0.251} \\
& 336 & \second{0.286} & 0.302 & 0.288 & \second{0.298} & \cellcolor{tabhighlight}\best{0.279} & \cellcolor{tabhighlight}\best{0.293} & \second{0.281} & 0.331 & \best{0.280} & \best{0.312} & \cellcolor{tabhighlight}0.285 & \cellcolor{tabhighlight}\second{0.329} & 0.280 & 0.298 & \second{0.278} & \best{0.291} & \cellcolor{tabhighlight}\best{0.277} & \cellcolor{tabhighlight}\second{0.293} \\
& 720 & \second{0.363} & 0.352 & 0.364 & \second{0.349} & \cellcolor{tabhighlight}\best{0.360} & \cellcolor{tabhighlight}\best{0.347} & 0.345 & 0.382 & \best{0.339} & \best{0.359} & \cellcolor{tabhighlight}\second{0.345} & \cellcolor{tabhighlight}\second{0.379} & 0.356 & 0.349 & \second{0.355} & \best{0.342} & \cellcolor{tabhighlight}\best{0.355} & \cellcolor{tabhighlight}\second{0.343} \\
\cmidrule(lr){2-20}
\rowcolor{blue!15} \cellcolor{white} & \emph{Avg} & \second{0.266} & \second{0.285} & 0.271 & \second{0.285} & \best{0.260} & \best{0.276} & 0.265 & 0.317 & \best{0.264} & \best{0.299} & \cellcolor{tabhighlight}\second{0.265} & \cellcolor{tabhighlight}\second{0.310} & \second{0.258} & 0.280 & 0.259 & \second{0.275} & \cellcolor{tabhighlight}\best{0.256} & \cellcolor{tabhighlight}\best{0.274} \\
\midrule
\multirow{5}{*}{\rotatebox[origin=c]{90}{Electricity}} 
& 96  & \second{0.192} & \second{0.277} & 0.200 & 0.280 & \cellcolor{tabhighlight}\best{0.164} & \cellcolor{tabhighlight}\best{0.248} & \second{0.210} & 0.302 & 0.211 & \second{0.296} & \cellcolor{tabhighlight}\best{0.195} & \cellcolor{tabhighlight}\best{0.275} & \second{0.180} & \second{0.273} & 0.185 & 0.274 & \cellcolor{tabhighlight}\best{0.180} & \cellcolor{tabhighlight}\best{0.270} \\
& 192 & \second{0.201} & \second{0.289} & 0.206 & 0.290 & \cellcolor{tabhighlight}\best{0.183} & \cellcolor{tabhighlight}\best{0.264} & \second{0.210} & 0.305 & \second{0.210} & \second{0.299} & \cellcolor{tabhighlight}\best{0.194} & \cellcolor{tabhighlight}\best{0.278} & \second{0.187} & \second{0.280} & 0.192 & 0.281 & \cellcolor{tabhighlight}\best{0.187} & \cellcolor{tabhighlight}\best{0.276} \\
& 336 & \second{0.221} & \second{0.308} & 0.225 & 0.309 & \cellcolor{tabhighlight}\best{0.192} & \cellcolor{tabhighlight}\best{0.276} & \second{0.223} & 0.319 & \second{0.223} & \second{0.313} & \cellcolor{tabhighlight}\best{0.206} & \cellcolor{tabhighlight}\best{0.292} & \second{0.204} & \second{0.296} & 0.209 & 0.297 & \cellcolor{tabhighlight}\best{0.204} & \cellcolor{tabhighlight}\best{0.293} \\
& 720 & \second{0.267} & \second{0.344} & 0.271 & 0.345 & \cellcolor{tabhighlight}\best{0.263} & \cellcolor{tabhighlight}\best{0.337} & 0.258 & 0.350 & \second{0.257} & \second{0.342} & \cellcolor{tabhighlight}\best{0.244} & \cellcolor{tabhighlight}\best{0.329} & \second{0.246} & \second{0.328} & 0.252 & 0.331 & \cellcolor{tabhighlight}\best{0.243} & \cellcolor{tabhighlight}\best{0.325} \\
\cmidrule(lr){2-20}
\rowcolor{blue!15} \cellcolor{white} & \emph{Avg} & \second{0.220} & \second{0.305} & 0.226 & 0.306 & \best{0.201} & \best{0.281} & 0.225 & 0.319 & \second{0.222} & \second{0.310} & \best{0.210} & \best{0.293} & \second{0.204} & \second{0.294} & 0.209 & 0.296 & \cellcolor{tabhighlight}\best{0.204} & \cellcolor{tabhighlight}\best{0.291} \\
\midrule
\multirow{5}{*}{\rotatebox[origin=c]{90}{Traffic}} 
& 96  & \second{0.545} & \second{0.371} & 0.597 & 0.386 & \cellcolor{tabhighlight}\best{0.434} & \cellcolor{tabhighlight}\best{0.285} & 0.697 & 0.429 & \second{0.690} & \second{0.421} & \cellcolor{tabhighlight}\best{0.642} & \cellcolor{tabhighlight}\best{0.379} & \second{0.459} & 0.298 & 0.461 & \second{0.292} & \cellcolor{tabhighlight}\best{0.455} & \cellcolor{tabhighlight}\best{0.290} \\
& 192 & \second{0.560} & \second{0.380} & 0.595 & 0.386 & \cellcolor{tabhighlight}\best{0.469} & \cellcolor{tabhighlight}\best{0.309} & 0.647 & 0.407 & \second{0.639} & \second{0.399} & \cellcolor{tabhighlight}\best{0.595} & \cellcolor{tabhighlight}\best{0.361} & \second{0.469} & 0.301 & 0.470 & \second{0.296} & \cellcolor{tabhighlight}\best{0.465} & \cellcolor{tabhighlight}\best{0.295} \\
& 336 & \second{0.591} & \second{0.395} & 0.620 & 0.400 & \cellcolor{tabhighlight}\best{0.465} & \cellcolor{tabhighlight}\best{0.297} & 0.653 & 0.410 & \second{0.645} & \second{0.400} & \cellcolor{tabhighlight}\best{0.601} & \cellcolor{tabhighlight}\best{0.358} & \second{0.483} & 0.307 & 0.483 & \second{0.301} & \cellcolor{tabhighlight}\best{0.478} & \cellcolor{tabhighlight}\best{0.298} \\
& 720 & \second{0.657} & \second{0.427} & 0.686 & 0.433 & \cellcolor{tabhighlight}\best{0.568} & \cellcolor{tabhighlight}\best{0.370} & 0.694 & 0.429 & \second{0.682} & \second{0.417} & \cellcolor{tabhighlight}\best{0.647} & \cellcolor{tabhighlight}\best{0.397} & \second{0.517} & 0.326 & 0.518 & \second{0.319} & \cellcolor{tabhighlight}\best{0.512} & \cellcolor{tabhighlight}\best{0.315} \\
\cmidrule(lr){2-20}
\rowcolor{blue!15} \cellcolor{white} & \emph{Avg} & \second{0.588} & \second{0.393} & 0.624 & 0.401 & \best{0.484} & \best{0.315} & 0.673 & 0.419 & \second{0.664} & \second{0.409} & \best{0.621} & \best{0.374} & \second{0.482} & 0.308 & 0.483 & \second{0.302} & \cellcolor{tabhighlight}\best{0.478} & \cellcolor{tabhighlight}\best{0.300} \\
\bottomrule
\end{tabular}
}

\vspace{3pt}
\noindent\parbox{\linewidth}{\scriptsize \textit{Note}: Average results (\emph{Avg}) are calculated over $T \in \{96, 192, 336, 720\}$. Evaluation for the \best{best} and \second{second-best} results are independently conducted within the \textbf{iTransformer}, the \textbf{DLinear}, and the \textbf{PatchTST}. "Ours" refers to different enhancements applied to the respective base models.}
\end{table}

\section{Complexity Analysis and Edge Selection via Structural Discrepancy}
\label{app:complexity_and_selection}

In the main text, we introduced CvLoss as a lightweight plug-in loss function that regularizes the structural consistency among variables. While a fully connected cross-variable graph effectively captures comprehensive spatiotemporal interactions, calculating the loss over all possible variable pairs yields a quadratic computational complexity of $\mathcal{O}(P^2 D^2)$. This poses a significant scalability challenge for high-dimensional multivariate systems. To address this bottleneck, we introduce a targeted edge selection mechanism based on structural discrepancy to significantly reduce the computational complexity.

\paragraph{Status of this mechanism.} Three edge constructions appear in this paper, and it matters which produced which result. The complete cross-variable graph of Section~\ref{cvl}, with no sampling and no selection, produced \emph{every} accuracy number in the main text and in the appendix. The random edge sampling of Appendix~\ref{sec:comp_app} and the top-$K$ selection described below are scalability devices for the $O(P^2D^2)$ edge count, and are evaluated only in the efficiency studies of Figures~\ref{fig:cvl_runtime} and~\ref{fig:cvl_gain}. The synchronous and asynchronous supports compared in Section~\ref{sec:ablation} are ablations of the topology, not the default. The distinction is important for one further reason: the top-$K$ rule below is the only construction in the paper that reads the ground-truth correlation matrix $\mathbf{C}(Y)$, and it does so during training only, on the training split. No configuration uses any statistic of the future window at test time.

\subsection{Targeting Maximum Structural Discrepancy}

As established in our theoretical analysis, the standard MSE fundamentally ignores latent cross-variable relationships in its design. Consequently, models trained purely with MSE often produce predictions with significant structural discrepancies. This phenomenon is visually evident in the cross-variable correlation heatmaps presented in our motivating examples (e.g., Figure~\ref{fig:idea}).

The heatmaps intuitively represent the structural relationships between variables, where each point corresponds to the correlation between a specific variable pair. By comparing the predicted correlation heatmap $\mathbf{C}(\hat{Y})$ with the ground-truth correlation heatmap $\mathbf{C}(Y)$, we observe clear differences indicating where the model fails to capture the true co-evolutionary dynamics.

To optimize computational efficiency, we do not need to compute the first-order difference loss over all points in the heatmap. Instead, we can selectively compute the loss exclusively for those variable pairs that exhibit the largest discrepancy between the prediction and the ground truth. Specifically, we measure the structural error for each variable pair $(i, j)$ as $\Delta_{ij} = |\mathbf{C}(\hat{Y})_{ij} - \mathbf{C}(Y)_{ij}|$. We then select the top-$K$ pairs with the maximum $\Delta_{ij}$ to construct our edge set $\mathcal{E}$. 

By calculating the CvLoss strictly on these most severely misaligned pairs, we force the model to correct its most critical structural failures. This targeted approach dramatically reduces the computational complexity from $\mathcal{O}(D^2)$ to a manageable $\mathcal{O}(K)$, demonstrating that CvLoss can be efficiently scaled to high-dimensional multivariate time series forecasting without suffering from quadratic computational overhead.

\subsection{Diagnostic study: does the objective track cross-variable structure?}
\label{app:diagnostic_weights}

\textbf{This subsection is a diagnostic and does not describe the deployed method.} Everywhere else in this paper, the two loss terms are combined by the fixed convex weight $\alpha\in(0,1)$ of Eq.~\eqref{eq:cvl_alpha}, selected on validation as described in Appendix~\ref{app:tuning}. We therefore ask a separate question: if the relative importance of the structural term were free to adapt, would it track how much cross-variable structure a dataset actually contains? To answer it, we run an auxiliary experiment in which the two terms carry \emph{unnormalised} coefficients learned during training. The quantities in Table~\ref{tab:learned_weights_corr} are those auxiliary coefficients. They are not values of $\alpha$, they are not constrained to $(0,1)$, and they produced no result reported anywhere else in the paper; we report them only for the correspondence described below.

With that caveat, Table~\ref{tab:learned_weights_corr} places the converged coefficients alongside the statistical properties of each dataset, specifically the average Concordance Correlation Coefficient (CCC) and Pearson Correlation Coefficient.

We observe a clear positive correspondence between the channel-wise correlation (CCC/Pearson), the dimensionality ($D$), and the converged coefficient on the structural term. For datasets with fewer variables and lower cross-channel correlation, such as the ETT series ($D=7$, CCC $\approx 0.14$--$0.18$), the optimization converges to relatively small and comparable coefficients for both MSE and the structural penalty, indicating a balanced reliance.

Conversely, for high-dimensional spatial-temporal graphs exhibiting strong internal consistency, such as the PEMS traffic datasets ($D \ge 170$, CCC $> 0.77$) and Solar ($D=137$, CCC $> 0.91$), the converged coefficient on the structural term grows substantially. This is the behaviour Theorem~\ref{thm:spatiotemporal_nll} predicts: the objective gap it identifies is larger precisely when the residual precision matrix is further from spherical, so a dataset with a rich, highly synchronized cross-variable geometry has more for the structural term to correct. The same pattern is visible in the deployed method, where the largest accuracy gains occur on the most correlated datasets (PEMS07, CCC $0.78$, $-11.3\%$ MSE; Solar, CCC $0.91$, $-4.4\%$) and the smallest on the least correlated (the ETT series, CCC $0.14$--$0.18$). The practical implication is that channel correlation can be measured before training to anticipate whether the term will help.

\begin{table}[ht]
\centering
\caption{\textbf{Diagnostic study only --- these are not values of $\alpha$.} Converged coefficients of an auxiliary experiment in which the two loss terms carry \emph{unnormalised} learned weights, reported alongside the channel-wise correlation of each dataset. They are unconstrained in magnitude, are not the convex weight $\alpha\in(0,1)$ of Eq.~\eqref{eq:cvl_alpha}, and produced none of the accuracy results in this paper; see Appendix~\ref{app:diagnostic_weights}. The datasets are ordered by their number of variables ($D$).}
\label{tab:learned_weights_corr}
\renewcommand{\arraystretch}{1.15}
\setlength{\tabcolsep}{6pt}
\small
\resizebox{\linewidth}{!}{
\begin{tabular}{l r c c c c}
\toprule
\textbf{Dataset} & \textbf{Channels ($D$)} & \textbf{Avg CCC} & \textbf{Avg Pearson} & \textbf{Weight for MSE} & \textbf{Weight for Add Loss} \\
\midrule
ETTh1   & 7   & 0.1765 & 0.2516 & 1.2706  & 1.2657  \\
ETTh2   & 7   & 0.1449 & 0.4001 & 1.2080  & 1.2002  \\
ETTm1   & 7   & 0.1804 & 0.2601 & 2.3572  & 1.9831  \\
ETTm2   & 7   & 0.1436 & 0.3952 & 2.1759  & 2.0209  \\
Weather & 21  & 0.2079 & 0.3520 & 2.4887  & 2.2864  \\
Solar   & 137 & 0.9115 & 0.9157 & 6.9919  & 3.4617  \\
PEMS08  & 170 & 0.7799 & 0.8161 & 17.1380 & 4.2828  \\
PEMS04  & 307 & 0.7739 & 0.7858 & 18.0127 & 4.3956  \\
ECL     & 321 & 0.4619 & 0.4965 & 8.1102  & 3.0385  \\
PEMS03  & 358 & 0.8206 & 0.8431 & 21.7118 & 4.8293  \\
Traffic & 862 & 0.5289 & 0.5643 & 4.6804  & 3.0653  \\
PEMS07  & 883 & 0.7810 & 0.8066 & 22.8313 & 4.9570  \\
\bottomrule
\end{tabular}
}
\end{table}

\section{Broader Impacts}

CvLoss is a general training objective for multivariate time series forecasting. Its potential positive impact is to improve forecasting reliability in applications such as energy management, traffic monitoring, weather analysis, and other operational systems where variables evolve jointly. Better forecasts may support more efficient resource allocation and planning. We do not identify direct negative societal impacts specific to CvLoss, since the method does not introduce new sensitive data, human-subject data, generative capabilities, or deployment-specific decision rules. Any societal risk would mainly come from downstream misuse or misinterpretation of forecasting models in application domains, rather than from the proposed loss itself.

\section{Statement on the Use of Large Language Models (LLMs)}
In accordance with the conference guidelines, we disclose our use of Large Language Models (LLMs) in the preparation of this paper as follows:

We used LLMs (specifically, OpenAI GPT-5.2, GPT-5.4 and Google Gemini 3) \emph{solely for checking grammar errors and improving the readability of the manuscript}. The LLMs \emph{were not involved in research ideation, the development of research contributions, experiment design, data analysis, or interpretation of results}. All substantive content and scientific claims were created entirely by the authors. The authors have reviewed all LLM-assisted text to ensure accuracy and originality, and take full responsibility for the contents of the paper. The LLMs are not listed as an author.

\clearpage
\section*{NeurIPS Paper Checklist}

\begin{enumerate}

\item {\bf Claims}
    \item[] Question: Do the main claims made in the abstract and introduction accurately reflect the paper's contributions and scope?
    \item[] Answer: \answerYes{} % Replace by \answerYes{}, \answerNo{}, or \answerNA{}.
    \item[] Justification: The abstract and Section~1 describe the output-side objective gap and graph residual regularizer; Sections~3--4 provide the corresponding analysis and benchmark evidence.

\item {\bf Limitations}
    \item[] Question: Does the paper discuss the limitations of the work performed by the authors?
    \item[] Answer: \answerYes{} % Replace by \answerYes{}, \answerNo{}, or \answerNA{}.
    \item[] Justification: Section~6 discusses the graph and patching assumptions, the Gaussian motivation versus the implemented $\ell_1$ penalty, and the scope of the forecasting evaluation.

\item {\bf Theory assumptions and proofs}
    \item[] Question: For each theoretical result, does the paper provide the full set of assumptions and a complete (and correct) proof?
    \item[] Answer: \answerYes{} % Replace by \answerYes{}, \answerNo{}, or \answerNA{}.
    \item[] Justification: The assumptions are stated with Theorem~\ref{thm:spatiotemporal_nll} and Proposition~\ref{prop:precision}; Appendix~A provides their derivations and proofs.

    \item {\bf Experimental result reproducibility}
    \item[] Question: Does the paper fully disclose all the information needed to reproduce the main experimental results of the paper to the extent that it affects the main claims and/or conclusions of the paper (regardless of whether the code and data are provided or not)?
    \item[] Answer: \answerYes{} % Replace by \answerYes{}, \answerNo{}, or \answerNA{}.
    \item[] Justification: Sections~3--4 specify the loss and evaluation protocol; Appendix~\ref{sec:reproduce} details datasets, preprocessing, validation-based selection and training, with additional results in Appendix~\ref{sec:results_app}.

\item {\bf Open access to data and code}
    \item[] Question: Does the paper provide open access to the data and code, with sufficient instructions to faithfully reproduce the main experimental results, as described in supplemental material?
    \item[] Answer: \answerYes{} % Replace by \answerYes{}, \answerNo{}, or \answerNA{}.
    \item[] Justification: The public repository at \url{https://github.com/Day333/CvLoss} provides the objective implementation, dependencies and experiment scripts; Appendix~\ref{sec:reproduce} describes the public benchmarks and reproduction protocol.

\item {\bf Experimental setting/details}
    \item[] Question: Does the paper specify all the training and test details (e.g., data splits, hyperparameters, how they were chosen, type of optimizer) necessary to understand the results?
    \item[] Answer: \answerYes{} % Replace by \answerYes{}, \answerNo{}, or \answerNA{}.
    \item[] Justification: Section~4.1 and Appendix~\ref{sec:reproduce} describe dataset splits, optimizer, early stopping and validation selection of the loss weight and patch length.

\item {\bf Experiment statistical significance}
    \item[] Question: Does the paper report error bars suitably and correctly defined or other appropriate information about the statistical significance of the experiments?
    \item[] Answer: \answerYes{} % Replace by \answerYes{}, \answerNo{}, or \answerNA{}.
    \item[] Justification: Figure~\ref{fig:backbone} defines its horizon-based intervals; Appendix~\ref{sec:results_app} reports means and standard deviations over seven seeds and paired confidence intervals, with their computation described.

\item {\bf Experiments compute resources}
    \item[] Question: For each experiment, does the paper provide sufficient information on the computer resources (type of compute workers, memory, time of execution) needed to reproduce the experiments?
    \item[] Answer: \answerNo{} % Replace by \answerYes{}, \answerNo{}, or \answerNA{}.
    \item[] Justification: Section~4.1 identifies the CPU and GPU hardware, and Appendix~\ref{sec:comp_app} reports training and inference timing measurements. A complete accounting of memory usage and total compute for every experiment is not provided.
    
\item {\bf Code of ethics}
    \item[] Question: Does the research conducted in the paper conform, in every respect, with the NeurIPS Code of Ethics \url{https://neurips.cc/public/EthicsGuidelines}?
    \item[] Answer: \answerYes{} % Replace by \answerYes{}, \answerNo{}, or \answerNA{}.
    \item[] Justification: The work evaluates a training objective on public forecasting benchmarks and involves no crowdsourcing or direct interaction with human participants. Dataset sources, reproducibility information and broader impacts are described in the appendix.

\item {\bf Broader impacts}
    \item[] Question: Does the paper discuss both potential positive societal impacts and negative societal impacts of the work performed?
    \item[] Answer: \answerYes{} % Replace by \answerYes{}, \answerNo{}, or \answerNA{}.
    \item[] Justification: The Broader Impacts section discusses applications in energy, traffic and weather forecasting and the risks of downstream misuse or misinterpretation.
    
\item {\bf Safeguards}
    \item[] Question: Does the paper describe safeguards that have been put in place for responsible release of data or models that have a high risk for misuse (e.g., pre-trained language models, image generators, or scraped datasets)?
    \item[] Answer: \answerNA{} % Replace by \answerYes{}, \answerNo{}, or \answerNA{}.
    \item[] Justification: The work does not release high-risk generative models or newly collected sensitive datasets; it releases a general-purpose forecasting loss implementation.

\item {\bf Licenses for existing assets}
    \item[] Question: Are the creators or original owners of assets (e.g., code, data, models), used in the paper, properly credited and are the license and terms of use explicitly mentioned and properly respected?
    \item[] Answer: \answerNo{} % Replace by \answerYes{}, \answerNo{}, or \answerNA{}.
    \item[] Justification: The original datasets, forecasting models and objective implementations are cited in Section~4 and Appendix~\ref{sec:reproduce}. The manuscript does not enumerate the license and terms of use of every third-party asset.

\item {\bf New assets}
    \item[] Question: Are new assets introduced in the paper well documented and is the documentation provided alongside the assets?
    \item[] Answer: \answerYes{} % Replace by \answerYes{}, \answerNo{}, or \answerNA{}.
    \item[] Justification: The new asset is the CvLoss code release at \url{https://github.com/Day333/CvLoss}; its README documents the loss, dependencies and example usage. No new dataset is introduced.

\item {\bf Crowdsourcing and research with human subjects}
    \item[] Question: For crowdsourcing experiments and research with human subjects, does the paper include the full text of instructions given to participants and screenshots, if applicable, as well as details about compensation (if any)? 
    \item[] Answer: \answerNA{} % Replace by \answerYes{}, \answerNo{}, or \answerNA{}.
    \item[] Justification: The research does not involve crowdsourcing or direct research with human subjects.

\item {\bf Institutional review board (IRB) approvals or equivalent for research with human subjects}
    \item[] Question: Does the paper describe potential risks incurred by study participants, whether such risks were disclosed to the subjects, and whether Institutional Review Board (IRB) approvals (or an equivalent approval/review based on the requirements of your country or institution) were obtained?
    \item[] Answer: \answerNA{} % Replace by \answerYes{}, \answerNo{}, or \answerNA{}.
    \item[] Justification: The research does not recruit or interact with human participants; it evaluates existing public forecasting benchmarks.

\item {\bf Declaration of LLM usage}
    \item[] Question: Does the paper describe the usage of LLMs if it is an important, original, or non-standard component of the core methods in this research? Note that if the LLM is used only for writing, editing, or formatting purposes and does \emph{not} impact the core methodology, scientific rigor, or originality of the research, declaration is not required.
    %this research? 
    \item[] Answer: \answerNA{} % Replace by \answerYes{}, \answerNo{}, or \answerNA{}.
    \item[] Justification: LLMs are not a component of the forecasting method or experiments. The appendix contains a separate statement on their use for manuscript preparation.

\end{enumerate}

\end{document}